\documentclass[9pt,twoside]{extarticle}

\usepackage{times}
\usepackage[papersize={17cm,24cm},margin=22mm,top=17mm,headsep=5mm]{geometry}
\usepackage[dvipsnames]{xcolor}
\usepackage[english]{babel} 

\definecolor{bred}{rgb}{0.8,0,0}

\usepackage{microtype}

\usepackage{subcaption}
\usepackage{booktabs} 
\usepackage{amsmath,amsfonts, amsthm, graphicx,epsfig}
\usepackage{enumerate, epstopdf}

\usepackage{wrapfig}

\usepackage[utf8]{inputenc} 
\usepackage[T1]{fontenc}    

\usepackage{bbm}
\usepackage{xargs}
\usepackage{enumerate}
\usepackage{latexsym}
\usepackage{wasysym}

\usepackage{float}
\usepackage{url}            
\usepackage{booktabs}       
\usepackage{amsfonts}       
\usepackage{nicefrac}       
\usepackage{microtype}      

\usepackage[textsize=footnotesize]{todonotes}

\usepackage{lipsum}
\usepackage{graphicx}
\usepackage{epstopdf}

\usepackage{hyperref}       
\usepackage{cleveref}

\hypersetup{colorlinks,linkcolor={blue},citecolor={bred},urlcolor={blue}}
\usepackage{natbib}
\setcitestyle{round,comma}

\newtheorem{theorem}{Theorem}[section]
\newtheorem{corollary}{Corollary}[theorem]
\newtheorem{proposition}{Proposition}[theorem]
\newtheorem{lemma}[theorem]{Lemma}
\newtheorem{remark}[theorem]{Remark}

\usepackage[algo2e]{algorithm2e}
\ifpdf
  \DeclareGraphicsExtensions{.eps,.pdf,.png,.jpg}
\else
  \DeclareGraphicsExtensions{.eps}
\fi

\usepackage{authblk}

\hypersetup{colorlinks,linkcolor={blue},citecolor={bred},urlcolor={blue}}

\def\rmd{\mathrm{d}}
\def\rset{\mathbb{R}}

\newcommand{\N}{\mathbb{N}}
\newcommand{\Z}{\mathbb{Z}}

\newcommand{\R}{\mathbb{R}}

\newcommand{\ind}{\mathbf{1}}

\newcommand{\lvrv}[1]{\left\Vert #1\right\Vert}
\newcommand{\eps}{\varepsilon}
\newcommand{\la}{\lambda}

\newcommand{\vinf}[1]{\lvrv{#1}_{\infty}}
\newcommand{\vlip}[1]{\lvrv{#1}_{\text{Lip}}}
\newcommand{\vsob}[1]{{\left\vert\kern-0.25ex\left\vert\kern-0.25ex\left\vert #1
		\right\vert\kern-0.25ex\right\vert\kern-0.25ex\right\vert}}

\newcommand{\lin}[1]{\text{Lin}\left(#1\right)}
\newcommand{\lip}[1]{\text{Lip}\left(#1\right)}
\newcommand{\id}[1]{\text{id}_{\R^{#1}}}

\newcommand{\hl}{H_\lambda(\theta_n,X_{n+1})}
\newcommand{\zs}{\bar{\zeta}_{s}^{\lambda, n}}
\newcommand{\fs}{\bar{\theta}^\lambda_{\lfloor s \rfloor}}
\newcommand{\ths}{\bar{\theta}^\lambda_ s}
\newcommand{\Xs}{X_{\lceil s\rceil}}
\newcommand{\zt}{\bar{\zeta}_{t}^{\lambda, n}}
\newcommand{\ft}{\bar{\theta}^\lambda_{\lfloor t \rfloor}}
\newcommand{\Xt}{X_{\lceil t\rceil}}
\newcommand{\tht}{\bar{\theta}^\lambda_ t}
\newcommand{\tn}{\theta^\lambda_n}
\DeclareMathOperator*{\il}{\mathbf{1}_{A_{n,M}}}
\DeclareMathOperator*{\ilp}{\mathbf{1}_{A_{n,M}}^C}
\DeclareMathOperator*{\ig}{\mathbf{1}_{A_{n,M_0}}}
\DeclareMathOperator*{\igp}{\mathbf{1}_{A_{n,M_0}^C}}
\DeclareMathOperator*{\E}{\mathbb{E}}

\begin{document}

\title{Taming neural networks with TUSLA: Non-convex learning via adaptive stochastic gradient Langevin algorithms
\thanks{All the authors were supported by The
 Alan Turing Institute, London under the EPSRC grant EP/N510129/1. A. L. and M. R. thank for the ``Lend\"ulet'' grant
 LP 2015-6 of the Hungarian Academy of Sciences.}}

\author[1]{Attila Lovas}
\author[2]{Iosif Lytras}
\author[1]{Mikl\'os R\'asonyi}
\author[2, 3, 4]{Sotirios Sabanis}

\affil[1]{\footnotesize Alfr\'ed R\'enyi Institute of Mathematics, 1053 Budapest, Re\'altanoda utca 13--15, Hungary}
\affil[2]{\footnotesize School of Mathematics, The University of Edinburgh, UK.}
\affil[3]{\footnotesize The Alan Turing Institute, London, UK.}
\affil[3]{\footnotesize National Technical University of Athens, Athens, 15780, Greece.}

\maketitle

\begin{abstract}
Artificial neural networks (ANNs) are typically highly nonlinear systems which are finely tuned via the optimization of their associated, non-convex loss functions. In many cases, the gradient of any such loss function has superlinear growth, making the use of the widely-accepted (stochastic) gradient descent methods, which are based on Euler numerical schemes, problematic. We offer a new learning algorithm based on an appropriately constructed variant of the popular stochastic gradient Langevin dynamics (SGLD), which is called tamed unadjusted stochastic Langevin algorithm (TUSLA). We also provide a nonasymptotic analysis of the new algorithm's convergence properties in the context of non-convex learning problems with the use of ANNs. Thus, we provide  finite-time guarantees for TUSLA to find approximate minimizers of both empirical and population risks. The roots of the TUSLA algorithm are based on the taming technology for diffusion processes with superlinear coefficients as developed in \cite{tamed-euler, SabanisAoAP} and for MCMC algorithms in \cite{tula}. Numerical experiments are presented which confirm the theoretical findings and illustrate the need for the use of the new algorithm in comparison to vanilla SGLD within the framework of ANNs.

\end{abstract}

\section{Introduction}

A new generation of stochastic gradient decent algorithms, namely stochastic gradient Langevin dynamics (SGLD), can be efficient in finding global minimizers of possibly complicated, high-dimensional landscapes under suitable regularity assumptions for the gradient, see \cite{raginsky, welling2011bayesian} and references therein. These algorithms are based on the fundamental concept that the problem of finding the minimizer of a non-convex objective function $u$ is connected to the problem of sampling from the target distribution $\pi_{\beta}(\mathrm{d} \theta) \propto \exp (-\beta u(\theta)) \mathrm{d} \theta $, for $\beta$ sufficiently large, see \cite{hwang}. Under mild conditions, this is the invariant distribution of the Langevin SDE:
\begin{equation}\label{eq-LangSDE}
    L_{0}=\theta_{0}, \quad \mathrm{~d} L_{t}=-\nabla u\left(L_{t}\right) \mathrm{d} t+\sqrt{2 \beta^{-1}} \mathrm{~d} B_{t}, \quad t \in \mathbb{R}_{+}.
\end{equation} The SGLD algorithm is given by \begin{equation}\label{eq-SGLD}
    \theta_{0}^{\mathrm{SGLD}}:=\theta_{0}, \quad \theta_{n+1}^{\mathrm{SGLD}}=\theta_{n}^{\mathrm{SGLD}}-\lambda H\left(\theta_{n}^{\mathrm{SGLD}}, X_{n+1}\right)+\sqrt{2 \lambda \beta^{-1}} \xi_{n+1}, \quad n \in \mathbb{N}_{0}
\end{equation}
where $\lambda>0$ is the stepsize, $\left(X_{n}\right)_{n \in \mathbb{N}_{0}}$ is an i.i.d. sequence of random variables $\beta>0$, and $\left\{\xi_{n}\right\}_{n \geq 1}$ is a sequence of independent standard $d$-dimensional Gaussian random variables.This algorithm is a version of an Euler discretization of \eqref{eq-LangSDE}  where in the drift coefficient, $\nabla u$ is replaced by an unbiased estimator $H$ such that $\nabla u(\theta)=\E[H(\theta,X_0)]$.\\
However, in the specific case of tuning ANNs, or simply neural networks henceforth,  problems could arise already at the theoretical level. As discussed in Section \ref{example_non_diss} below in some detail, the functionals to be minimized in such a task may fail any form of dissipativity which should be a \emph{sine qua non} for guaranteeing the stability of associated gradient algorithms. Adding a quadratic regularization term cannot always remedy this, due to the superlinear features of the associated gradients (see Proposition \ref{A2_for_ANN}), in which case one needs to replace it with a higher order penalty term. However, the addition of such a term maintains the violation of the global Lipschitz continuity for the regularized gradient (due to Proposition \ref{polip H}), which in turn renders the use of gradient descent methods problematic. This issue has been highlighted in the case of Euler discretizations (of which SGLD is an example) in \cite{hutzenthaler2011}, where it is proven that the difference between the exact solution of the corresponding stochastic differential equation (SDE) and the numerical approximation at even a finite time point diverges to infinity in the strong mean square sense.

A natural way to address the above issue is to combine higher order regularization with taming techniques to improve the stability of any resulting algorithm. In particular, the use of taming techniques in the construction of stable numerical approximations for nonlinear SDEs has gained substantial attention in recent years and was introduced by \cite{hutzenthaler2012} and, independently, by \cite{tamed-euler, SabanisAoAP}. The latter taming approach was used in the creation of a new generation of Markov chain Monte Carlo (MCMC) algorithms, see \cite{tula, hola}, which are designed to sample from distributions such that the gradient of their log density is only locally Lipschitz continuous and is allowed to grow superlinearly at infinity.

It is essential here to recall the importance of Langevin based algorithms. Their nonasymptotic convergence analysis has been highlighted in recent years by numerous articles in the literature. For the case of deterministic gradients one could consult \cite{dalalyan2017theoretical, durmus2017nonasymptotic, durmus2019high, berkeley, hola} and references therein, whereas for stochastic gradients of convex potentials details can be found in \cite{brosse2018promises, dalalyan2019user} and in \cite{convex} which goes beyond the case of iid data. Further, due to the newly obtained results in the study of contraction rates for Langevin dynamics, see \cite{eberle2019couplings, Harris}, the case of nonconvex potentials within the framework of stochastic gradients was studied in \cite{raginsky, xu2018global} and, in particular, substantial progress has been made in \cite{nonconvex} by obtaining the best known convergence rates even in the presence of dependent data streams. The latter article has inspired the development of the SGLD theory under local conditions, see \cite{zhang2019nonasymptotic}, which provides theoretical convergence guarantees for a wide class of applications, including scalable posterior sampling for Bayesian inference and nonconvex optimization arising in variational inference problems.

Despite all this very significant progress, the use of SGLD algorithms for the fine tuning of neural networks remained only at a heuristic level without any theoretical guarantees for the discovery of approximate minimizers of empirical and population risks. To the best of the authors' knowledge, the current article is the first work to address this shortcoming in the theory of Langevin algorithms by presenting a novel algorithm, which is called tamed unadjusted stochastic Langevin algorithm (TUSLA), along with a nonasymptotic analysis of its convergence properties.

We conclude this section by introducing some notation. Let $(\Omega,\mathcal{F},P)$ be a probability space. We denote by $\E[X]$  the expectation of a random variable $X$.
For $1\leq p<\infty$, $L^p$ is used to denote the usual space of $p$-integrable real-valued random variables.
Fix an integer $d\geq 1$. For an $\mathbb{R}^d$-valued random variable $X$, its law on $\mathcal{B}(\mathbb{R}^d)$, i.e. the Borel sigma-algebra of $\mathbb{R}^d$, is denoted by $\mathcal{L}(X)$. Scalar product is denoted
by $\langle \cdot,\cdot\rangle$, with $|\cdot|$ standing for the
corresponding norm (where the dimension of the space may vary depending on the context). For $\mu\in\mathcal{P}(\mathbb{R}^d)$ and for a non-negative measurable $f:\mathbb{R}^d\to\mathbb{R}$, the notation $\mu(f):=\int_{\mathbb{R}^d} f(\theta)\mu(\rmd \theta)$ is used. For any integer $q \geq 1$, let $\mathcal{P}(\mathbb{R}^q)$ denote the set of probability measures on $\mathcal{B}(\mathbb{R}^q)$.
For $\mu,\nu\in\mathcal{P}(\mathbb{R}^d)$, let $\mathcal{C}(\mu,\nu)$ denote the set of probability measures $\zeta$
on $\mathcal{B}(\mathbb{R}^{2d})$ such that its respective marginals are $\mu,\nu$. For two probability measures $\mu$ and $\nu$, the Wasserstein distance of order $p \geq 1$ is defined as
\begin{equation} \label{eq:definition-W-p}
{W}_p(\mu,\nu):=\inf_{\zeta\in\mathcal{C}(\mu,\nu)}
\left(\int_{\rset^d}\int_{\rset^d}|\theta-\theta'|^p\zeta(\rmd \theta \rmd \theta')\right)^{1/p},\ \mu,\nu\in\mathcal{P}(\rset^d).
\end{equation}
{We note here that our main results contain several constants, which are given explicitly, in most cases, within the relevant proofs. However, in order to help the reader identify their structure and dependence on real-problem parameters in a systematic way, two tables appear in the Appendix which list these constants along with the necessary information.}

\section{Main results and assumptions}
We consider initially the setting which is required for the precise formulation of the newly proposed algorithm. To this end, let us denote by $(\mathcal{G}_n)_{n\in\mathbb{N}}$ a given filtration representing the flow of past information. Moreover, let $(X_n)_{n\in\mathbb{N}}$ be an $\mathbb{R}^m$-valued, $(\mathcal{G}_n)$-adapted process and $(\xi_{n})_{n\in\mathbb{N}}$ be  an $\mathbb{R}^d$-valued Gaussian process.  It is assumed throughout the paper that the random variable $\theta_0$ (initial condition), $\mathcal{G}_{\infty}$ and $(\xi_{n})_{n\in\mathbb{N}}$ are independent. Let also $G:\mathbb{R}^d \times \mathbb{R}^m\rightarrow \mathbb{R}^d$ be a continuously differentiable function. The required assumptions are as follows.
\subsection{Assumptions and key observations}
\label{Assumptions}
Although the assumptions below are presented in a formal way for the general case of locally Lipschitz continuous gradients, the connection with neural networks is given explicitly in Section \ref{MNNs}. In particular, the function $G$ below can be seen as the stochastic gradient described in equation \eqref{wq:G}.
\noindent
\newtheorem{assumption}{Asssumption}
\begin{assumption} \label{A2}
	There exist positive constants $L_1, \rho$ and $q\ge1$ such that
	\[\left|G(\theta, x)-G\left(\theta^{\prime}, x\right)\right| \leq L_1(1+|x|)^{\rho}(1+|\theta|+|\theta'|)^{q-1}\left|\theta-\theta^{\prime}\right|,
	\mbox{ for all } x \in \mathbb{R}^{m}
\]
and $\theta, \theta^{\prime} \in \mathbb{R}^d$.

\end{assumption}
\newtheorem{regularization}{Definition}
\begin{regularization} \label{A1}
Let  $\eta\in (0,1)$ be a regularization parameter and $r$ be a constant such that $r\geq \frac{q}{2}+1.$ Then, the stochastic gradient with the necessary regularised term is given by
\[
H(\theta,x):=G(\theta,x) +\eta \theta |\theta|^{2r}
\]
for all $x\in \mathbb{R}^{m}$ and $\theta \in \mathbb{R}^d$. Moreover, $
g(\theta):=\mathbb{E}\left[G\left(\theta, X_{0}\right)\right]$ and $h(\theta):=\E\left[H(\theta,X_0)\right]$ for every $\theta \in \mathbb{R}^d$.\\
The gradient of the 'regularized' objective function $u$ is given as
\begin{equation}
    \nabla u(\theta)=h(\theta)=\E \left[H(\theta,X_0)\right]=\E \left[G(\theta,X_0)\right]+\eta \theta |\theta|^{2r}.
\end{equation}
\end{regularization}
\newtheorem{Optimization problem}{Remark}
\begin{Optimization problem} \label{Opt_problem}
As an example, $H$ can be seen as the gradient of a function of the form
\[\begin{aligned}
U(\theta,x):=F(\theta,x)+ \frac{\eta}{2(r+1)}|\theta|^{2(r+1)}, \mbox{ and } G(\theta,x):=\nabla_{\theta} F(\theta,x),
\end{aligned}
\]
$\mbox{ for all  }\theta \in \mathbb{R}^d, x \in \mathbb{R}^m.$
\end{Optimization problem}
\begin{assumption} \label{A3}
The process $(X_n)_{n\ge0}$ is a sequence of i.i.d. random variables with $\E |X_0|^{16\rho(2r+1)}< \infty$, where $\rho$ is given in Assumption \ref{A2} and $r$ in Definition \ref{A1}. In addition, the initial condition is such that
$\E |\theta_0|^{16(2r+1)}<\infty$.
\end{assumption}
\newtheorem{Differentiation of g}[Optimization problem]{Remark}

\begin{Differentiation of g}\label{remark2.3}
By taking a closer look at Assumption \ref{A2}, one observes that  the growth of $G$ can be controlled, i.e.  for every $\theta \in \mathbb{R}^d$ and $x \in \mathbb{R}^m$
\begin{equation}\label{G_growth}
|G(\theta,x)|\leq K(x)(1+|\theta|^{q}),
\end{equation}
where $K(x)=2^{q-1}(L_1(1+|x|)^{\rho}+|G(0,x)|$.
\end{Differentiation of g}
\newtheorem{remark2}[Optimization problem]{Remark}
\begin{remark2} \label{remark2}
In view of Assumptions \ref{A2} and \ref{A3}, one obtains that
$$
\langle \theta,  \E\left[G(\theta,X_0)\right]\rangle \geq  -\E\left[K(X_0)\right](|\theta|+|\theta|^{q+1}).
$$
which leads to
\[
\langle \theta, h(\theta)\rangle=\langle \theta ,\E G(\theta,X_0)\rangle +\langle \theta,\eta \theta |\theta|^{2r}\rangle \geq \eta |\theta|^{2r+2}-\E[K(X_0)]|\theta|(1+|\theta|^q).
\]
Furthermore, for  $A= \E[K(X_0)]$ and $B= \left(3\E[K(X_0)]\right)^{q+2}\eta^{-q-1}$, it holds that
\begin{equation}\label{eq-AB}
    \langle \theta, h(\theta)\rangle\geq A|\theta|^2-B.
\end{equation}
\end{remark2}
\newtheorem{Remark1}{Proposition}
\begin{Remark1}\label{Proposition 1}
Let Assumptions \ref{A2} and \ref{A3} hold. Then, for every $\theta, \theta^{\prime} \in \mathbb{R}^d$,
\[\langle \theta-\theta',h(\theta)-h(\theta')\rangle\geq -a |\theta-\theta'|^2,\]
where $a=L_2(1+2|R|)^{q-1}$ and $R$ , $L_2$ are given explicitly in the proof.
\end{Remark1}
At this point, a natural question arises about the use of the above specified regularization term $\eta\theta|\theta|^{2r}$. One notes first that a dissipativity property such as the one in Remark \ref{remark2} is, typically, stated as an assumption in  the stochastic gradient literature. This is due to the fact that dissipativity plays a pivotal role in the derivation of moment estimates and, consequently, in the algorithm's stability. Although in many well known examples such a condition is verified, it is desirable that a theoretical framework is built for more complicated cases. In the current work, the $\eta-r$-regularization is a novel way to deal with scenarios where the validity of a dissipativity condition cannot  be verified. The reason for this is that local Lipshcitz continuity typically yields significantly underestimated lower bounds for the growth of the gradient.  Thus, in order to provide here full theoretical guarantees of the behaviour and convergence properties of our proposed algorithm, regularization of the order $2r+1$ is used in order to compensate for such (extreme) lower bounds. Remark 3 describes how such a compensation is achieved. One further notes that it is possible that better lower bounds can be guaranteed a-priori, which depend on more specific information about the structure of the gradient, and thus a suitable dissipativity condition can be achieved by a weaker regularization. For example, if a dissipativity condition as in Remark \ref{remark2} is already satisfied for the gradient of the objective function, we simply set $\eta=0.$ That is to say, in real-world applications, the confirmation whether the gradient of an objective function is
dissipative becomes a problem-specific calculation, which finally dictates whether (and what kind of)
high-order regularisation is required.\\
To sum up, when one works within the full theoretical framework as it is shaped by Assumption \ref{A2}, the proposed regularization guarantees that a suitable dissipativity condition holds true, something that is central to the analysis of algorithms for non-convex potentials.\\
The following proposition states that the stochastic gradient is not globally Lipschitz continuous in $\theta$, hence a new approach is required for learning schemes which rely on the analysis of Langevin dynamics with gradients satisfying weaker smoothness conditions. Crucially though, the local Lipschitz continuity property remains true and, moreover, the associated local Lipschitz constant is controlled by powers of the state variables which allow us to use an approach based on taming techniques.
\newtheorem{polip H}[Remark1]{Proposition}
\begin{polip H} \label{polip H}
Let Assumptions \ref{A2} and \ref{A3} hold. Then, in view of Definition \ref{A1} one obtains that
\[
|H(\theta,x)-H(\theta',x)|\leq L (1+|x|)^\rho(1+|\theta|+|\theta'|)^l |\theta-\theta'|, \mbox{ for all } x \in \mathbb{R}^{m},
\]
and $\theta, \theta^{\prime} \in \mathbb{R}^d$, where $L=L_1+8r\eta$ and $l=2r+1$.
\end{polip H}

\subsection{The new algorithm and main results}
We introduce a new iterative scheme, which is a hybrid of the stochastic gradient Langevin dynamics (SGLD) algorithm and of the tamed unadjusted Langevin algorithm and uses `taming', see \cite{tamed-euler, SabanisAoAP}, \cite{tula} and references therein, for asserting control on the superlinearly growing gradient. This new algorithm is called TUSLA, tamed unadjusted stochastic  Langevin algorithm,  and is given by
\begin{equation}\label{eq:TUSLA}
\theta^{\lambda}_{n+1}:=\theta^{\lambda}_n-\lambda H_\lambda(\theta^{\lambda}_n,X_{n+1})+ \sqrt{2\lambda\beta^{-1}} \xi_{n+1},\ n\in\mathbb{N},
\end{equation}
where $ \theta^{\lambda}_0:=\theta_0$ and
\begin{equation} \label{def: H_lambda}
    H_\lambda(\theta,x):=\frac{H(\theta,x)}{1+\sqrt{\lambda} |\theta|^{2r}}, \qquad \mbox{for every } \theta \in \mathbb{R}^d, x \in \mathbb{R}^m,
\end{equation}
where $\{\xi_n\}_{n\ge 1}$ is a sequence of independent standard $d$-dimensional Gaussian random variables and $H$ is given in Definition \ref{A1}. {This algorithm has two new elements compared to standard SGLD algorithms. The first is the added regularization term in the numerator of the drift term, which enables to us to derive important conditions (e.g. see Proposition  \ref{Proposition 1}) with minimal assumptions. The second new element is the division of the regularised gradient by a suitable term, which enables the new algorithm to inherit the stability properties of tamed algorithms.} Consequently, it addresses known stability issues of SGLD algorithms, and can be seen as an SGLD algorithm with an adaptive step size. This is due to the fact that, at each iteration, the stochastic gradient $H$ is multiplied with a step size which is controlled by the $2r$-th power of the (vector) norm of the parameter, i.e. by $\lambda\left(1+\sqrt{\lambda} |\theta|^{2r}\right)^{-1}$.

Henceforth, $\lambda$ is assumed to be controlled by
\begin{equation} \label{lambda_max}
    \lambda_{max}=\min\{1,\frac{\kappa^2}{4\eta\left(8(p+1)\binom{p}{\lceil \frac{p}{2}\rceil}^2\right)^2 }\}
\end{equation}
where $p$ depends on which $2p$-th moment of $\theta_n$ we need to estimate and $\kappa$ is given in the proof in  \eqref{eq-kappa}, see Appendix.
\newtheorem{remark5}[Optimization problem]{Remark}
\begin{remark5}
Observe that, due to Remark \ref{remark2.3} and \eqref{def: H_lambda},
\begin{equation} \label{H_growth}
\begin{aligned}
  \E[\sqrt{\lambda}|H_\lambda(\theta_n^{\lambda},X_{n+1})|\big{|}\theta_n^{\lambda}]& \leq \sqrt{\lambda}\frac{\E \left[K(X_0)\right](1+|\theta_n^{\lambda}|^q)+\eta |\theta_n^{\lambda}|^{2r+1}}{1+\sqrt{\lambda}|\theta^\lambda_n|^{2r}}\\& \le \E \left[K(X_0)\right]+\eta|\theta_n^{\lambda}|.
  \end{aligned}
\end{equation}
Moreover,
\begin{align} \label{sqr_growth}
    \E[{\lambda}|H_\lambda(\theta_n^{\lambda},X_{n+1})|^2|\theta_n^{\lambda}] &\le 4 \E[K^2(X_0)]  +2\eta^2|\theta_n^{\lambda}|^2.
\end{align}
\end{remark5}
It is well-known that, under mild conditions, which in this case are satisfied due to Assumptions \ref{A2}--\ref{A3} and, in particular, due to \eqref{eq-AB}, the so-called (overdamped) Langevin SDE which is given by
\begin{equation} \label{langevin-SDE}
\mathrm{d} Z_{t}=-h\left(Z_{t}\right) \mathrm{d} t+ \sqrt{2\beta^{-1}} \mathrm{d} B_{t}, \quad t>0
\end{equation}
with a (possibly random) initial condition $\theta_0$ and with $B_t$ denoting a $d$-dimensional Brownian motion, admits a unique invariant measure $\pi_\beta$ given by
\begin{equation}\label{eq-pibeta}
    \pi_\beta(x)=\frac{e^{-\beta u(x)} }{\int e^{-\beta u(x)}dx}
\end{equation}
where $u$ is a function such that $\nabla u=h.$
The two main results are given below with regards to the convergence of TUSLA \eqref{eq:TUSLA} to $\pi_\beta$ in metrics $W_{1}$ and $W_{2}$ as defined in \eqref{eq:definition-W-p}.
\newtheorem{W1 measure}{Theorem}
\begin{W1 measure}\label{Thrm1}
Let Assumptions \ref{A2} and \ref{A3} hold. Then, there exist positive constants $C_1$, $C_2$, $\hat{c}$, $\dot{c}$ and $z_1$ such that, for every $0<\lambda\leq \lambda_{\max}$,
\[\begin{aligned}
W_{1}\left(\mathcal{L}\left(\theta_{n}^{\lambda}\right), \pi_{\beta}\right)& \leq\sqrt{\lambda} (z_1+\sqrt{e^{3a}(C_1+C_2+C_3)})\\&+\hat{c} e^{-\dot{c}\lambda n }\left[1+\mathbb{E}\left[V_{2}\left(\theta_{0}\right)\right]+\int_{\mathbb{R}^{d}} V_{2}(\theta) \pi_{\beta}(d \theta)\right],
\end{aligned}
\]
where $V_2$ is defined in \eqref{def:Lyapunov} and $a$ is defined in Proposition \ref{Proposition 1}. The constants are given explicitly in the proof.
\end{W1 measure}
\newtheorem{measure conv}{Corollary}
\begin{measure conv}\label{Thrm2}
Let Assumptions \ref{A2} and \ref{A3} hold. Then, there exist positive constants $C_1$, $C_2$ and $z_2$ such that, for every $0<\lambda \leq \lambda_{max}$,
\[\begin{aligned}
W_{2}\left(\mathcal{L}\left(\theta_{n}^{\lambda}\right), \pi_{\beta}\right)  &\leq \sqrt{e^{3a}(C_1+C_2+C_3)}\sqrt{\lambda}+z_2\lambda^\frac{1}{4} \\&+\sqrt{2 \hat{c} e^{-\dot{c}\lambda n} \left(1+\mathbb{E}\left[V_{2}\left(\theta_{0}\right)\right]+\int_{\mathbb{R}^{d}} V_{2}(\theta) \pi_{\beta}(d \theta)\right)},
\end{aligned}
\]
where $V_2$ is defined in \eqref{def:Lyapunov}. The constants are given explicitly in the proof.
\end{measure conv}

If we further assume the setting of Remark \ref{Opt_problem}, where $h:=\nabla u$ with $u(\theta) = \E[U(\theta, X_0)] \ge 0$, then the following non-convex optimization problem can be formulated
\[
\text{minimize} \quad u(\theta) : = \mathbb{E}[U(\theta, X_0)],
\]
where $\theta \in \mathbb{R}^d$ and $X_0$ is a random element with some unknown probability law. One then needs to estimate a $\hat{\theta}$, more precisely its law,  such that the expected excess risk $\mathbb{E}[u(\hat{\theta})] - \inf_{\theta \in \mathbb{R}^d} u(\theta)$ is minimized. This optimization problem can thus  be decomposed into subproblems, see \cite{raginsky}, one of which is a problem of sampling from the target distribution $ \pi_{\beta}( \theta) \wasypropto \exp(-\beta u(\theta))$ with $\beta>0$. The results in Theorem \ref{Thrm1} and Corollary \ref{Thrm2} provide the estimates for this sampling problem. Moreover, at  an intuitive level, one understands that the two problems, namely sampling and optimization,  are linked in this case since $\pi_{\beta}$ concentrates around the minimizers of $u$ when $\beta$ takes sufficiently large values, see \cite{hwang} for more details.  In fact, one observes that if $\theta_{n}^{\lambda}$ is used in place of $\hat{\theta}$, then  expected excess risk can be estimated as follows
\begin{equation} \label{sum_T1_T2}
\mathbb{E}\left[u\left(\theta_{n}^{\lambda}\right)\right]-u_{\star} =\underbrace{\mathbb{E}\left[u\left(\theta_{n}^{\lambda}\right)\right] -\mathbb{E}\left[u\left(\theta_{\infty}\right)\right]}_{\mathcal{T}_{1}} +\underbrace{\mathbb{E}\left[u\left(\theta_{\infty}\right)\right]-u_{\star}}_{\mathcal{T}_{2}}
\end{equation}
where $u_{\star}:=\inf _{\theta \in \mathbb{R}^{d}} u(\theta)$ and $\theta_\infty$ stands for a random variable that follows $\pi_\beta$. Moreover, the estimates for $\mathcal{T}_{1}$ rely on the $W_{2}$ estimates of Corollary \ref{Thrm2} and the estimates for $\mathcal{T}_{2}$ on the properties of the corresponding Gibbs algorithm, see \cite[Section~3.5]{raginsky}.

\newtheorem{excess_risk}[W1 measure]{Theorem}
\begin{excess_risk}\label{excess_risk}
Let Assumptions \ref{A2} and \ref{A3} hold. Then, if $\beta \geq \frac{2}{A}$,
\begin{align*}
  \mathbb{E}\left[u\left(\theta_{n}^{\lambda}\right)\right]-u_{\star}  \leq & \left( \frac{a_1}{l+1}\sqrt{\E|\theta_0|^{2l}+C'_l}+\frac{a_1}{l+1}\sqrt{\sigma_{2l}}+r_2\right)
  W_{2}\left(\mathcal{L}\left(\theta_{n}^{\lambda}\right), \pi_{\beta}\right)   \\
   &  + \frac{d}{2 \beta} \log \left(\frac{e K}{A}\left(\frac{B \beta}{d}+1\right)\right)-\frac{1}{\beta}\log\left(1-e^{-(R_0\sqrt{K\beta}-\sqrt{d})^2}\right),
\end{align*}
where $a_1=2^l(\E[K(X_0)] +\eta)$, $r_2=2\E[K(X_0)]$, $\sigma_{2l}$ is the $2l$-moment of $\pi_\beta$,
\[
R_0 = \inf \{y\geq \sqrt{B/A}: \quad y^2  (1+4y)^l>\frac{d+1}{\beta L\E (1+|X_0|)^\rho}\}
\]
$K=L\mathbb{E}(1+|X_0|)^\rho(1+4R_0)^l$ and $W_{2}\left(\mathcal{L}\left(\theta_{n}^{\lambda}\right), \pi_{\beta}\right) $ is given in Corollary \ref{Thrm2}.
\end{excess_risk}
\section{Comparison with related work and our contributions}
While in \cite{zhang2019nonasymptotic} the analysis of non-convex (stochastic) optimization problems is presented, its main focus remains on objective functions with gradients which are globally Lipschitz in the parameter (denoted by $\theta$), while this assumption is significantly relaxed in our article and thus a much larger class of optimization problems is included. Despite the technical obstacles imposed by this more general framework, our article succeeds in dealing with both the sampling problem and the excess risk minimization problem achieving the best known rates of convergence (for non-convex optimization problems). Moreover, while the achieved convergence rates in $W_1$ and $W_2$ distances are the same for both articles, since both of them rely on contraction estimates from \cite{eberle2019couplings}, the novelty in our article is achieved by the newly developed methodology, which allows considerable loosening of the smoothness assumption. This, in turn, allows the inclusion of the fine tuning (via expected risk minimization) of the parameters of feed-forward neural networks in their full generality within our setting, i.e. even in the presence of online data streams (online learning) or with data from distributions with unbounded support. To the best of the authors' knowledge, this is the first such result. \\
More concretely regarding the comparison of assumptions, one observes that Assumption 2 in \cite{zhang2019nonasymptotic} is considerably stronger than our Assumption \ref{A2}, since in the latter only local Lipschitzness of the objective function's gradient is assumed. Moreover, there is no dissipativity assumption in our setting in contrast to Assumption 3 of \cite{zhang2019nonasymptotic}. However, we need to include a high-order regularisation term in our objective function, which is controlled by a tiny quantity $\eta$, and as a result a dissipative condition is satisfied  which can be found in our Remark 3. We stress here that the high order regularisation becomes necessary only in the absence of dissipativity (if a dissipativity condition holds we set $\eta=0$). Finally, Assumption \ref{A3}, regarding moment requirements, is comparable with Assumption 1 of \cite{zhang2019nonasymptotic} as it is problem dependent.\\
We turn now our attention to the  article \cite{tula}, which also uses a taming approach to address the instability due to superlinear gradients. One immediately notes that \cite{tula} focuses on deterministic gradients, whereas we work with the full stochastic counterparts, and thus, even in the context of the corresponding  sampling problem, our setting is much more general. Furthermore,  the $W_2$ estimates in \cite{tula} are obtained within a strongly convex setting (see Assumption H3 in \cite{tula}). Here we note that although we obtain a rate of 1/4 in $W_2$ in our non-convex setting, this trivially increases to 1/2, as in \cite{tula}, if a strong convexity condition is assumed as one then replaces the contraction estimates due to \cite{eberle2019couplings} with standard $W_2$ estimates under strong convexity.\\
Finally, we discuss the constants which appear in Theorem \ref{Thrm1} and \ref{excess_risk}. A careful analysis  of our results shows that there is an exponential dependence in dimension (see Tables \ref{tab:Other constants} and \ref{tab:basic constants} in the Appendix), which is inherited from the contraction results of \cite{eberle2019couplings}. The same is true for the corresponding results in \cite{nonconvex} and \cite{zhang2019nonasymptotic} as the aforementioned contraction results are central to the analysis of the full non-convex case.  Any other dependence on the dimension is polynomial and is obtained via the finiteness of the required moments (see Lemma \ref{pmoments}), very much like in \cite{tula}. Note that if the more restrictive, convex setting of the aforementioned article is adopted, then the exponential dependence on the dimension in our results simply ceases to exist.

\section{Preliminary estimates} \label{Prelim}
\newtheorem{lemma2}{Lemma}
At this point the necessary moments estimates are presented, which guarantee the stability of the new algorithm, along with the necessary (for the approach taken in the proof of the main results) auxiliary processes.
\begin{lemma2} \label{pmoments}
Let Assumption \ref{A2} and \ref{A3} hold.
For all $n \in \mathbb{N}$, $p \in [1,\, 8(2r+1)]$ and $0<\lambda<\lambda_{max},$
\[
\E|\theta^\lambda_{n+1}|^{2p}\leq (1-\lambda \frac{\kappa}{2}\eta)^n \E |\theta_0|^{2p}+C'_{p} \mbox{ and, thus, }
\sup_n \E|\theta_n^{\lambda}|^{2p}<\E|\theta_0|^{2p}+ C'_p,
\]
where $C'_p$ and $\kappa$ is given explicitly in the proof.
\end{lemma2}
Before proceeding with the detailed calculations regarding the convergence properties of TUSLA, a suitable family of Lyapunov functions is introduced. For each $m\geq 1$, define the Lyapunov function $V_m$ by
\begin{equation} \label{def:Lyapunov}
V_{m}(\theta):=\left(1+|\theta|^{2}\right)^{m / 2},\quad  \theta \in \mathbb{R}^{d},
\end{equation}
and similarly $v_m(x)=(1+x^2)^\frac{m}{2}$ for any real $x \geq 0$.\\ Both functions are continuously differentiable and
$\lim _{|\theta| \rightarrow \infty} \nabla V_{m}(\theta)/V_{m}(\theta)=0.$\\
We next introduce the auxiliary processes which are used in our analysis.\\ For each $\lambda>0$,  $Z_{t}^{\lambda}:=Z_{\lambda t},$  $t \in \mathbb{R}_{+},$ where the process $\left(Z_{s}\right)_{s \in \mathbb{R}_{+}}$ is defined in
\eqref{langevin-SDE}.\\ We also define $\tilde{B}_{t}^{\lambda}:=B_{\lambda t} / \sqrt{\lambda}, t \in \mathbb{R}_{+},$ where $\left(B_{s}\right)_{s \in \mathbb{R}_{+}}$ denotes the standard Brownian
motion. We note that $\tilde{B}_{t}^{\lambda}$ is a Brownian motion and
\begin{equation}
    \mathrm{d} Z_{t}^{\lambda}=-\lambda h\left(Z_{t}^{\lambda}\right) \mathrm{d} t+\sqrt{2 \lambda \beta^{-1}} \mathrm{~d} \tilde{B}_{t}^{\lambda}, \quad Z_{0}^{\lambda}=\theta_{0} \in \mathbb{R}^{d}.
\end{equation}

{\text { Denote by } $\mathcal{F}_{t}$ \text { the natural filtration of } $B_{t}, t \in \mathbb{R}_{+}$ . \text {Then, } $\mathcal{F}_{t}^{\lambda}:=\mathcal{F}_{\lambda t}, t \in \mathbb{R}_{+} \text {is }$
 the natural filtration of  $\tilde{B}_{t}^{\lambda}, t \in \mathbb{R}_{+} $ and \text { is independent of } $\mathcal{G}_{\infty} \vee \sigma\left(\theta_{0}\right) $\text { . }}

\newtheorem{definition 1}[regularization]{Definition}
\begin{definition 1}
We define the continuous-time interpolation of TUSLA, see \eqref{eq:TUSLA}, as
\begin{equation} \label{eq:CI_TUSLA}
    \mathrm{d} \bar{\theta}_{t}^{\lambda}=-\lambda H_\lambda\left(\bar{\theta}_{\lfloor t\rfloor}^{\lambda}, X_{\lceil{t}\rceil}\right) \mathrm{d} t+ \sqrt{2 \lambda\beta^{-1}} \mathrm{d} \tilde{B}_{t}^{\lambda}
\end{equation}
with initial condition $\bar{\theta}_{0}^{\lambda}=\theta^\lambda_0.$
\end{definition 1}
\newtheorem{combinelaw}[Optimization problem]{Remark}
\begin{combinelaw} \label{same_moments}
Moreover, due to the homogeneous nature of the coefficients of the continuous-time interpolation of the TUSLA algorithm, the law of the  interpolated process \eqref{eq:CI_TUSLA} has the same law with the process of TUSLA  \eqref{eq:TUSLA} a.s
at grid points, i.e. $
   \mathcal{L}\left(\bar{\theta}_{n}^{\lambda}\right)=\mathcal{L}\left(\theta_{n}^{\lambda}\right), \quad \forall n \in \mathbb{N}$.
Combining this with the bounds obtained in Lemmas \ref{pmoments}, one deduces that under the same assumptions,
\begin{equation} \label{law-connection}
    \sup_{t\geq 0} \E|\bar{\theta}_{\lfloor t \rfloor}^{\lambda}|^{2p}\leq \E |\theta_0|^{2p} +C'_p.
\end{equation}
\end{combinelaw}
\newtheorem{definition 2}[regularization]{Definition}
Furthermore consider a continuous-time process $\zeta_{t}^{s, v, \lambda}, t \geq s$ which is the solution to the SDE
\begin{equation} \label{zeta}
    \mathrm{d} \zeta_{t}^{s, v, \lambda}=-\lambda h\left(\zeta_{t}^{s, v, \lambda}\right) \mathrm{d} t+ \sqrt{2 \lambda\beta^{-1}} \mathrm{d} \tilde{B}_{t}^{\lambda}
\end{equation}
with initial condition $\zeta_{s}^{s, v, \lambda}:=v, v \in \mathbb{R}^{d}$. Let $T:=\lfloor 1 / \lambda\rfloor$.
\begin{definition 2}
Fix $n\in \mathbb{N}$ and define $
\bar{\zeta}_{t}^{\lambda, n}:=\zeta_{t}^{n T, \bar{\theta}_{n T}^{\lambda}, \lambda}$
where $\zeta_{t}^{n T, \bar{\theta}_{n T}^{\lambda}, \lambda}$ is defined in \eqref{zeta}.
\end{definition 2}
Henceforth, any constant denoted by $C'_{p}$, for $p\geq 1$, is given explicitly  in the proof of Lemma \eqref{pmoments}.
\newtheorem{bounds continuous}[lemma2]{Lemma}
\begin{bounds continuous} \label{V4_bound_continuous}
Let Assumptions \ref{A2} and \ref{A3} hold.
Then, for $0<\lambda<\lambda_{\max}$
\[\mathbb{E}\left[V_{4}\left(\bar{\theta}_{nT}^{\lambda}\right)\right] \leq 2(1- \lambda\frac{\kappa}{2}\eta)^{nT} \mathbb{E}|\theta_0|^4+2+2C'_2.\]
\end{bounds continuous}
\newtheorem{lemma citation}[lemma2]{Lemma}
\begin{lemma citation}\label{Lemma citation}
Let Assumption \ref{A3} holds. Then, for any $p\geq 2,$ $\theta\in \mathbb{R}^d$,
\[
\Delta V_{p}/\beta-\left\langle h(\theta), \nabla V_{p}(\theta)\right\rangle \leq-\bar{c}(p) V_{p}(\theta)+\tilde{c}(p),
\]
where $\bar{M}_{p}=\sqrt{1 / 3+4 B /(3 A)+4 d /(3 A \beta)+4(p-2) /(3 A \beta)}$ , $\bar{c}(p)=A p / 4$,   $\tilde{c}(p)=(3 / 4) A p v_{p}\left(\bar{M}_{p}\right)$, $\bar{c}(p)=A p / 4$ and $A$, $B$ are  given explicitly  in the proof.
\end{lemma citation}
\newtheorem{moments auxiliary}[lemma2]{Lemma}
\begin{moments auxiliary} \label{aux_moments}
Let Assumptions \ref{A2} and \ref{A3} hold. Then,
\begin{align*}\E \left[V_2\left(\bar{\zeta_t}^{\lambda,n}\right)\right]  \leq & \mathbb{E}\left[V_{2}\left(\theta_{0}\right)\right] +2\left(C_{X}\eta^{-1} + 2M_0^2(2 + \eta)+2d(\eta\beta)^{-1}\sqrt{\lambda_{max}}\right) \\&  +\frac{\tilde{c}(2)}{\bar{c}(2)} +1,\\ \mbox{ and }\\
\E\left[V_4\left(\bar{\zeta_t}^{\lambda,n}\right)\right] &\leq 2\mathbb{E}|\theta_0|^4+2+2C'_2 +\frac{\tilde{c}(4)}{\bar{c}(4)}.
\end{align*}
The associated constants $\bar{c}(p),\tilde{c}(p)$ come from Lemma \ref{Lemma citation} and the rest from the moment computations in Lemma \ref{pmoments}.
\end{moments auxiliary}
\subsection{Proofs of main results} \label{lemmas_main_proof}
We mainly present the proof of Theorem \ref{Thrm1}. The goal is to establish a non-asymptotic bound for
$W_1(\mathcal{L}(\theta^{\lambda}_n),\pi_{\beta})$, which can be split as follows:
$$
W_1(\mathcal{L}(\theta^{\lambda}_n),\pi_{\beta}) \leq
W_1(\mathcal{L}(\bar{\theta}^{\lambda}_n),\mathcal{L}(Z^{\lambda}_n))+W_1(\mathcal{L}(Z^{\lambda}_n),\pi_{\beta}).
$$ 
To achieve this, we introduce a functional which is associated with the contraction results in \cite{Harris} and is crucial for obtaining convergence rate estimates in $W_1$ and $W_2$. Let $\mathcal{P}_{V_2}$ denote the subset of $\mathcal{P}(\mathbb{R}^d)$ such that every $\mu \in\mathcal{P}_{V_2}$ satisfies $\int_{\mathbb{R}^d} V_2(\theta)\mu(d\theta)<\infty$. The functional $w_{1,2}$ is given by
\begin{equation} \label{seminorm}
    w_{1, 2}(\mu, \nu):=\inf _{\zeta \in \mathcal{C}(\mu, \nu)} \int_{\mathbb{R}^{d}} \int_{\mathbb{R}^{d}}\left[1 \wedge | \theta-\theta|^{\prime}\right]\left[\left(1+V_{2}(\theta)+V_{2}\left(\theta^{\prime}\right)\right) \zeta\left(\mathrm{d} \theta \mathrm{d} \theta^{\prime}\right)\right.
\end{equation}
where $\mathcal{C}(\mu, \nu)$ is defined immediately before \eqref{eq:definition-W-p}.
The functional $w_{1,2}$ is related to the Wasserstein distances in the following way:
\newtheorem{w12conn}[lemma2]{Lemma}
\begin{w12conn}\label{w12conn}
For any $\mu, \nu \in \mathcal{P}_{V_{p}}\left(\mathbb{R}^{d}\right)$, the following inequalities hold for $w_{1,2}$
$$
W_{1}(\mu, \nu) \leq w_{1,2}(\mu, \nu), \quad W_{2}(\mu, \nu) \leq \sqrt{2 w_{1,2}(\mu, \nu)} .
$$
\end{w12conn}
We can now proceed with the statement of the contraction property of the Langevin SDE \eqref{langevin-SDE} in $w_{1,2}$, which yields the desired result for $W_1(\mathcal{L}(Z^{\lambda}_n),\pi_{\beta})$.
\newtheorem{w1,2}[Remark1]{Proposition}
\begin{w1,2}\label{eberle}
Let $Z_{t}^{\prime}, t \in \mathbb{R}_{+}$ be the solution of the Langevin SDE \eqref{langevin-SDE} with initial condition $Z_{0}^{\prime}=\theta_{0}$ which is independent of $\mathcal{G}_{\infty}$ and $\left|\theta_{0}\right| \in L^{2} .$ Then,
\[
w_{1,2}\left(\mathcal{L}\left(Z_{t}\right), \mathcal{L}\left(Z_{t}^{\prime}\right)\right) \leq \hat{c} e^{-\dot{c} t} w_{1,2}\left(\mathcal{L}\left(\theta_{0}\right), \mathcal{L}\left(\theta_{0}^{\prime}\right)\right)
\]
where $w_{1,2}$ is defined in \eqref{seminorm}.
\end{w1,2}
Since the functional $w_{1,2}$ is closely related to $W_1$ and $W_2$ distances as shown in Lemma \ref{w12conn}, the statement of Proposition \ref{eberle} which is based on the results of the pivotal work in \cite{Harris}, indirectly shows the contraction behaviour in $W_1$ and $W_2$ distances.\\
The following two Lemmas combined establish the required $W_1(\mathcal{L}(\bar{\theta}^{\lambda}_n),\mathcal{L}(Z^{\lambda}_n))$ estimate.
\newtheorem{Contraction constants}[lemma2]{Lemma}
\newtheorem{lemma aux}[lemma2]{Lemma}
\begin{lemma aux} \label{Lemma 4.7}
Let Assumptions \ref{A2} and \ref{A3} hold. For $0<\lambda<\lambda_{max}$ and $t\in [nT,(n+1)T],$
\[W_{2}\left(\mathcal{L}\left(\bar{\theta}_{t}^{\lambda}\right), \mathcal{L}\left(\bar{\zeta}_{t}^{\lambda, n}\right)\right) \leq \sqrt{\lambda} \sqrt{e^{3a} (C_1+C_2+C_3)}\]
where $C_1$, $C_2$ are given explicitly in the proof.
\end{lemma aux}
The auxiliary process $\bar{\zeta}_{t}^{\lambda, n}$ plays the role of a `stepping stone' to bridge the gap between $\bar{\theta}_{t}^{\lambda}$ and $Z_{t}^{\lambda}$.
\newtheorem{Coupling}[lemma2]{Lemma}
\begin{Coupling}\label{Lemma 4.8}
Let Assumptions \ref{A2} and \ref{A3} hold. For $0<\lambda\leq \lambda_{max}$ and $t\in [nT,(n+1)T]$,
\[W_{1}\left(\mathcal{L}\left(\bar{\zeta}_{t}^{\lambda, n}\right), \mathcal{L}\left(Z_{t}^{\lambda}\right)\right)\leq \sqrt{\lambda} z_1 \]
where $z_1$ is  given explicitly in the proof.
\end{Coupling}
Thus,  in view of the above results, and the facts that $W_1(\mu,\nu)\leq w_{1,2}(\mu,\nu)$ and $\mathcal{L}(\bar{\theta}^{\lambda}_n)=\mathcal{L}(\theta_n^{\lambda})$, for each $n\in\mathbb{N}$, one obtains the results of Theorem~\ref{Thrm1}. The proof of Corollary \ref{Thrm2} follows the same lines by noticing $W_2 \leq \sqrt{2w_{1,2}}$. Full details of all the aforementioned derivations can be found in the Appendix.

Finally, the excess risk as described in \eqref{sum_T1_T2} is controlled thanks to the following two Lemmas.
\newtheorem{lemmaT1}[lemma2]{Lemma}
\begin{lemmaT1} \label{T_1}
Let  the assumptions of the main theorems hold.\\ Set $\mathcal{T}_{1}: =\E[u(\theta_n^\lambda)]-\E[u(\theta_\infty)] .$ Then,
\begin{align*}
 \mathcal{T}_{1}\leq & \left( \frac{a_1}{l+1}\sqrt{\E|\theta_0|^{2l}+C'_l} +\frac{a_1}{l+1}\sqrt{\sigma_{2l}}+r_2\right) W_2\left(\mathcal{L}\left(\theta_{n}^{\lambda}\right),\pi_\beta\right)
\end{align*}
where $a_1=2^l(\E K(X_0) +\eta)$ and $r_2=2\E K(X_0)$.
\end{lemmaT1}
\newtheorem{lemmaT2}[lemma2]{Lemma}
\begin{lemmaT2}  \label{T_2}
Let and Assumptions \ref{A2} and \ref{A3} hold. If $\beta\geq \frac{2}{A}$ and $R_0 = \inf \{y\geq  \sqrt{B/A}: \quad y^2  (1+4y)^l>\frac{d+1}{\beta L\E (1+|X_0|)^\rho} \}$, then
\[\begin{aligned}\mathcal{T}_{2}: = \E[u(\theta_\infty)]-u_*&\leq \frac{d}{2 \beta} \log \left(\frac{e K}{A}\left(\frac{B \beta}{d}+1\right)\right)-\frac{1}{\beta}\log\left(1-e^{-(R_0\sqrt{K\beta}-\sqrt{d})^2}\right).\end{aligned}\]
\end{lemmaT2}
Lemma \ref{T_1} and Lemma \ref{T_2} can be viewed as generalizations of the important work in \cite{raginsky} which decribes the connection between sampling and optimization with non-asymptotic estimates.\\
The proofs of the aforementioned Lemmas follow, in general, the proofs of the analogous results in \cite{raginsky} with certain modification to allow for the more general local Lipschitz continuity assumption (Assumption \ref{A2} and Proposition \ref{polip H}) compared to the global Lipschitz continuity assumption in \cite{raginsky}.
More specifically, in Lemma \ref{T_1} the same steps as the analogous result in \cite{raginsky} are followed while superlinear growth estimates are used (instead of linear) which are induced by the local Lipschitz continuity.
In Lemma \ref{T_2}, exploiting the fact $|\theta^*|$ can be explicitly bounded as a result of dissipativity, thus we are able to underestimate the integral $I$  with a smaller integral around $\theta^*$ where local Lispchitzness implies global Lipschitzness. This way one can bound the given integral by one related to a Gaussian distribution (the comparison with such an integral in \cite{raginsky} is straightforward because of the global Lispchitz assumption).
An application of a standard concentration inequality yields a slighlty worse upper bound of the same order with respect to inverse temperature parameter ( $\frac{\log\beta}{\beta}$).
\begin{proof}[\textbf{Proof of Theorem \ref{excess_risk}}]
Due to \eqref{sum_T1_T2}, Lemma \ref{T_1} and Lemma \ref{T_2}, the desired result is obtained.
\end{proof}

\section{Multilayer neural networks} \label{MNNs}
	
Some further notation is introduced in this section. The set $\mathbb{N}_+:=\mathbb{N}\setminus\{1\}$ and  $\id{k}$ denotes the identity operator of $\R^k$, $k\in\N$. For $k,l\in\N$, $\lin{\R^k,\R^l}$ stands for the vector space of $\R^k\to\R^l$ linear operators. In particular, $(\R^k)^\ast$ denotes $\lin{\R^k,\R}$, that is the dual space of $\R^k$. In our setting, linear functionals and vectors are identified through the inner product. Moreover,  for a fixed $v\in\R^k$, we define $M_v\in\lin{\R^k,\R^k}$ the element-wise multiplication by $v$, i.e.  $[M_v z]_l=v_l z_l$, $l=1,\ldots,k$. Furthermore, for an arbitrary $W\in\lin{\R^k,\R^l}$, $\Vert W\Vert$ stands for the corresponding operator norm, that is $\Vert W \Vert = \sup_{|z|=1}|Wz|$. Also, for an arbitrary $W\in\lin{\R^k,\R^l}$, $[W]_{ij}$ denotes the element at $ij$-th place in the matrix of $W$ with respect to the standard bases of $\R^k$ and $\R^l$.

Let $C_b(\R)$ be the space of continuous and bounded functions and $C^k_b(\R)$ denotes the subset of at least $k$-times continuously differentiable functions. The norm on $C_b(\R)$ is given by $\vinf{\sigma}:=\sup_{z\in\R}|\sigma (z)|$. Moreover, for a function $\eta:\R\to\R$, let us define the  Lipschitz constant of $\eta$ as
\[
\vlip{\eta}=\inf\{L>0\mid \forall x,y\in\R\, |\eta(x)-\eta(y)|\le L |x-y| \}.
\]
The set of those $\R\to\R$ functions for which $\vlip{.}$ is finite is denoted by $\lip{\R}$. In the sequel, we employ the convention that $\sum_{k}^{l}=0$ and $\prod_{k}^{l}=1$ whenever $k,l\in\Z$, $k>l$.

Let us fix a function $\sigma:\R\to\R$ to serve as the activation function of our neural network. We assume that $\sigma\in C^1_b(\R)$ and $\sigma'\in C_b(\R)\cap\lip{\R}$. Note that these assumptions imply the Lipschitz-continuity of $\sigma$, too. The Sobolev space $W^{1,\infty}(\R)$ is the space of Lipschitz functions moreover the norm on this space is
	$\Vert \cdot \Vert_{1,\infty} = \vinf{\cdot}+\vlip{\cdot}$, therefore
	$\sigma'\in W^{1,\infty}(\R)$ and it is natural to regard $\sigma$ as an
	element of $\sigma\in W^{2,\infty}(\R)$. The norm which we use frequently in the sequel is the
	$W^{2,\infty}(\R)$-norm of $\sigma$ that is
\[
	\vsob{\sigma} := \Vert\sigma\Vert_{2,\infty} =
	\vinf{\sigma}+\vinf{\sigma'}+\vlip{\sigma'}.
\]	
Next, we consider networks consisting of $n\in\N_+$ hidden layers, where the number of nodes in each layer is given by $(d_1,\ldots,d_n)\in\N_+^n$.
	The space of the learning parameters is
	\begin{equation*}
		\R^d\cong\Theta:=(\R^{d_n})^\ast\oplus\bigoplus_{i=1}^n\lin{\R^{d_{i-1}},\R^{d_i}},
	\end{equation*}
	where $d:=\dim (\Theta)=d_n+\sum_{i=1}^n d_i d_{i-1}$ and $d_0=m-1$ for some $m>1$ which corresponds to
	the dimension of the training data sequence.
	For the diameter of the network, we introduce the notation
	\begin{equation*} 
	D:=\max_{0\le i\le n} d_j.	
	\end{equation*}
	A general element of $\Theta$
	is of the form $\theta=(\phi,\mathbf{w})$, where $\phi\in (\R^{d_n})^\ast$ is a linear functional aggregating the node's output and $\mathbf{w}:=(W_1,W_2,\ldots,W_n)$ is the sequence of weight matrices, where $W_i\in\lin{\R^{d_{i-1}},\R^{d_i}}$, $i=1,\ldots,n$.
	The Euclidean norm on $\Theta$ is
	\begin{equation*}
		|(\phi,\mathbf{w})| = \left(|\phi|^2+\sum_{i=1}^{n}|W_i|^2\right)^{1/2}.
	\end{equation*}
Let us further introduce the notations
	\begin{equation*}
		\sigma (\mathbf{w}_i^j,\cdot) = \begin{cases}
		\sigma_{W_j}\circ\sigma_{W_{j-1}}\circ\ldots\circ\sigma_{W_i}(\cdot) & \text{if } 1\le i\le j\le n \\
		\id{d_j} & \text{otherwise,}
		\end{cases}
	\end{equation*}
	where $\sigma_{W_i}:\R^{d_{i-1}}\to\R^{d_{i}}$ is a nonlinear map given by
	$[\sigma_{W_i}(z)]_l=\sigma\left({[W_i z]_l}\right)$, $z\in\R^{d_{i-1}}$, $l=1,\ldots,d_i$, $i=1,\ldots,n$.
	\newtheorem{NNremark}[Optimization problem]{Remark}
\begin{NNremark}{\rm In our setting, seemingly, the bias is always chosen to be $0$ inside the activation function.
However, it is easy to incude a nonzero bias, too. We show this only for the first layer, for simplicity. It is
not restrictive to assume
$\sigma(1)=1$ and we will add a $0$th coordinate $z_{0}=1$ to the the input vector $\mathbf{z}$.
We wish to obtain the output $\sigma(a_{i}^{T}\mathbf{z}+b_{i})$, $i=1,\ldots,d_{1}$ from the first layer
with $a_{i}\in \mathbb{R}^{d_{0}}$ and with biases $b_{i}\in\mathbb{R}$.
To this end, we should define a $(d_{1}+1)\times (d_{0}+1)$
matrix $W_{1}$ whose $i$th row is $(b_{i},a_{i})\in\mathbb{R}^{d_{0}+1}$, $i=1,\ldots,d_{1}$ and whose $0$th row is $(1,0,\ldots,0)$.
In this way $[\sigma_{W_{1}}((1,\mathbf{z}))]_{i}=\sigma(a_{i}^{T}z+b_{i})$ for $i=1,\ldots,d_{1}$
and $[\sigma_{W_{1}}((1,\mathbf{z}))]_{0}=1$.
It is clear that the construction can be continued for arbitrarily many layers.
Thus, it doesn't affect our calculations.}
\end{NNremark}
	Let $\mathbf{z}:=(z_{1},\ldots,z_{d_{0}})\in\mathbb{R}^{m-1}$ represent an input vector.
	With this, the function computed by a neural network with the above characteristics is given by
	$f:\Theta\times\R^{m-1}\to\R$
	\begin{equation}\label{eq:fdef}
	f((\phi,\mathbf{w}),\mathbf{z}) := \phi \left(\sigma(\mathbf{w}_1^n,\mathbf{z})\right)
	\end{equation}
	For all $r>0$ and $\eta>0$, we define the regularized empirical risk function
	$U:\Theta\times\R^m\to [0,\infty)$ such that
	\begin{equation}\label{eq:objective}
	U(\theta,x):=(y-f(\theta,\mathbf{z}))^{2}+\frac{\eta}{2(r+1)}|\theta|^{2(r+1)},
	\end{equation}
	where we used the simpler notation for the input $x:=(\mathbf{z},y)$. The second term in \eqref{eq:objective} serves to regularize the optimization problem.
	We seek to optimize the parameter $\theta$ in such a way that, for some $r>0$ and $\eta>0$, $\theta\mapsto u(\theta):=E[U(\theta,X)]$ is minimized where
	$X=(\mathbf{Z},Y)\in\R^{m}$ is a pair of random variables, $\mathbf{Z}$ representing the input
	and $Y$ the target. The target variable $Y$ is assumed one-dimensional for simplicity. For the derivative of $U$ with respect to the learning parameter, the following notation is used
	\begin{equation}\label{eq:H}
		H(\theta,x):=\partial_\theta U(\theta,x) = -2(y-f(\theta,\mathbf{z}))\partial_\theta f(\theta,\mathbf{z}) + \eta |\theta|^{2r}\theta,
	\end{equation}
	where we refer to the first term in the sequel as $G:\Theta\times\R^m\to\Theta^\ast\cong\R^d$. Thus,
	\begin{equation}\label{wq:G}
		G(\theta,x):= -2(y-f(\theta,\mathbf{z}))\partial_\theta f(\theta,\mathbf{z}).
	\end{equation}
Further, it is shown that within the framework of \eqref{eq:objective} and \eqref{eq:H}, Assumptions \ref{A2} and \ref{A3} hold.
\newtheorem{propositionnn}[Remark1]{Proposition}
\begin{propositionnn} \label{A2_for_ANN}  Assumption \ref{A2} is satisfied by $G$, which is given in \eqref{wq:G}. In particular,
\begin{equation*}
|G(\theta,x)-G(\theta',x)| \le  L_1(1+|x|)^{\rho}(1+|\theta|+|\theta'|)^{q-1}
	|\theta-\theta'|, \mbox{ for all } x \in \mathbb{R}^{m}
\end{equation*}
and $\theta$, $\theta^{\prime} \in \mathbb{R}^d$, where $L_1=16(n+1)D^{3/2}(1+\vsob{\sigma})^{2n+4}$, $\rho = 3$ and $q-1= 2n+1$.
\end{propositionnn}
\newtheorem{remarkkk}[Optimization problem]{Remark}
\begin{remarkkk}\label{remarkkk}
Assumption \ref{A3} is trivially satisfied in the context of neural networks when $X_0$ has either bounded support or a distribution with enough bounded moments. Similarly, the initialization of the algorithm is chosen appropriately either by using deterministic values or samples from distributions with enough bounded moments.
\end{remarkkk}

Thus, the main results of this paper, namely Theorem \ref{Thrm1}, Corollary \ref{Thrm2} and, most importantly, Theorem \ref{excess_risk} hold true in this setting.

\section{Examples}

The purpose of this section is twofold. On one hand,
we present here a simple one-dimensional
optimization problem for which our method outperforms the usual unadjusted Langevin dynamics and even
the ADAM optimizer.

On the other hand, we highlight the relevance of our results to neural networks
by showing a toy example, where dissipativity of the algorithm fails
for quadratic regularization which further supports our claim that higher-order
regularization is needed. We conclude the section with a real-world example on image classification.

\subsection{Experiment: A comparison between SGLD, ADAM and TUSLA}\label{sec:ADAM}

In this point, we present an example where both ADAM and the usual SGLD algorithm fail to find
the optimum but TUSLA converges rapidly to it. Let $(X_n)_{n\ge 1}$ be an i.i.d. sequence such that $X_1\sim\mathcal{U}([0,11])$. We consider the following parametric family of objective functions
\begin{equation}\label{eq:obj}
u_s (\theta) = \begin{cases}
\frac{1}{22}(\theta-0.1)^2 + (\theta-0.1)^{2s} & \text{ if } |\theta-0.1|\le 1 \\
\frac{1}{11}\left(|\theta-0.1|-\frac{1}{2}\right)
+ (\theta-0.1)^{2s} & \text{ if } |\theta-0.1|>1
\end{cases},\,\, s\ge 0.
\end{equation}
It is easy to see that $\theta_{\ast}=0.1$ is the global minimum of $u_s$ for $s\ge 0$ (See Figure \ref{fig:obj})
\begin{figure}[!h]
	\centering
	\includegraphics[width=0.75\linewidth]{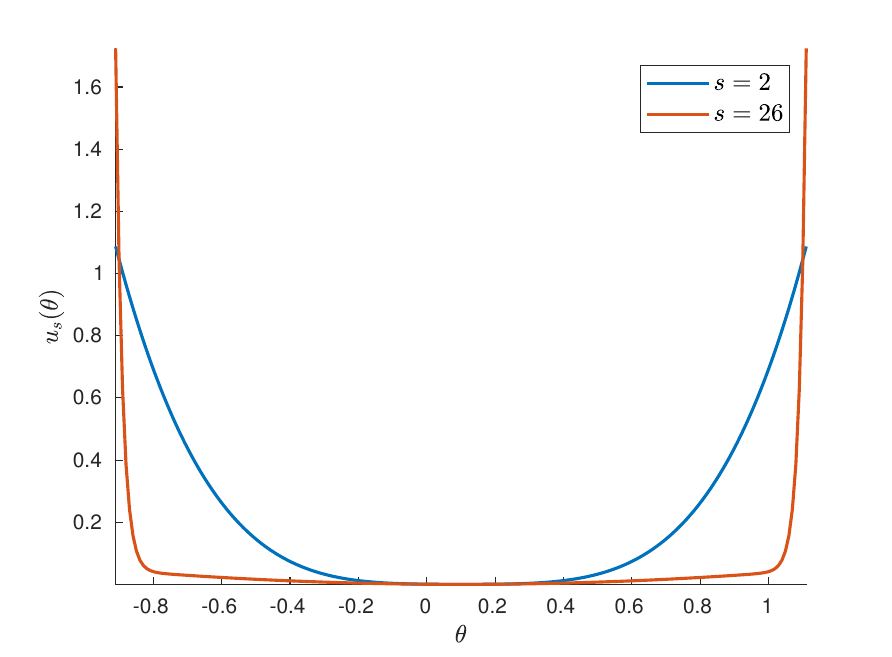}
	\caption{The objective function for $s=2$ and $s=26$.}
	\label{fig:obj}
\end{figure}

\noindent
Furthermore,
\begin{equation}\label{eq:obj_sg}
G_s (\theta,x) = (12\times \ind_{x\in [0,1]}-1)\times
\begin{cases}
(\theta-0.1)^2 + 2s (\theta-0.1)^{2s-1} & \text{ if } |\theta-0.1|\le 1 \\
|\theta-0.1|-\frac{1}{2}
+ 2s (\theta-0.1)^{2s-1} & \text{ if } |\theta-0.1|>1
\end{cases}
\end{equation}
is an unbiased estimate $u'_s$ that is $\E [G_s (\theta,X_1)] = u_s'(\theta)$, $\theta\in\R$.
Note that $G_s(\theta,x)$ is discontinuous in $x\in\R$ but satisfies polynomial Lipschitz continuity in $\theta \in R$ and thus Assumption \ref{A2} is in force. Since $X_1$ has
bounded support, Assumption \ref{A3} trivially holds.

\smallskip
We did a comparison between SGLD without adaptive step size, ADAM and TUSLA.
The parameter update in the unadjusted stochastic gradient Langevin dynamics is given by
\begin{equation}\label{eq:SGLD}
\theta_{n+1}:=\theta_n-\la G(\theta_n,x_{n+1})+\sqrt{\frac{2\la}{\beta}}\xi_{n+1},\ n\in\N,
\end{equation}
where $(\xi_n)_{n\in\N}$ is an i.i.d sequence of standard Gaussian random variables, $\lambda$
is the step size and $\beta$ is the so-called inverse temperature parameter.

\smallskip
ADAM (an abbreviation for Adaptive Moment Estimation) is a variant of stochastic gradient descent presented first in \cite{kingma2015adam}. Despite its raising popularity to solve deep learning problems, practitioners  started noticing that in some cases ADAM performs worse than the original SGD. Several research papers are devoted to the mathematical analysis of ADAM and other ADAM-type optimization algorithms. See, for example \cite{Barakat2019ConvergenceAD} and \cite{chen2019convergence}.
The main idea of ADAM is that the algorithm calculates an exponential moving average of the gradient and the squared gradient making it robust against discontinuous and noisy stochastic gradients. In ADAM (See Algorithm \ref{alg:ADAM}), the closer $\beta_{1}$ and $\beta_{2}$ to $1.0$, the smaller is the bias of moment estimates towards zero.
\begin{algorithm2e}[h]
	\SetKwInOut{Input}{input}
	\SetKwInOut{Output}{output}
	\SetKwInOut{Parameter}{parameter}
	
	\BlankLine
	\Input{$\theta_0$ (Initial value)}
	\Input{$(x_m)_{m\ge 1}$ (i.i.d. random numbers drawn from a $\mathcal{U}([0,11])$ distribution)}
	
	\Output{$\theta_n$ (Approximation of the global minimum of the optimum)}
	
	\BlankLine
	\Parameter{$\alpha$ (Step size)}
	\Parameter{$\beta_{1},\beta_{2}\in [0,1)$ (Exponential decay rates for the moment estimates)}
	\Parameter{$\eps$ (Small positive scalar to avoid division by zero)}
	
	\BlankLine
	$m_0$ $\leftarrow$ $0$ (Initialize $1^{\text{st}}$ moment)\\
	$v_0$ $\leftarrow$ $0$ (Initialize $2^{\text{nd}}$ moment)\\
	$n$ $\leftarrow$ $0$ (Initialize timestep)
	
	\BlankLine
	\While{$\theta_n$ not converged}{
		$m_{n+1} \leftarrow \beta_{1} m_n + (1-\beta_{1})\times G (\theta_n, x_{n+1})$ (Update biased first moment estimate)\\
		
		$v_{n+1} \leftarrow \beta_{2} v_n + (1-\beta_{2})\times G (\theta_n, x_{n+1})^2$ (Update biased second raw moment estimate)\\
		
		$\hat{m}_{n+1} \leftarrow \frac{m_{n+1}}{1-\beta_{1}^{n+1}}$ (Compute bias-corrected first moment estimate)\\
		
		$
		\hat{v}_{n+1} \leftarrow \frac{v_{n+1}}{1-\beta_{2}^{n+1}} $ (Compute bias-corrected second raw moment estimate)\\
		
		$\theta_{n+1} \leftarrow \theta_n - \frac{\alpha \hat{m}_{n+1}}{\sqrt{\hat{v}_{n+1}}+\eps}$ (Update parameters)
	}
	
	\BlankLine
	\BlankLine	
	\Return{$\theta_n$}
	
	\BlankLine	
	\BlankLine	
	\caption{ADAM algorithm for stochastic optimization}\label{alg:ADAM}
\end{algorithm2e}

\smallskip
It is worth mentioning that the TUSLA iteration scheme in its original form (See equation \eqref{eq:TUSLA} and \eqref{def: H_lambda}) may cause overflow error on computers
because of the limitation of the floating-point arithmetic. To be more precise, let us consider the definition of $H_\lambda$. Since in the expression of $H_\lambda$, both the numerator and denominator contains $|\theta|^{2r}$, it is quite common that during the iteration, $|\theta_n|^{2r}$ exceeds the numeric limit of the floating point type used and thus resulting NaN. To overcome this issue, we compute $H_\la (\theta,x)$ as follows:
\begin{equation*}
H_\la (\theta,x) = \begin{cases}
\frac{G(\theta,x)+\eta \theta^{2r+1}}{1+\sqrt{\la}\theta^{2r}} & \text{ if } |\theta|<1 \\[0.75em]
\frac{\theta^{-2r} G(\theta,x)+\eta \theta}{\theta^{-2r}+\sqrt{\la}} & \text{ if } |\theta|\ge 1.
\end{cases}	
\end{equation*}

\smallskip
In numerical experiments, both in SGLD and in TUSLA, we set $\la = 0.05$, $\beta = 0.05$, $\eta = 0.01$ and $r=s+10$, where $s$ is as in the definition of $u_s$ (See \eqref{eq:obj}).
Furthermore, in ADAM, we set the step size to $\alpha=10$, and used parameter values proposed by authors in \cite{kingma2015adam} i.e. $0.9$ for $\beta_{1}$, $0.999$ for $\beta_{2}$, and $10^{-8}$ for $\eps$.

As initial value, we used $\theta_0 = 10^3$, simulated $10^4$ time steps, and studied the convergence of these three algorithms when $s=2$ and $s=26$ in \eqref{eq:obj}. We found that the SGLD algorithm rapidly diverges in all cases after 1-5 steps. Figure \ref{fig:ADAMvsTUSLA} shows
that under these parameter settings, for $s=2$, ADAM and TUSLA perform equally well (See Figure \ref{fig:res1}.).
However, interestingly, when we increase $s$ to $26$, ADAM become practically non-convergent
but surprisingly, TUSLA approaches $\theta_{\ast}=0.1$ as fast as before (See Figure \ref{fig:res2}.).
\begin{figure}[!h]
	\centering
	\begin{subfigure}{0.45\textwidth}
		\includegraphics[width=\textwidth]{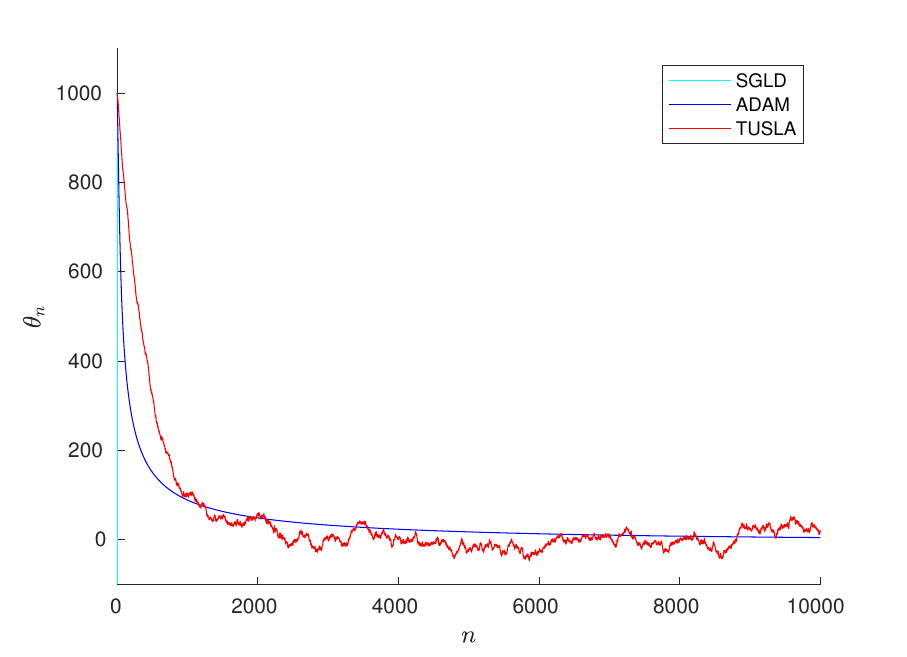}
		\caption{$s=2$}\label{fig:res1}
	\end{subfigure}
	~
	\begin{subfigure}{0.45\textwidth}
		\includegraphics[width=\textwidth]{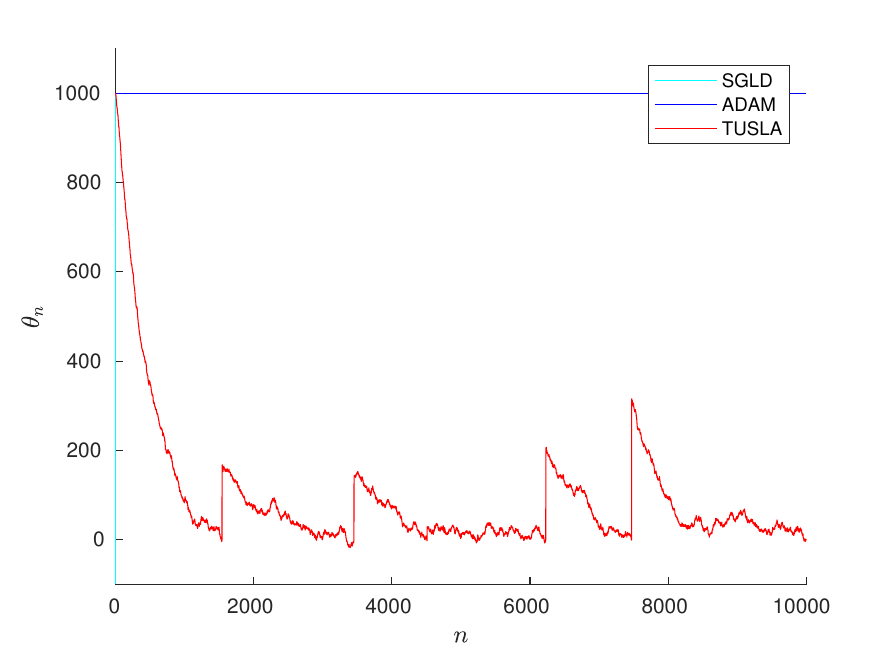}
		\caption{$s=26$}\label{fig:res2}
	\end{subfigure}
	\caption{Finding the global minimum of $u_s$ with ADAM (blue line) and with TUSLA (red line), where, in the first experiment $s=2$ (left), and in the second experiment $s=26$ (right).}\label{fig:ADAMvsTUSLA}
\end{figure}
We attempted to modify the parameters used in ADAM making the iteration convergent. Actually, we varied the learning parameter $\alpha$ between $10^{-2}$ and $10^{10}$, the forgetting factors $\beta_1,\beta_2$ in $(0,1)$, and tried several different combination but the result was the same as in Figure \ref{fig:res2}.

\subsection{One-layer neural network without dissipativity} \label{example_non_diss}

Let us define, for simplicity, $\sigma(x):=\mathrm{arctan}(x)$ but the example
would work equally well with a sigmoidal activation function. Define
$f(w_1,w_2,x,y):=(y-w_2\sigma(w_1 x+S))^2+\eta (w_{1}^{2}+w_{2}^{2})$ where $(w_{1},w_{2})$
are the parameters, $\eta>0$ is a given weight for the regularization term, $x,y$ are the data points
and $S$ is a constant to be specified later.
This is a one-layer neural ``network''
with one neuron where $(w_{1},w_{2})$ need to be tuned to find an optimal approximation of
$y$ as a function of $x$. Needless to say that everything would work with several neurons and layers, too.
Then
\begin{eqnarray*}
\partial_{w_1} f &=& 2(y-w_2\sigma(w_1x +S))(-w_2)\sigma'(w_1 x+S)x+2\eta w_{1},\\
\partial_{w_2} f &=& 2(y-w_2\sigma(w_1 x+S))(-\sigma(w_1 x+S))+2\eta w_{2}.
\end{eqnarray*}
Now let $y=0$, $x=1$. Then we get that at such a data point,
$$
w_1\partial_{w_1} f=2w_2^2w_1\sigma(w_1 +S)\sigma'(w_1 +S)+2\eta w_{1}^{2},
\ w_2\partial_{w_2} f=2w_2^2\sigma^2(w_1 +S)+2\eta w_{2}^{2}.
$$
Let us notice that $2$ is a bound for both $|\sigma|,\sigma'$.{}
Also, $\sigma(-\pi/4)=-1$ and $\sigma'(-\pi/4)=1/(1+\pi^{2}/16)$.
Choose $w_{1}:=(4+\eta+1)(1+\pi^{2}/16)$ and set $S:=-w_{1}-1$.

We can then check that
$$
\langle \nabla f(w_{1},w_{2},1,0), (w_{1},w_{2})\rangle\leq -2w_{2}^{2}+2\eta w_{1}^{2},
$$
hence dissipativity cannot hold since it would require
$$
\langle \nabla f(w_{1},w_{2},1,0), (w_{1},w_{2})\rangle\geq \alpha (w_{1}^{2}+w_{2}^{2})-\beta
$$
for all $w_{1},w_{2}$ with some $\alpha,\beta>0$.
The stochastic gradient Langevin algorithm on this non-dissipative
problem is expected to diverge due to the lack of dissipativity, as simple numerical simulations
readily confirm.

\subsection{Image classification}
 We conduct image classification on Fashion MNIST dataset  consisting of a training set of 60,000 images and a test set of 10,000 images. Each sample of the dataset $\left(z_i\right)_{i=1}^{60,000}$ is a $28 \times 28$ pixel image, i.e. $z_i \in \mathbb{R}^{784}$, and is assigned to one of 10 different labels $l_i \in\{0,1, \ldots, 9\}$ describing T-shirt (0), Trouser (1), Pullover (2), Dress (3), Coat (4), Sandal (5), Shirt (6), Sneaker (7), Bag (8), and Ankle boot (9). Then, the label variables are converted to vectors such that $y_i=\left[y_{i, 0}, y_{i, 1}, \ldots, y_{i, 9}\right]^{\top} \in \mathbb{R}^{10}$ with $y_{i, j}=\mathbf{1}_{\{j=l_i\}}, j=0,1, \cdots, 9, i=1, \ldots, 60,000 .$

For image classification, we consider the following SLFN with 50 neurons given by
\begin{equation}\label{eq-SLFN}
    \mathfrak{N}(\theta, z):=W_2 \sigma_1\left(W_1 z+b_1\right)+b_2, \quad \text{(SLFN)}
\end{equation}
where $\theta=(W_1,W_2,b_1,b_2)$, $W_1\in \mathbb{R}^{50x384}$, $W_2\in \mathbb{R}^{10x50}$, $b_1\in \mathbb{R}^{50}$, $b_2\in \mathbb{R}^{50}$, $\sigma_1$ the Sigmoid activation function.
We also consider a TLFN with 50 neurons on each hidden layer, which is defined by
\begin{equation}\label{eq-TLFN}
    \mathfrak{N}(\theta, z):=W_5 \sigma_1\left(W_4 \sigma_1\left(W_3 z+b_3\right)+b_4\right)+b_5, \quad \text { (TLFN) }
\end{equation}
where $\theta=\left(W_3, W_4, W_5, b_3, b_4, b_5\right), W_3 \in \mathbb{R}^{50 \times 784}, W_4 \in \mathbb{R}^{50 \times 50}, W_5 \in \mathbb{R}^{10 \times 50}, b_3 \in \mathbb{R}^{50}, b_4 \in \mathbb{R}^{50}$, $b_5 \in \mathbb{R}^{10}$ and $\sigma_1$ is the Sigmoid activation function. Therefore, we have $d=39,760$ for the SLFN and $d=42,310$ for the TLFN. Furthermore, the cross entropy loss is used, which is given by $\ell(u, v)=$ $-\sum_{i=1}^{10} u_i \log( \operatorname{softmax}( v)_i)$ for $u=\left[u_1, u_2, \cdots, u_{10}\right]^{\top} \in \mathbb{R}^{10}, v=\left[v_1, v_2, \cdots, v_{10}\right]^{\top} \in \mathbb{R}^{10}$ and $\operatorname{softmax}$ is given by
\[
p_i=\operatorname{softmax}(v)_i=\frac{e^{v_i}}{\sum_{j=1}^{10} e^{v_j}}.
\]
Essentially we are going to solve the following optimization problem:
\[\text { minimize } \quad \mathbb{R}^d \ni \theta \mapsto u(\theta):=\mathbb{E}[\ell(Y, \mathfrak{N}(\theta, Z))]+\frac{\eta}{2(r+1)}|\theta|^{2(r+1)}\]
where $R$ is given by \eqref{eq-SLFN} or \eqref{eq-TLFN} and
$\eta$ is fixed to $10^{-5}$ for all experiments. The models are trained for 200 epochs with 128 batch size.
For ADAM and AMSGrad,
we search the optimal learning rate between $\{0.01, 0.001\}$ and set $\epsilon = 10^{-8}$, $\beta_1 = 0.9$, $\beta_2 = 0.999.$ For
RMSprop, the learning rate is chosen from $\{0.01, 0.001\}$, where $\beta = 0.99$ and $\epsilon = 10^{-8}$ are fixed. For
TUSLA, we use $\lambda = 0.5$, $r = 0.5$, and $\beta = 10^{12}$ throughout the experiment.
Also, we decay the initial learning rate by 10 after 150 epochs.
\newline
\subsubsection{Performance of TUSLA by switching different hyperparameters}
\begin{table}[h!]
    \centering
    \begin{tabular}{c|cccc}
\hline$\beta$ & $10^4$ & $10^8$ & $10^{10}$ & $10^{12}$ \\
\hline test accuracy & $61.03$ & $87.77$ & $87.72$ & $87.78$ \\
\hline
\end{tabular}
    \caption{Table for different values of $\beta$ in TUSLA-SLFN}
    \label{tab:beta}
\end{table}
As expected from Theorem \ref{excess_risk} we witness that an increase in $\beta$ improves the performance in our optimizer.

\begin{table}[h!]
    \centering
    \begin{tabular}{c|cccc}
\hline$\eta$ & $10^{-5}$ & $10^{-4}$ & $10^{-3}$ & $10^{-2}$ \\
\hline test accuracy & $87.78$ & $85.62$ & $81.32$ & $72.6$ \\
\hline
\end{tabular}
    \caption{Different values of $\eta$- TUSLA SLFN}
    \label{tab:eta}
\end{table}
We investigate the impact of $\eta$, which controls the magnitude of the regularization term
on test accuracy. When the regularized term is incorporated in optimization problems,
overfitting can be reduced by forcing the neural network to have smaller values of its parameters which
leads to a simpler model. On the other hand, the deviation between the regularized and original objective could lead to a worse performance of the model. It is interesting to see the loss of test accuracy in the SLFN model when $\eta=10^{-2}$ compared to the cases $\eta=10^{-5}$ and $\eta=10^{-4}.$

\begin{table}[h!]
    \centering
    \begin{tabular}{c|cccc}
\hline$r$ & $0.5$ & $1$ & $2$  \\
\hline test accuracy & $87.78$ & $85.195 $& $72.32$ \\
\hline
\end{tabular}
    \caption{Different values of $r$- TUSLA SLFN}
    \label{tab:r}
\end{table}
The hyperparameter $r \geq 0.5$ controls the intensity of the taming function of TUSLA. We conduct experiments with $\lambda=0.5, \bar{\beta}=10^{12}$, and different $r \in\{0.5,1,2,3\}$ and summarize the results in Table \ref{tab:r} . It turns out that the choice of an appropriate $r$ is a crucial factor for the performance of TUSLA. It is encouraged to gradually increase $r$, as a large $r$ can excessively suppress the gradient part in the formula of TUSLA.\\
It is crucial to take into account that the $(\eta,r)$ regularization is mostly needed in the absence of dissipativity. The restriction $r\geq \frac{q}{2}+1$ has been imposed to produce a dissipativity property under worse-case bounds. In practice, it is very possible that a smaller $r$  is needed for optimal performance.\\
\begin{table}[h!]
    \centering
    \begin{tabular}{c|cccccc}
\hline$\lambda$ & $0.5$ & $0.1$ & $0.05$ & $0.01$ & $0.005$  \\
\hline test accuracy & $87.78$ & $87.60$ & $86.52$ & $84.57$ & $83.58$ \\
\hline best epoch & 200 & 292 & 275 &475 & 441  \\
\hline
\end{tabular}
    \caption{Different values of $\lambda$-TUSLA SLFN}
    \label{tab:l}
\end{table}
We see that there is a small difference in the test accuracy for $\lambda=0.1$ and $\lambda=0.5$. For $\lambda=0.5$ we obtain the highest accuracy the quickest.
\subsubsection{Comparison between algorithms}
\begin{table}[h!]
    \centering
    \begin{tabular}{c|cc}
\hline Dataset Model & Fashion MNIST SLFN & Fashion MNIST TLFN \\
\hline TUSLA & $\mathbf{8 7 . 7 8}$ & $\mathbf{8 8 . 1 8}$ \\
ADAM & $87.65$ & $87.26$ \\
AMSgrad & $87.45$ & $87.13$ \\
RMSprop & $87.93$ & $87.99$ \\
\hline
\end{tabular}
    \caption{Comparison- best test accuracy}
    \label{tab:my_label}
\end{table}


\begin{figure}[htbp]
    \begin{subfigure}[b]{0.49\textwidth}
     \includegraphics[scale=0.45]{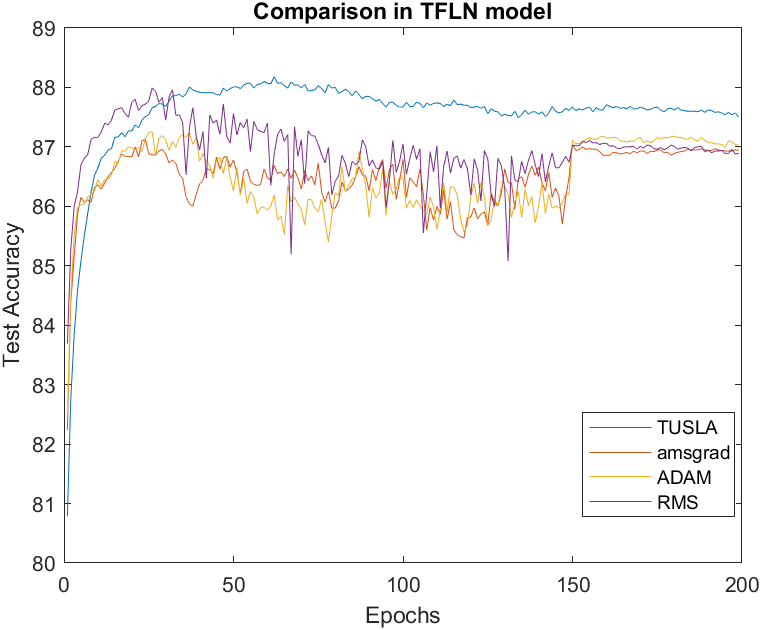}
    \caption{Performance curve in TLFN model }
    \label{fig:my_label1}
    \end{subfigure}
    \begin{subfigure}[b]{0.49\textwidth}
    \includegraphics[scale=0.45]{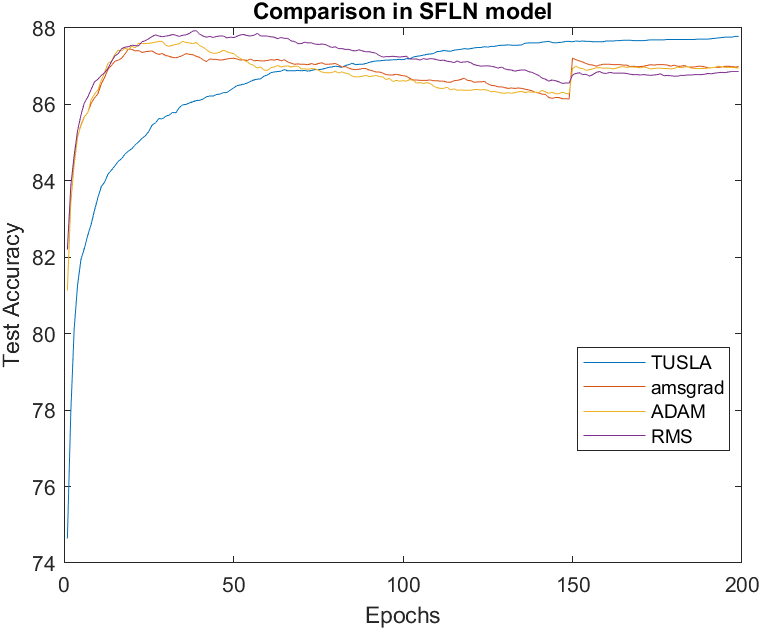}
    \caption{Performance curve in SLFN model}
    \label{fig:my_label2}
    \end{subfigure}
    \end{figure}


From the comparison table and the figures we witness that performance of TUSLA is comparable to the other algorithms and even marginally outscores them in both models. In addition, the performance curve of TUSLA is quite stable which is quite expected from an algorithm based on a tamed numerical scheme.
\newpage
\newtheorem{lemma1-experiment}{Lemma}
\begin{lemma1-experiment}
    In both models, the loss function $l(Y,\mathcal{R}(\theta,z))$ satisfies Assumption \ref{A2}.
\end{lemma1-experiment}
\begin{proof}
We shall do the proof for the SLFN model. The proof for the TFLN model follows in a similar way.
By standard calculations one obtains that
\begin{equation}\label{eq- proderiv growth}
    |\frac{\partial l}{\partial v_i}|=|\operatorname{softmax}(v)_i-u_i|\leq (1+|u_i|)
    \end{equation}
and  \begin{equation}\label{eq-proderiv Lip}
    \begin{aligned}
       &|\frac{\partial l}{\partial v_i v_j}|\leq p_ip_j\leq 1 \quad  &i\neq j \quad
       \\& |\frac{\partial l}{\partial v_i v_j}|\leq p_i(1-p_i)\leq 1 \quad & i=j
    \end{aligned}
\end{equation}
This means that the gradient of $l$ with respect to $v$ is Lispchitz and bounded.
In addition, since $\sigma_1(x)=\frac{e^x}{1+e^x}$ is easy to see that its derivative is also bounded and Lipschitz.
In addition, for the sigmoid activation function direct calculations yield
\[|\sigma_1|_{2,\infty}\leq 1\]
so $\sigma_1$ is Lipschitz and bounded.
We are now ready to analyse the behaviour of the partial derivatives of $R(\theta,z)$ with respect to $\theta.$
Since
\begin{equation}
    R^i(\theta,z)=\sum_{j=1}^N W_2^{(i,j)} \sigma_1((W_1z+b_1)_j)+(b_2)_i
\end{equation}
then,\[|\frac{\partial R^i}{\partial W_2^{i,k}}|=|\sigma_1((W_1 z+b_1)_k)|\leq 1\]
and
\[|\frac{\partial R^i}{\partial W_2^{k,j}}(\theta,z)-\frac{\partial R^i}{\partial W_2^{i,k}}(\theta',z)|\leq |(W_1z+b_1)_k-(W_1'z+b_1')_k|\leq (1+|z|)|\theta-\theta'|.\]
In addition,
\[|\frac{\partial R^i}{\partial W_1^{i,k}}|= |W_2^{(i,k)} \sigma_1^{\prime} ((W_1z+b_1)_k)z_k|\leq (1+|\theta|)(1+|z|)\]
and
\[|\frac{\partial R^i}{\partial W_1^{i,k}}(\theta,z)-\frac{\partial R^i}{\partial W_1^{i,k}}(\theta ',z)|\leq (1+|\theta|+|\theta'|)(1+|z|)^2 |\theta-\theta'|.\]
It is easy to see that the linear growth and Lipschitzness also holds for the derivatives with respect to $b_1$ and $b_2$ so
bringing all together, one obtains that
\begin{equation}\label{eq-Ri growth}
    |\frac{\partial R^i}{\partial \theta_i}(\theta,z)|\leq L_1 (1+\theta)(1+|z|)
\end{equation}
and
\begin{equation}\label{eq-Ri pol-Lip}
   |\frac{\partial R^i}{\partial \theta_i}(\theta,z)-\frac{\partial R^i}{\partial \theta_i}(\theta',z)|
   \leq L_2 (1+|\theta|+|\theta'|)(1+|z|)^2 |\theta-\theta'|.
\end{equation}
Using the chain rule one obtains for the function $G=l(Y,R(\theta,z))$,combining \eqref{eq-Ri growth}, \eqref{eq-Ri pol-Lip}, \eqref{eq-proderiv Lip} and \eqref{eq- proderiv growth} leads to
\[\begin{aligned}
   | \frac{\partial G^i}{\partial \theta_i}(\theta,y)-\frac{\partial G^i}{\partial \theta_i}(\theta',y)|&=|\frac{\partial l}{\partial R^i}(\theta,z)\frac{\partial R^i}{\partial \theta_i}(\theta,z)-\frac{\partial l}{\partial R^i}(\theta',z)\frac{\partial R^i}{\partial \theta_i}(\theta',z)|
   \\&\leq \left|\frac{\partial l}{\partial R^i}(\theta,z)\left(\frac{\partial R^i}{\partial \theta_i}(\theta,z)-\frac{\partial R^i}{\partial \theta_i}(\theta',z)\right)\right|
   \\&+\left|\frac{\partial R^i}{\partial \theta_i}(\theta',z) \left( \frac{\partial l}{\partial R^i}(\theta,z) -\frac{\partial l}{\partial R^i}(\theta',z)\right)\right|\\&\leq
   L_2'(1+|R^i|)(1+|y|)(1+|\theta|+|\theta'|)|\theta-\theta'|
   \\&+ L_1'(1+|y|) (1+|\theta'|)(1+|z|)|R^i(\theta,z)-R^i(\theta',z)|
   \\&\leq L(1+|y|) (1+|\theta|+|\theta'|)^2 (1+|z|) |\theta-\theta'|.
\end{aligned}\]
\end{proof}
\newtheorem{lemma2-experiment}{Lemma}
\begin{lemma2-experiment}
In both models, the loss function $l(Y,\mathcal{R}(\theta,z))$ satisfies Assumption \ref{A3}.
\end{lemma2-experiment}
\begin{proof}
    See Remark \ref{remarkkk}.
\end{proof}

\section{Conclusions}
We introduce a new sampling algorithm, namely TUSLA \eqref{eq:TUSLA},
which can be used within the context of empirical risk minimization for neural networks.
It does not have the stability shortcomings of other SGLD algorithms and our experiments
demonstrate this important discovery. We also provide nonasymptotic estimates for TUSLA which
explicitly bound the error between the target measure and its law in Wasserstein-$1$ and $2$ distances.
Convergence rates and explicit constants are provided too.

\newpage

\appendix

\footnotesize
\section{Proofs}

\subsection{Complementary details to Section \ref{Assumptions}}

\begin{remark}
By Assumption \ref{A2}, since the function \[\phi_{i,h}=\frac{|G(\theta,X_0)-G(\theta+he_i,X_0)|}{h}\]
can be dominated for all $i=1,\dots d$ ,$h<1$ by the random variable $Z=L_1(1+|X_0|)^\rho(2+2|\theta|)^{q-1}$ and $\E(Z)<\infty$, using a dominated convergence argument it can be concluded that partial derivation and expectation can be interchanged.
As a result, $ g\in C^1$ and consequently $h \in C^1.$
\end{remark}

\begin{proof}[\textbf{Proof of Remark \ref{remark2.3}}]
By setting $\theta'=0$ in Assumption \ref{A1} it is easy to see that
\[|G(\theta,x)-G(0,x)|\leq L_1(1+|x|)^{\rho}(1+|\theta|)^{q-1}|\theta|\leq L_1(1+|x|)^{\rho}(1+|\theta|)^{q}\leq 2^{q-1}L_1(1+|x|)^{\rho}(1+|\theta|^q)\]
which leads to
\[\begin{aligned}
   |G(\theta,x)|&\leq |G(\theta,x)-G(0,x)|+|G(0,x)|\\&\leq 2^{q-1}L_1(1+|x|)^{\rho}(1+|\theta|^q)+|G(0,x)|\\&\leq \left(2^{q-1}L_1(1+|x|)^{\rho}+|G(0,x)|\right)(1+|\theta|^q).
\end{aligned}\]
\end{proof}
\begin{proof}[\textbf{Proof of Remark \ref{remark2}}]
In view of Remark \ref{remark2.3} there holds
\[\E \langle \theta,G(\theta,X_0)\rangle\geq -\E |\theta||G(\theta,X_0)|\geq -\E |\theta|
|K(X_0)|(1+|\theta|^q).\]
As a result,
\[\langle \theta,h(\theta)\rangle \geq \eta |\theta|^{2r+2} -\E|K(X_0)||\theta|(1+|\theta|^q) \]
For the last claim, one observes that it suffices to show
\begin{equation}\label{ineq_A_and_B}
\eta |\theta|^{2r+2}-\E[K(X_0)]|\theta|(1+|\theta|^q) \ge  A|\theta|^2-B
\end{equation}
for some suitable $A$ and $B$ or, equivalently,
\[
\eta |\theta|^{2r+2} + B  \ge  A|\theta|^2 + \E[K(X_0)]|\theta|(1+|\theta|^q).
\]
Thus, setting $A= \E[K(X_0)]$ yields that \eqref{ineq_A_and_B} is satisfied with $B= \left(3\E[K(X_0)]\right)^{q+2}\eta^{-q-1}$.
\end{proof}
\begin{proof}[\textbf{Proof of Proposition \ref{Proposition 1}}]
Denote $H_{\bar{g}}$ the Hessian with respect to the antiderivative $\bar{g}$ of $g$ and  $H_{reg}$ the Hessian of the antiderivative of the regularization part. Then, the Hessian with respect to the antiderivative of $h$ is
\[H=H_{reg}+H_{\bar{g}}.\]
Let $x\in \mathbb{R}^d$. Then, since $H_{\bar{g}}$ is a symmetric matrix, it has real eigenvalues. Denote $\lambda_1(x)$ the smallest eigenvalue and $u_x$ its unit eigenvector.
One notes initially that due to the polynomial Lipchitzness of $g$,
\[|g(x+hu_x)- g(x)|\leq L_2 (1+|x|+|x+hu_x|)^{(q-1)}h\]
where $L_2=L_1\E(1+|X_0|)^{\rho}.$\\
This implies that, since $u_x$ is a unit vector,
        \[\langle g(x+hu_x)-g(x), u_x\rangle \geq -  L_2 (1+|x|+|x+hu_x|)^{(q-1)}h.\]
        Since $H_{\bar{g}}$ equals the Jacobian of the vector valued function $g$ by using a Taylor approximation for $g(x+hu_x)$ for small $h$, one obtains
        \[g(x+hu_x)-g(x)= H_{\bar{g}}(x) (hu_x) + o(h)\]
        which by the eigenvector property of $u_x$ is equivalent to \[g(x+hu_x)-g(x)= h\lambda_1(x)u_x + o(h).\]
        Multiplying by $u_x^T$ and using that $|u_x|=1$ one obtains
        \[ -  L_2 (1+|x|+|x+hu_x|)^{(q-1)}h\leq h\lambda_1(x)+o(h)\] which implies
        \[\lambda_1(x)\geq -  L_2 (1+|x|+|x+hu_x|)^{(q-1)} +\frac{o(h)}{h}. \]
Moreover, as $h\rightarrow 0$,
\[\lambda_1(x)\geq -  L_2 (1+2|x|)^{(q-1)}\]
which implies that for all eigenvalues of $H_{\bar{g}}$
\[\lambda(x) +L_2(1+2|x|)^{(q-1)}\geq 0\]
and thus the matrix $A(x)=H_{\bar{g}}(x)+ L_2(1+2|x|)^{(q-1)} I_d$ is semi-positive definite.
After some simple calculations, one deduces that \begin{equation}\label{eq-Hreg}
    H_{reg}(x)=\eta |x|^{2r}I_d +\eta 4r |x|^{2r-1}xx^T.
\end{equation}
where it is observed that the second term is semi-positive definite.
Let \begin{equation}\label{eq-R}
    R=\max\{(2^{3(q-1)+1}\frac{L_2}{\eta})^{\frac{1}{2r-q}},(2^{q}\frac{L_2}{\eta})^{\frac{1}{2r}}\}.
\end{equation}
\\For all $x$ such that $|x|>R$, one notes that
\[
\eta |x|^{2r}- L_2(1+2|x|)^{(q-1)} >0
\]
which yields that
\begin{equation*}
   \eta |x|^{2r}- L_2(1+2|x|)^{(q-1)}+L_2(1+2|R|)^{(q-1)} >0, \quad \forall \, x:\quad |x|>R.
\end{equation*}
On the other hand, if $|x|\leq R$
one obtains
\begin{equation*}
     \eta |x|^{2r}- L_2(1+2|x|)^{(q-1)}+L_2(1+2|R|)^{(q-1)}\geq 0.
\end{equation*}
Thus, one concludes that for all $x \in \mathbb{R}^d$, the matrix  \[B(x)=\eta |x|^{2r}I_d-L_2(1+2|x|)^{(q-1)}I_d+L_2(1+2|R|)^{(q-1)}I_d\] is positive definite.
As a result, the matrix $A(x)+B(x)+\eta 4r |x|^{2r-1}xx^T=H_{reg}+H_{\bar{g}}+L_2(1+2|R|)^{(q-1)}I_d$ is positive definite, which yields
\[
\langle \theta-\theta',h(\theta)-h(\theta')\rangle\geq -a |\theta-\theta'|^2,
\]
where $a=L_2(1+2|R|)^{q-1}$.
\end{proof}

\begin{proof}[\textbf{Proof of Proposition \ref{polip H}}]
Let the reguralisation part $\Theta(\theta):=\eta \theta |\theta|^{2r}$ for any $\theta \in \mathbb{R}^d$. By using the mean value theorem, one deduces
\[
|\Theta(\theta)-\Theta(\theta')|\leq ||H_{reg}(t\theta +(1-t)\theta') ||_2 |\theta-\theta'|, \quad \text{for some } t \in [0,1],
\]
where $||\cdot||_2$ denotes the spectral norm of a matrix. Due to \eqref{eq-Hreg}, one observes that
\[
||H_{reg}(x)||_2\leq \eta |x|^{2r} + \eta 4r |x|^{2r+1}\leq 4r \eta (1+2|x|^{2r+1})\leq 8r\eta (1+|x|)^{2r+1}.
\]
Thus,
\[|\Theta(\theta)-\Theta(\theta')| \leq 8r\eta(1+|t\theta+(1-t)\theta')|)^{2r+1}|\theta-\theta'|\leq 8r\eta (1+|\theta|+|\theta'|)^{2r+1}|\theta-\theta'|.\]
In view of Assumption \ref{A2}, the desired result follows.
\end{proof}

\subsection{Complementary details to Section \ref{Prelim}}

\begin{lemma} \label{moments}
Let Assumptions \ref{A2} and \ref{A3} hold. Then, for any $\lambda $ such that $0<{\lambda}\le\lambda_{max}$, one obtains for every $n \in \mathbb{N}$,
\begin{equation*}
   \E |\theta_{n+1}^\lambda|^2 \leq (1-\lambda  \kappa \eta)^n \E |\theta_0|^2+ \frac{(C_{M_0}+\frac{2d}{\beta}+\kappa\eta M_0^2)}{\kappa \eta} \quad \forall n \in \mathbb{N}
\end{equation*}
and, moreover,
\begin{equation*}
    \sup_n E|\theta_n^{\lambda}|^2< \E|\theta_0|^2 +  \frac{(C_{M_0}+\frac{2d}{\beta}+\kappa\eta M_0^2)}{\kappa \eta},
\end{equation*}
where $C_{M_0}$ is given in \eqref{eq-kappa} \eqref{eq-CM0} and $M_0$ in the proof.
\end{lemma}

\begin{proof}
\begin{equation}
    \begin{aligned} \label{eq-5}
            &2\lambda\E\left[ \langle \frac{\theta_n^{\lambda}}{|\theta_n^{\lambda}|^2},H_\lambda(\theta_n^{\lambda},X_{n+1})\rangle -\frac{\lambda}{2|\theta_n^{\lambda}|^2}|H_\lambda(\theta_n^{\lambda},X_{n+1})|^2|\theta_n^{\lambda}\right] \nonumber \\
            &\geq 2\lambda\E\left[ \langle \frac{\theta_n^{\lambda}}{|\theta_n^{\lambda}|^2},\frac{G(\theta_n^{\lambda},X_{n+1}) +\eta\theta_n^{\lambda}|\theta_n^{\lambda}|^{2r}}{1+\sqrt{\lambda} |\theta_n^{\lambda}|^{2r}}\rangle -\frac{\lambda}{2|\theta_n^{\lambda}|^2}|H_\lambda(\theta_n^{\lambda},X_{n+1})|^2|\theta_n^{\lambda}\right] \nonumber  \\
            &=2\lambda\frac{1}{|\theta_n^{\lambda}|^2(1+\sqrt{\lambda}|\theta_n^{\lambda}|^{2r})}\left( \langle \theta_n^{\lambda},\E G(\theta_n^{\lambda},X_0)\rangle+\eta|\theta_n^{\lambda}|^{2r+2}\right ) -2\lambda\frac{ 4\E[K^2(X_0)]}{|\theta_n^{\lambda}|^2}-2\lambda\eta^2 \nonumber  \\
            &\geq \lambda \left(\frac{ -2\E(K(X_0)(|\theta_n^\lambda|+|\theta_n^\lambda|^{q+1}) }{|\theta_n^{\lambda}|^2(1+\sqrt{\lambda}|\theta_n^{\lambda}|^{2r})} +\frac{\eta|\theta_n^{\lambda}|^{2r}}{2(1+\sqrt{\lambda}|\theta_n^{\lambda}|^{2r})}\right)\\&+\lambda\left(\frac{\eta|\theta_n^{\lambda}|^{2r}}{1+\sqrt{\lambda}|\theta_n^{\lambda}|^{2r}} -\frac{ 8\E[K^2(X_0)]}{|\theta_n^{\lambda}|^2}-2\eta^2\right)\\&+\lambda \frac{\eta|\theta_n^{\lambda}|^{2r}}{2(1+\sqrt{\lambda}|\theta_n^{\lambda}|^{2r})}\\&=J_1+J_2+\lambda \frac{\eta|\theta_n^{\lambda}|^{2r}}{2(1+\sqrt{\lambda}|\theta_n^{\lambda}|^{2r})}.
\end{aligned}
\end{equation}
Let $M_{0,1}>1$ such that  \[|\tn|>M_{0,1}\implies \eta |\tn|^{2r+2}\geq 4\E K(X_0)(|\tn|+|\tn|^{q+1}),\]
$M_{0,2}$ such that \[|\tn|>M_{0,2}\implies \eta \frac{|\tn|^{2r+2}}{2(1+|\tn|^{2r})}>8\E K^2(X_0)\]
and
noticing that
\[|\tn|>\eta \implies \frac{\eta |\tn|^{2r}}{2(1+\frac{1}{4\eta}|\tn|^{2r})}\geq 2\eta^2\]
then, using the fact that $\lambda \leq \min\{1,\frac{1}{16\eta^2}\}$
for $M_0=\max\{M_{0,1},M_{0,2},\eta\}$, one deduces that $J_1\geq 0$
and $J_2\geq 0$.\\
As a result, \eqref{eq-5} yields
that for $|\tn|>M_0$ there holds
\[\begin{aligned}2\lambda\E\left[ \langle \frac{\theta_n^{\lambda}}{|\theta_n^{\lambda}|^2},H_\lambda(\theta_n^{\lambda},X_{n+1})\rangle -\frac{\lambda}{2|\theta_n^{\lambda}|^2}|H_\lambda(\theta_n^{\lambda},X_{n+1})|^2|\theta_n^{\lambda}\right] &\geq \lambda \frac{\eta|\theta_n^{\lambda}|^{2r}}{2(1+\sqrt{\lambda}|\theta_n^{\lambda}|^{2r})}\\&\geq \lambda \eta \frac{|\tn|^{2r}}{2(1+\|tn|^{2r}}
\\&\geq  \lambda \eta \frac{M_0^{2r}}{2(1+M_0^{2r})}
\end{aligned}\]
where in the last step was derived from the monotonicity of the function $g(x)=x/2(1+x)$. Setting
\begin{equation}\label{eq-kappa}
    \kappa:=\frac{M_0^{2r}}{2(1+M_0^{2r})}
\end{equation} one concludes that
for $|\tn|\geq M_0$, there holds
\begin{equation}\label{eq-crucestim}
    \lambda \E \left[-2\langle \tn,H_\lambda(\tn,X_{n+1})\rangle+ |H_\lambda(\tn,X_{n+1})|^2 \big  |\tn\right]\leq -\lambda \kappa \eta |\tn|^2 .
\end{equation}
In addition, it is easy to see that for $|\tn|<M_0$
 \[\lambda \E \left[-2\langle \tn,H_\lambda(\tn,X_{n+1})\rangle+ |H_\lambda(\tn,X_{n+1})|^2\big | |\tn\right]\leq \lambda C_{M_0}\leq -\lambda \kappa \eta |\tn|^2 + \lambda (C_{M_0}+\kappa \eta M_0^2) \]
 where \begin{equation}\label{eq-CM0}
     C_{M_0}= \left(2 \E K(X_0) (M_0+M_0^{q+1}) +8\E K^2(X_0) +2\eta^2 M_0^2\right).
 \end{equation}\\
 As a result,
 \begin{equation}\label{eq-crucestim 2}
     \lambda \E \left[-2\langle \tn,H_\lambda(\tn,X_{n+1})\rangle+ |H_\lambda(\tn,X_{n+1})|^2|\right]\leq -\lambda \eta \kappa |\tn|^2 + \lambda C_{M_0} +\kappa \eta M_0^2 .
 \end{equation}
 We are now ready to derive an estimate for the second moments of our algorithm.\\
 Writing \[\begin{aligned}
         \E [|\theta^\lambda_{n+1}|^2\big |\tn]&=\E \left[ \left(\tn -\lambda H_\lambda(\tn,X_{n+1}) +\sqrt{\frac{2\lambda}{\beta}
}|\xi_{n+1}|\right)^2\big | \tn \right]\\&=|\tn|^2 \lambda \E \left[-2\langle \tn,H_\lambda(\tn,X_{n+1})\rangle+ |H_\lambda(\tn,X_{n+1})|^2|\right] + \frac{2\lambda}{\beta} \E |\xi_{n+1}|^2
\\&\leq |\tn|^2 -\lambda \eta \kappa |\tn|^2 + \lambda C_{M_0}+\lambda \kappa \eta M_0^2 +\frac{2\lambda}{\beta}d
\\&= (1-\lambda  \kappa \eta)|\tn|^2 + \lambda (C_{M_0}+\frac{2d}{\beta}+\kappa \eta M_0^2)
 \end{aligned}\]
 where the first step was derived by the independence of the normal random variable $\xi_{n+1}$ with respect to $\tn$ and the second as a result of \eqref{eq-crucestim}.
 Taking expectations one obtains
 \[\E |\theta_{n+1}^\lambda|^2 \leq (1-\lambda  \kappa \eta)\E |\tn|^2 + \lambda (C_{M_0}+\frac{2d}{\beta}) \quad \forall n\in \mathbb{N}. \]
 By iteration one concludes that
 \[\E |\theta_{n+1}^\lambda|^2 \leq (1-\lambda  \kappa \eta)^n \E |\theta_0|^2+ \frac{(C_{M_0}+\frac{2d}{\beta}+\kappa\eta M_0^2)}{\kappa \eta} \quad \forall n \in \mathbb{N}.\]
 Taking the supremum over $n$ completes the proof.
\end{proof}
\begin{proof}[\textbf{Proof of Lemma \ref{pmoments}}]
First one defines, for every $n \in \mathbb{N}$,
\begin{equation} \label{Delta}
\Delta_n: =\theta^{\lambda}_n-\lambda H_\lambda(\theta^{\lambda}_n,X_{n+1}).
\end{equation}
Then, one calculates that, for any integer $p> 1$ (since the case $p=1$ is covered by Lemma \ref{moments}),
\[
|\theta^\lambda_{n+1}|^{2p} = \left(|\Delta_n|^2+\frac{2\lambda}{\beta}|\xi_{n+1}|^2+2\langle \Delta_n,\sqrt{\frac{2\lambda}{\beta}}\xi_{n+1}\rangle\right)^p.
\]
Hence,
\begin{align} \label{eq-14}
   \hspace{-12pt} \mathbb{E}\left[|\theta^\lambda_{n+1}|^{2p}|\theta_n^{\lambda}\right]&=\E \left[\left(|\Delta_n|^2+\frac{2\lambda}{\beta}|\xi_{n+1}|^2+2\langle \Delta_n,\sqrt{\frac{2\lambda}{\beta}}\xi_{n+1}\rangle\right)^p|\theta_n^{\lambda} \right]
    \nonumber \\
    &=\sum_{k_1+k_2+k_3=p}\frac{p !}{k_{1} ! k_{2} ! k_{3} !} \E\left[|\Delta_n|^{2k_1}\left|\sqrt{\frac{2\lambda}{\beta}}\xi_{n+1}\right|^{2k_2}\left(2\langle \Delta_n,\sqrt{\frac{2\lambda}{\beta}}\xi_{n+1}\rangle\right)^{k_3}|\theta_n^{\lambda}\right]
    \nonumber \\
    &\leq \E[|\Delta_n|^{2p}|\theta_n^{\lambda}]+2p\E\left[|\Delta_n|^{2p-2}\langle \Delta_n,\sqrt{\frac{2\lambda}{\beta}}\xi_{n+1}\rangle|\theta_n^{\lambda}\right]
    \nonumber \\
    & \phantom{\leq \E[|\Delta_n|^{2p}|\theta_n^{\lambda}]} +\sum_{k=2}^{2p}\binom{2p}{k}\E\left[|\Delta_n|^{2p-k}\left|\sqrt{\frac{2\lambda}{\beta}} \xi_{n+1}\right|^k|\theta_n^{\lambda}\right]
    \nonumber \\
    & \leq \E[|\Delta_n|^{2p}|\theta_n^{\lambda}]+\E\bigg[\sum_{l=0}^{2(p-1)}\binom{2p}{l+2}\bigg(|\Delta_n|^{2(p-1)-l}  \left|\sqrt{\frac{2\lambda}{\beta}}\xi_{n+1}\right|^{(q-1)}\bigg)\\&\phantom{\E[|\Delta_n|^{2p}|\theta_n^{\lambda}]+\E\bigg[\sum_{l=0}^{2(p-1)}\binom{2p}{l+2}\bigg(|\Delta_n|^{2(p-1)-l}}\times \frac{2\lambda}{\beta}|\xi_{n+1}|^2     |\theta_n^{\lambda}\bigg]
    \nonumber \\
    &=\E[|\Delta_n|^{2p}|\theta_n^{\lambda}]+\E\bigg[\binom{2p}{2}\sum_{l=0}^{2(p-1)}\binom{2(p-1)}{l}\bigg(|\Delta_n|^{2(p-1)-l}  \bigg|\sqrt{\frac{2\lambda}{\beta}}\xi_{n+1}\bigg|^l\bigg)\\&\phantom{\E[|\Delta_n|^{2p}|\theta_n^{\lambda}]+\E\bigg[\binom{2p}{2}\sum_{l=0}^{2(p-1)}\binom{2(p-1)}{l}\bigg(|\Delta_n|^{2(p-1)-l} }\times\frac{2\lambda}{\beta}|\xi_{n+1}|^2     |\theta_n^{\lambda}\bigg]
    \nonumber \\
    &\leq \E[|\Delta_n|^{2p}|\theta_n^{\lambda}] + 2^{2p-3}p(2p-1)\E[|\Delta_n|^{2p-2}|\theta_n^{\lambda}] \frac{2\lambda}{\beta}d\\&\phantom{\E[|\Delta_n|^{2p}}+2^{2p-3}p(2p-1)\left(\frac{2\lambda}{\beta}\right)^p\E|\xi_{n+1}|^{2p}.
\end{align}
Let us also define, for every $n \in \mathbb{N}$,
\begin{equation} \label{r_n}
r_n :=-2\lambda \langle \theta_n^{\lambda},H_\lambda(\theta_n^{\lambda},X_{n+1})\rangle+\lambda^2 |H_\lambda(\theta_n^{\lambda},X_{n+1})|^2
\end{equation}
and observe that, due to \eqref{Delta},
 \[
 |\Delta_n|^2=|\theta_n^{\lambda}|^2+r_n.
 \]
Consequently,
\begin{align} \label{eq:Delta_n-intermediate}
    \E[|\Delta_n|^{2p}|\theta_n^{\lambda}]&=\sum_{k=0}^p\binom{p}{k}|\theta_n^{\lambda}|^{2(p-k)}\E\left[r_n^k|\theta_n^{\lambda}\right]
    \nonumber \\&=|\theta_n^{\lambda}|^{2p} +p |\theta_n^{\lambda}|^{2p-2}\E [r_n |\theta_n^{\lambda}]+ \sum_{k=2}^p\binom{p}{k}|\theta_n^{\lambda}|^{2(p-k)}\E\left[r_n^k|\theta_n^{\lambda}\right]
\end{align}
Let us also define the constant $M$ by the following expression
\begin{equation} \label{M}
\begin{aligned}
    M&:=\max\{M_0,1,\max_{2\leq k\leq p}{{\left(\frac{1}{\kappa}{\binom{p}{k}\binom{k}{\lceil\frac{k}{2}\rceil}}2^{4k}(1+ \E[K^{2k}(X_0))] \frac{4(p+1)}{\eta}\right)^{\frac{1}{k}}}},\\&\max_{2\leq k\leq p-1}{{\left(\frac{1}{\kappa}{\binom{p-1}{k}\binom{k}{\lceil\frac{k}{2}\rceil}}2^{4k}(1+ \E[K^{2k}(X_0))] \frac{4p}{\eta}\right)^{\frac{1}{k}}}}, \sqrt{2^{2p-3}(2p-1)p\frac{d}{\beta \eta}}\}.
    \end{aligned}
\end{equation}
When $|\theta_n^{\lambda}|>M$ and due to the fact that $\lambda \le 1$, see \eqref{lambda_max}, one obtains
\begin{align} \label{1}
    |\lambda\langle \theta^\lambda_n,H_\lambda(\theta^\lambda_n,X_{n+1})\rangle|&\leq \frac{\lambda K(X_{n+1})(1+|\theta^\lambda_n|^q)|\theta^\lambda_n|+\lambda\eta |\theta^\lambda_n|^{2r+2}}{1+\sqrt{\lambda}|\theta^\lambda_n|^{2r}} \nonumber\\
    &\leq \frac{\lambda K(X_{n+1})(|\theta^\lambda_n|+|\theta_n|^{q+1})}{1+\sqrt{\lambda}|\theta^\lambda_n|^{2r}}+\sqrt{\lambda}\eta |\theta^\lambda_n|^2 \nonumber\\
    &\leq \frac{\lambda K(X_{n+1})(2+2|\theta^\lambda_n|^{2r})}{1+\sqrt{\lambda}|\theta^\lambda_n|^{2r}}+\sqrt{\lambda}\eta |\theta^\lambda_n|^2 \nonumber\\
    &\leq \frac{2\sqrt{\lambda} K(X_{n+1})(\sqrt{\lambda}+\sqrt{\lambda}|\theta^\lambda_n|^{2r})}{1+\sqrt{\lambda}|\theta^\lambda_n|^{2r}}+\sqrt{\lambda}\eta |\theta^\lambda_n|^2 \nonumber \\
    &\leq 2\sqrt{\lambda}K(X_{n+1})+\sqrt{\lambda}\eta |\theta^\lambda_n|^2 \nonumber \\
    &\leq(\sqrt{a_n}+\sqrt {b_n})|\theta^\lambda_n|
\end{align}
where $a_n=2\sqrt{\lambda}K(X_{n+1})$ and $b_n=\sqrt{\lambda} \eta |\theta^\lambda_n|$ since $|\theta_n^{\lambda}| > M \ge 1$.
In addition,
\begin{align} \label{2}
    |\lambda^2H_\lambda^2(\theta^\lambda_n,X_{n+1})|&\leq \frac{2\lambda^2K^2(X_{n+1})(1+|\theta^\lambda_n|^q)^2+2\lambda^2\eta^2 |\theta^\lambda_n|^{4r+2}}{1+\lambda|\theta^\lambda_n|^{4r}} \nonumber \\
    &\leq
    \frac{4\lambda^2K^2(X_{n+1})(1+|\theta^\lambda_n|^{2q})}{1+\lambda |\theta^\lambda_n|^{4r}}
    +2\lambda \eta^2|\theta^\lambda_n|^2 \nonumber \\
    &\leq
    \frac{4\lambda^2K^2(X_{n+1})(2+|\theta^\lambda_n|^{4r})}{1+\lambda |\theta^\lambda_n|^{4r}}
    +2\lambda \eta^2|\theta^\lambda_n|^2 \nonumber \\
    &\leq
    \frac{4\lambda K^2(X_{n+1})(2\lambda+\lambda|\theta^\lambda_n|^{4r})}{1+\lambda |\theta^\lambda_n|^{4r}}
    +2\lambda \eta^2|\theta^\lambda_n|^2 \nonumber \\
    &\leq 8\lambda K^2(X_{n+1})  +2\lambda \eta^2|\theta^\lambda_n|^2 \nonumber \\
    &=2 a_n+2 b_n.
\end{align}

Observing that, due to \eqref{r_n}, \eqref{1} and \eqref{2},
\begin{align*}
    r_n^k&=\sum_{j=0}^k\binom{k}{j}2^{k-j}(\sqrt{a_n}+\sqrt{b_n})^{k-j}2^j(a_n+b_n)^j|\theta^\lambda_n|^{k-j}\\
    &=\sum_{j=0}^k\binom{k}{j}2^k((\sqrt{a_n}+\sqrt{b_n})^2)^{\frac{k-j}{2}}(a_n+b_n)^j|\theta^\lambda_n|^{k-j}\\&=\sum_{j=0}^k\binom{k}{j}2^k2^{\frac{k-j}{2}}(a_n+b_n)^{\frac{k-j}{2}}(a_n+b_n)^j|\theta^\lambda_n|^{k-j}\\&=\sum_{j=0}^k\binom{k}{j}2^{\frac{k+j}{2}}(a_n+b_n)^{\frac{k+j}{2}}|\theta^\lambda_n|^{k-j}\\&\leq \sum_{j=0}^k\binom{k}{j} 2^{k+j}(a_n^{\frac{k+j}{2}}+b_n^{\frac{k+j}{2}})|\theta^\lambda_n|^{k-j}
\end{align*}
yields that
    \[
    \hspace{-10pt}\E[r_n^k|\theta^\lambda_n]\leq \sum_{j=0}^k\binom{k}{j} 2^{2(k+j)}\lambda^{\frac{k+j}{2}}\E[ K^{k+j}(X_0)]|\theta^\lambda_n|^{k-j}+\sum_{j=0}^k\binom{k}{j}2^{k+j}\lambda^\frac{k+j}{2}\eta^{k+j}|\theta^\lambda_n|^{k+j}|\theta^\lambda_n|^{k-j}.
    \]
Consequently, and in view of \eqref{eq:Delta_n-intermediate},
\begin{align*}
     \hspace{-18pt}   \E[|\Delta_n|^{2p}|\theta^\lambda_n] \leq & |\theta_n^{\lambda}|^{2p} +p |\theta_n^{\lambda}|^{2p-2}\E [r_n |\theta_n^{\lambda}] + \sum_{k=2}^p\binom{p}{k}|\theta^\lambda_n|^{2p-2k} \\ & \times\left[\sum_{j=0}^k\binom{k}{j} 2^{2(k+j)}\lambda^{\frac{k+j}{2}}\E[ K^{k+j}(X_0)] |\theta^\lambda_n|^{k-j}+ |\theta^\lambda_n|^{2k} \sum_{j=0}^k\binom{k}{j}2^{k+j}\lambda^\frac{k+j}{2}\eta^{k+j}.\right]
\end{align*}
Moreover, due to \eqref{eq-crucestim},
 \[
 p |\theta_n^{\lambda}|^{2p-2}\E [r_n |\theta^\lambda_n]\leq -p\frac{1}{2}\lambda \kappa \eta |\theta_n^{\lambda}|^{2p},
 \]
 and thus one obtains
\begin{align*}
     \hspace{-15pt}    \E[|\Delta_n|^{2p}|\theta^\lambda_n] & \leq |\theta_n^{\lambda}|^{2p}-p\frac{1}{2}\lambda \kappa \eta |\theta_n^{\lambda}|^{2p}+ \sum_{k=2}^p\binom{p}{k}|\theta^\lambda_n|^{2p-2k}\bigg[\sum_{j=0}^k\binom{k}{j} 2^{2(k+j)}\lambda^{\frac{k+j}{2}}\E[K^{k+j}(X_0)]\\
           & \phantom{\leq |\theta_n^{\lambda}|^{2p}} \times\frac{|\theta^\lambda_n|^{2k}}{|\theta^\lambda_n|^{k+j}}\bigg]+\sum_{k=2}^p\binom{p}{k}|\theta^\lambda_n|^{2p}
           \left[\sum_{j=0}^k\binom{k}{j}2^{k+j}\lambda^\frac{k+j}{2}\eta^{k+j}\right]
          \\
          &\leq |\theta^\lambda_n|^{2p}-\frac{1}{2}p\lambda \kappa \eta|\theta^\lambda_n|^{2p} +\sum_{k=2}^p\binom{p}{k}|\theta^\lambda_n|^{2p}\bigg[\sum_{j=0}^k\binom{k}{j} 2^{2(k+j)}\lambda^{\frac{k+j}{2}}\\
          & \phantom{\leq |\theta_n^{\lambda}|^{2p}}\times \E[K^{k+j}(X_0)] (\frac{1}{M})^{k+j}\bigg]   +\sum_{k=2}^p\binom{p}{k}|\theta^\lambda_n|^{2p}\bigg[\sum_{j=0}^k\binom{k}{j}2^{k+j}\lambda^\frac{k+j}{2}\eta^{k+j}\bigg]
          \\&\leq |\theta^\lambda_n|^{2p}-\frac{1}{2}p\lambda \kappa \eta|\theta^\lambda_n|^{2p}\\
          & +|\theta^\lambda_n|^{2p}\sum_{k=2}^p\binom{p}{k}\bigg[\sum_{j=0}^k\binom{k}{j} 2^{2(k+j)}\lambda^{\frac{k+j}{2}}\E K^{k+j}(X_0)\left(\frac{1}{M}\right)^{k}\\&\phantom{+|\theta^\lambda_n|^{2p}\sum_{k=2}^p\binom{p}{k}\bigg[\sum_{j=0}^k\binom{k}{j} 2^{2(k+j)}\lambda^{\frac{k+j}{2}}}+\sum_{j=0}^k\binom{k}{j}2^{k+j}\lambda^\frac{k+j}{2}\eta^{k+j}\bigg]
\end{align*}
      Applying the previous relation for $p-1$ and bringing it all together using \eqref{eq-14},
\begin{align} \label{eq- chief}
      \hspace{-27pt}    \E[|\theta^\lambda_{n+1}|^{2p}|\theta^\lambda_n]&\leq |\theta^\lambda_n|^{2p}-\frac{1}{2}p\lambda \kappa \eta|\theta^\lambda_n|^{2p} \nonumber \\
          &+|\theta^\lambda_n|^{2p}\sum_{k=2}^p\binom{p}{k}\sum_{j=0}^k\binom{k}{j} 2^{2(k+j)}\lambda^{\frac{k+j}{2}}\E[K^{k+j}(X_0) \left(\frac{1}{M}\right)^{k}\nonumber \\
          &+|\theta^\lambda_n|^{2p}\sum_{k=2}^p\binom{p}{k}\sum_{j=0}^k\binom{k}{j}2^{k+j}\lambda^\frac{k+j}{2}\eta^{k+j} \\
          &+
          2^{2p-3}(2p-1)2\frac{\lambda}{\beta}d( |\theta^\lambda_n|^{2p-2}-\frac{1}{2}(p-2)\lambda \kappa \eta|\theta^\lambda_n|^{2p-2} \nonumber \\
          &+|\theta^\lambda_n|^{2p-2}\sum_{k=2}^{p-1}\binom{p-1}{k}\sum_{j=0}^k\binom{k}{j} 2^{2(k+j)}\lambda^{\frac{k+j}{2}}
          \E[K^{k+j}(X_0)] \left(\frac{1}{M}\right)^{k}
          \nonumber  \\
          &+|\theta^\lambda_n|^{2p-2}\sum_{k=2}^{p-1}\binom{p-1}{k}\sum_{j=0}^k\binom{k}{j}2^{k+j}\lambda^\frac{k+j}{2}\eta^{k+j}+2^{2p-3}(2p-1)\left(\frac{2\lambda}{\beta}\right)^p \E|\xi_{n+1}|^{2p}.
\end{align}
We now show that the restriction $\lambda \leq \min\{1,\frac{\kappa^2}{4\eta\left(8(p+1)\binom{p}{\lceil \frac{p}{2}\rceil}^2\right)^2 }\} $ yields the desired result. We start by showing that
\[
\frac{(p-2)}{4}\lambda \kappa \eta >\sum_{k=2}^p\binom{p}{k}[\sum_{j=0}^k\binom{k}{j} 2^{2(k+j)}\lambda^{\frac{k+j}{2}} \E[K^{k+j}(X_0)]\left(\frac{1}{M}\right)^{k}.
\]
Since for all $0 \leq j\leq k$, $2\leq k \leq p$
\begin{align*}
          \lambda^{\frac{k+j-2}{2}}&\leq 1\\&\leq \frac{\kappa \eta M^{k}}{4(p+1)\binom{p}{k}\binom{k}{\lceil\frac{k}{2}\rceil}2^{4k}(1+ \E[K^{2k}(X_0)])}\\&\leq
          \frac{\kappa \eta M^{k}}{4(p+1)\binom{p}{k}\binom{k}{j}2^{2(k+j)}\E[K^{k+j}(X_0)]},
\end{align*}
one deduces that, for $0 \leq j\leq k$, $2\leq k\leq p $
\[
\frac{\lambda \kappa \eta}{4(p+1)}\geq \binom{p}{k}\binom{k}{j} 2^{2(k+j)}\lambda^{\frac{k+j}{2}} \E[K^{k+j}(X_0)]\left(\frac{1}{M}\right)^{k}
\]
which yields that
\[
(p+1)\frac{\lambda \kappa \eta}{4(p+1)}\geq (k+1)\frac{\lambda \kappa \eta}{4(p+1)}\geq \binom{p}{k}\sum_{j=0}^k\binom{k}{j} 2^{2(k+j)}\lambda^{\frac{k+j}{2}} \E[K^{k+j}(X_0)]\left(\frac{1}{M}\right)^{k}
\]
Consequently,
\begin{equation}\label{eq-28}
 \hspace{-16pt}       \frac{1}{4}p\lambda \kappa \eta \geq \frac{1}{2} \lambda \kappa \eta + \frac{(p-2)}{4}\lambda \kappa \eta \geq \frac{1}{2} \lambda \kappa \eta+ \sum_{k=2}^p\binom{p}{k}[\sum_{j=0}^k\binom{k}{j} 2^{2(k+j)}\lambda^{\frac{k+j}{2}}\E K^{k+j}(X_0)\left(\frac{1}{M}\right)^{k}.
\end{equation}
Moreover, for $0\leq j \leq k$ and $2\leq k\leq p$,
\[
\lambda\leq \frac{\kappa}{4\eta\left(8(p+1)\binom{p}{\lceil \frac{p}{2}\rceil}^2\right)^2 }\leq \frac{\kappa}{4\frac{\kappa}{2}\eta\left(8(p+1)\binom{p}{\lceil \frac{p}{2}\rceil}^2\right)^{\frac{2}{k+j-1}}, }
\]
and thus,
    \[\begin{aligned}
    \lambda^{\frac{k+j-2}{2}}&\leq \frac{\kappa^{k+j-2} }{4^{\frac{k+j-1}{2}}8(p+1)\binom{p}{\lceil \frac{p}{2}\rceil}^2\eta^{k+j-1}}\\&\leq \frac{\kappa }{2^{k+j}4(p+1)\binom{p}{\lceil \frac{p}{2}\rceil}^2\eta^{k+j-1}}\\&\leq \frac{\kappa \eta }{4(p+1)\binom{k}{j}\binom{p}{k}2^{k+j} \eta^{k+j}}
    \end{aligned}
    \]
     which leads to
     \begin{equation}\label{eq-31}
          \frac{p-2}{4}\lambda \kappa \eta>\sum_{k=2}^{p}\binom{p}{k}\sum_{j=0}^k\binom{k}{j}2^{k+j}\lambda^\frac{k+j}{2}\eta^{k+j}.
     \end{equation}
    The combination of the inequalities \eqref{eq-28}, \eqref{eq-31} yields
 \begin{align} \label{eq-32}
      \E [|\Delta_n|^{2p}|\theta^\lambda_n]  \leq & |\theta^\lambda_n|^{2p} -\frac{1}{2}p\lambda \kappa \eta|\theta^\lambda_n|^{2p} \nonumber \\
      & + |\theta^\lambda_n|^{2p}\sum_{k=2}^p\binom{p}{k}\bigg[\sum_{j=0}^k\binom{k}{j} 2^{2(k+j)}\lambda^{\frac{k+j}{2}} \E[K^{k+j}(X_0)] \left(\frac{1}{M}\right)^{k}\\&\phantom{|\theta^\lambda_n|^{2p}\sum_{k=2}^p\binom{p}{k}\bigg[\sum_{j=0}^k\binom{k}{j} 2^{2(k+j)}\lambda^{\frac{k+j}{2}}} +\sum_{j=0}^k\binom{k}{j}2^{k+j}\lambda^\frac{k+j}{2}\eta^{k+j}\bigg] \nonumber \\
       \leq & (1-\lambda \kappa \eta)|\theta^\lambda_n|^{2p}.
 \end{align}
 Using similar arguments for $p-1$ leads to
    \begin{equation}\label{eq- p-1}
          \E [|\Delta_n|^{2p-2}|\theta^\lambda_n]\leq(1-\lambda \kappa \eta)|\theta^\lambda_n|^{2p-2}\leq \frac{1}{M^2}|\theta^\lambda_n|^{2p}
    \end{equation}
Thus, when $|\theta^\lambda_n|\geq M$, and in view of \eqref{eq-14}, \eqref{eq-32},\eqref{eq- p-1} and \eqref{M}, one obtains
\begin{align} \label{eq-34}
          \E[|\theta^\lambda_{n+1}|^{2p}\il|\theta^\lambda_n]\leq & (1-{\lambda}\kappa\eta)|\theta^\lambda_n|^{2p}+\frac{2^{2p-3}p(2p-1)\lambda d}{\beta M^2}|\theta^\lambda_n|^{2p}\il \nonumber \\
          & +2^{2p-3}(2p-1)p\left(\frac{2\lambda}{\beta}\right)^p \E|\xi_{n+1}|^{2p}\il
          \nonumber \\
          \leq  & (1-\frac{\kappa}{2}\lambda \eta)|\theta^\lambda_n|^{2p}\il+2^{2p-3}(2p-1)p\left(\frac{2\lambda}{\beta}\right)^p \E|\xi_{n+1}|^{2p}\il.
\end{align}
When $|\theta^\lambda_n|<M$, one observes that
\begin{align*}
        |\Delta_n|^{2p} &\leq |\theta^\lambda_n|^{2p}+\sum_{k=0}^{p-1} \binom{p}{k}|r_n|^{p-k}|\theta^\lambda_n|^{2k}\\
        &\leq
       (1-\lambda\frac{\kappa}{2}\eta)|\theta^\lambda_n|^{2p}+\lambda\frac{\kappa}{2}\eta M^{2p} +\sum_{k=0}^{p-1} \binom{p}{k}2^{p-k}M^{2k}
        (\lambda^{2(p-k)} |H_\lambda(\theta^\lambda_n,X_{n+1})|^{2(p-k)} \\
        & \phantom{\leq (1-\lambda\frac{\kappa}{2}\eta)|\theta^\lambda_n|^{2p}} +
        \lambda^{p-k} M ^{2(p-k)} |H_\lambda(\theta^\lambda_n,X_{n+1})|^{p-k}).
\end{align*}
Analysing the terms,
 \begin{align*}
 \hspace{-4pt}     \lambda ^{p-k} |H_\lambda(\theta^\lambda_n,X_{n+1})|^{p-k} &\leq \lambda^{p-k} (K(X_{n+1})(1+|\theta^\lambda_n|^q)+\eta |\theta^\lambda_n|^{2r+1})^{p-k}\\
      &\leq 2^{p-k}\lambda^{p-k}(\lambda^{p-k}( K(X_{n+1})^{p-k}(1+M^q)^{p-k} + \eta^{p-k} M^{(2r+1)(p-k)})
 \end{align*}
one obtains
\begin{align} \label{complement-2p}
    \E[|\Delta_n|^{2p}\ilp|\theta^\lambda_n] \leq &(1-\lambda\frac{\kappa}{2}\eta)|\theta^\lambda_n|^{2p}\ilp+(\lambda\frac{\kappa}{2}\eta M^{2p} \nonumber \\
    & +\lambda{\sum_{k=0}^{p-1} \binom{p}{k}2^{p-k}M^{2k}(R_{\lambda,M,p,\eta}^2+ M^{2(p-k)} R_{M,p,\eta})}\ilp \nonumber \\
    = & (1-\lambda\frac{\kappa}{2}\eta)|\theta^\lambda_n|^{2p}\ilp+\lambda\frac{\kappa}{2}\eta M^{2p}+ \lambda C(\eta,p,M)\ilp.
\end{align}
where
\[
R_{M,p,\eta}=2^{p-k}\left( \E[K(X_0)^{p-k}](1+M^q)^{p-k} + \eta^{p-k} M^{(2r+1)(p-k)}\right)
\]
and
\[
C(\eta,p,M)={\sum_{k=0}^{p-1} \binom{p}{k}2^{p-k}M^{2k}(R_{\lambda,M,p,\eta}^2+ M^{2(p-k)} R_{M,p,\eta})}.
\]
Moreover, in a similar way to \eqref{complement-2p}, one concludes that
\[
\E[|\Delta_n|^{2p-2}\ilp|\theta^\lambda_n]\leq M^{2p-2}\ilp + \lambda C(\eta,p,M)\ilp,
\]
and hence
\begin{align} \label{eq-30}
    \E\left[|\theta^\lambda_{n+1}|^{2p}\ilp|\theta^\lambda_n\right]&\leq(1-\lambda \frac{\kappa}{2}\eta)|\theta^\lambda_n|^{2p}\ilp \nonumber \\
    &+ \lambda\bigg( C(\eta,p,M)+ \frac{\kappa}{2}\eta M^{2p}+2^{2p-3}p(2p-1)( C(\eta,p-1,M) \\&\phantom{+ \lambda\bigg( C(\eta,p,M)+ \frac{\kappa}{2}\eta M^{2p}+2^{2p-3}p(2p-1)}+M^{2p-2})\frac{2}{\beta }d\bigg)\ilp \nonumber
    \\&+\lambda\left(2^{2p-3}p(2p-1)\left(\frac{2}{\beta}\right)^p\E|\xi_{n+1}|^{2p}\right)\ilp. \nonumber \\
    &\leq (1-\lambda \frac{\kappa}{2}\eta)|\theta^\lambda_n|^{2p}\ilp +\lambda A_p \ilp
\end{align}
where
\begin{equation}\label{eq-Ap}
\begin{aligned}
A_p&= C(\eta,p,M)+ \frac{\kappa}{2}\eta M^{2p}+2^{2p-3}p(2p-1)( C(\eta,p-1,M)+M^{2p-2})\frac{2}{\beta }d\\&+2^{2p-3}p(2p-1)\left(\frac{2}{\beta}\right)^p\E|\xi_{n+1}|^{2p}.
\end{aligned}
\end{equation}
Adding (\ref{eq-34}) and (\ref{eq-30}), one obtains
\begin{align*}
    \E|\theta^\lambda_{n+1}|^{2p}\leq & (1-\lambda \frac{\kappa}{2}\eta)\E|\theta^\lambda_n|^{2p} +\lambda A_{p}\leq  (1-\lambda \frac{\kappa}{2}\eta)^n \E |\theta_0|^{2p} +\frac{2}{\kappa \eta}A_{p} \nonumber \\
    \leq & (1-\lambda\frac{\kappa}{2}\eta)^n \E |\theta_0|^{2p}+C'_{p}
\end{align*}
where, in view of \eqref{eq-Ap},
\begin{equation}\label{eq-C'p}
    C'_{p}=  \frac{2}{\kappa \eta}A_p
\end{equation}
which yields the desired result.
\end{proof}

\begin{proof}[\textbf{Proof of Lemma \ref{V4_bound_continuous}}]
This is an immediate consequence of Remark \ref{same_moments} and the definition of the Lyapunov function as given in \eqref{def:Lyapunov} with $m=4$.\\
More specifically, \[\E (V_4(\bar{\theta}^\lambda_{nT})= \E(1+|\bar{\theta}^\lambda_{nT}|^2)^2\leq 2 +2 \E ||\bar{\theta}^\lambda_{nT}|^4\leq 2+ 2 \E |\theta_0|^{2p} +C'_p.\]
\end{proof}

\begin{proof}[\textbf{Proof of Lemma \ref{Lemma citation}}] See \cite[ Lemma 3.5]{nonconvex}.

\end{proof}

\begin{proof}[\textbf{Proof of Lemma \ref{aux_moments}}]
For $p\geq 1,$ application of Ito's lemma and taking expectation yields
\begin{equation*}
    \mathbb{E}\left[V_{p}\left(\bar{\zeta}_{t}^{\lambda, n}\right)\right]=\mathbb{E}\left[V_{p}\left(\bar{\theta}_{n T}^{\lambda}\right)\right]+\int_{n T}^{t} \mathbb{E}\left[\lambda \frac{\Delta V_{p}\left(\bar{\zeta}_{s}^{\lambda, n}\right)}{\beta}-\lambda\left\langle h\left(\bar{\zeta}_{s}^{\lambda, n}\right), \nabla V_{p}\left(\bar{\zeta}_{s}^{\lambda, n}\right)\right\rangle\right] \mathrm{d} s.
\end{equation*}
Differentiating both sides and using Lemma \ref{Lemma citation}, we obtain
\[\begin{aligned}
\frac{\mathrm{d}}{\mathrm{d} t} \mathbb{E}\left[V_{p}\left(\bar{\zeta}_{t}^{\lambda, n}\right)\right]&=\mathbb{E}\left[\lambda \frac{\Delta V_{p}\left(\bar{\zeta}_{t}^{\lambda, n}\right)}{\beta}-\lambda\left\langle h\left(\bar{\zeta}_{t}^{\lambda, n}\right), \nabla V_{p}\left(\bar{\zeta}_{t}^{\lambda, n}\right)\right\rangle\right] \\&\leq-\lambda \bar{c}(p) \mathbb{E}\left[V_{p}\left(\bar{\zeta}_{t}^{\lambda, n}\right)\right]+\lambda \tilde{c}(p)
\end{aligned}\]
which yields
\begin{align*}
\mathbb{E}\left[V_{p}\left(\bar{\zeta}_{t}^{\lambda, n}\right)\right] & \leq e^{-\lambda(t-n T) \bar{c}(p)} \mathbb{E}\left[V_{p}\left(\bar{\theta}_{n T}^{\lambda}\right)\right]+\frac{\tilde{c}(p)}{\bar{c}(p)}\left(1-e^{-\lambda \bar{c}(p)(t-n T)}\right) \\
& \leq e^{-\lambda(t-n T) \bar{c}(p)} \mathbb{E}\left[V_{p}\left(\bar{\theta}_{n T}^{\lambda}\right)\right]+\frac{\tilde{c}(p)}{\bar{c}(p)}.
\end{align*}
For p=2:
\begin{align*}
\mathbb{E}\left[V_{2}\left(\bar{\zeta}_{t}^{\lambda, n}\right)\right] \leq & e^{-\lambda(t-n T) \bar{c}(2)} \mathbb{E}\left[V_{2}\left(\bar{\theta}_{n T}^{\lambda}\right)\right]+\frac{\tilde{c}(2)}{\bar{c}(2)}
\\ \leq & (1-\sqrt{\lambda}\frac{\eta}{2} )^{n T} e^{-\lambda(t-n T) \bar{c}(2)} \mathbb{E}\left[V_{2}\left(\theta_{0}\right)\right]+\frac{\tilde{c}(2)}{\bar{c}(2)} \\
&+2\left(C_{X}\eta^{-1} + 2M_0^2(2 + \eta)+2d(\eta\beta)^{-1}\sqrt{\lambda_{max}}\right)+1
\\ \leq & \mathbb{E}\left[V_{2}\left(\theta_{0}\right)\right] +\frac{\tilde{c}(2)}{\bar{c}(2)}+2\left(C_{X}\eta^{-1} + 2M_0^2(2 + \eta)+2d(\eta\beta)^{-1}\sqrt{\lambda_{max}}\right)+1.
\end{align*}
For p=4:
\begin{align*}
\mathbb{E}\left[V_{4}\left(\bar{\zeta}_{t}^{\lambda, n}\right)\right] &\leq e^{-\lambda(t-n T) \bar{c}(4)} \mathbb{E}\left[V_{4}\left(\bar{\theta}_{n T}^{\lambda}\right)\right]+\frac{\tilde{c}(4)}{\bar{c}(4)}
\\&\leq 2 \mathbb{E}|\theta_0|^4+2+2C'_2 +\frac{\tilde{c}(4)}{\bar{c}(4)}.
\end{align*}
\end{proof}
\begin{proof}[\textbf{Proof of Lemma \ref{w12conn}.}]
Let $\mu,\nu \in \mathcal{P}_{{V_p}}.$ For any $\zeta \in C(\mu,\nu),$ one deduces
\[\begin{aligned}
W_1(\mu,\nu)&\leq\int_{\mathbb{R}^{d}} \int_{\mathbb{R}^{d}}\left|\theta-\theta^{\prime}\right| \zeta\left(\mathrm{d} \theta \mathrm{d} \theta^{\prime}\right)\\
&=\int_{\mathbb{R}^{d}} \int_{\mathbb{R}^{d}}\left|\theta-\theta^{\prime}\right| \mathbf{1}_{\left\{\left|\theta-\theta^{\prime}\right| \geq 1\right\}} \zeta\left(\mathrm{d} \theta \mathrm{d} \theta^{\prime}\right)+\int_{\mathbb{R}^{d}} \int_{\mathbb{R}^{d}}\left|\theta-\theta^{\prime}\right| \mathbf{1}_{\left\{\left|\theta-\theta^{\prime}\right|<1\right\}} \zeta\left(\mathrm{d} \theta \mathrm{d} \theta^{\prime}\right)\\
&\leq \int_{\mathbb{R}^{d}} \int_{\mathbb{R}^{d}}\left(|\theta|+\left|\theta^{\prime}\right|\right) \mathbf{1}_{\left\{\left|\theta-\theta^{\prime}\right| \geq 1\right\}} \zeta\left(\mathrm{d} \theta \mathrm{d} \theta^{\prime}\right)+\int_{\mathbb{R}^{d}} \int_{\mathbb{R}^{d}}\left|\theta-\theta^{\prime}\right|\left(1+V_{2}(\theta)+V_{2}\left(\theta^{\prime}\right)\right) \mathbf{1}_{\left\{\left|\theta-\theta^{\prime}\right|<1\right\}} \zeta\left(\mathrm{d} \theta \mathrm{d} \theta^{\prime}\right)\\
&\leq \int_{\mathbb{R}^{d}} \int_{\mathbb{R}^{d}}\left(1+V_{2}(\theta)+V_{2}\left(\theta^{\prime}\right)\right) \mathbf{1}_{\left\{\left|\theta-\theta^{\prime}\right| \geq 1\right\}} \zeta\left(\mathrm{d} \theta \mathrm{d} \theta^{\prime}\right)\\
&+\int_{\mathbb{R}^{d}} \int_{\mathbb{R}^{d}}\left|\theta-\theta^{\prime}\right|\left(1+V_{2}(\theta)+V_{2}\left(\theta^{\prime}\right)\right) \mathbf{1}_{\left\{\left|\theta-\theta^{\prime}\right|<1\right\}} \zeta\left(\mathrm{d} \theta \mathrm{d} \theta^{\prime}\right)\\
&=\int_{\mathbb{R}^{d}} \int_{\mathbb{R}^{d}}\left[1 \wedge\left|\theta-\theta^{\prime}\right|\right]\left(1+V_{2}(\theta)+V_{2}\left(\theta^{\prime}\right)\right) \mathbf{1}_{\left\{\left|\theta-\theta^{\prime}\right| \geq 1\right\}} \zeta\left(\mathrm{d} \theta \mathrm{d} \theta^{\prime}\right)\\
&+\int_{\mathbb{R}^{d}} \int_{\mathbb{R}^{d}}\left[1 \wedge\left|\theta-\theta^{\prime}\right|\right]\left(1+V_{2}(\theta)+V_{2}\left(\theta^{\prime}\right)\right) \mathbf{1}_{\left\{\left|\theta-\theta^{\prime}\right|<1\right\}} \zeta\left(\mathrm{d} \theta \mathrm{d} \theta^{\prime}\right)\\
&=\int_{\mathbb{R}^{d}} \int_{\mathbb{R}^{d}}\left[1 \wedge\left|\theta-\theta^{\prime}\right|\right]\left(1+V_{2}(\theta)+V_{2}\left(\theta^{\prime}\right)\right) \zeta\left(\mathrm{d} \theta \mathrm{d} \theta^{\prime}\right) \text {. }
\end{aligned}\]
Taking infimum over $\zeta$ completes the proof of the first inequality.
In order to prove the second inequality, one writes
\[\begin{aligned}
W_2^2(\mu,\nu)&\leq \int_{\mathbb{R}^{d}} \int_{\mathbb{R}^{d}}\left|\theta-\theta^{\prime}\right|^{2} \zeta\left(\mathrm{d} \theta \mathrm{d} \theta^{\prime}\right)\\
&=\int_{\mathbb{R}^{d}} \int_{\mathbb{R}^{d}}\left|\theta-\theta^{\prime}\right|^{2} \mathbf{1}_{\left\{\left|\theta-\theta^{\prime}\right| \geq 1\right\}} \zeta\left(\mathrm{d} \theta \mathrm{d} \theta^{\prime}\right)+\int_{\mathbb{R}^{d}} \int_{\mathbb{R}^{d}}\left|\theta-\theta^{\prime}\right|^{2} \mathbf{1}_{\left\{\left|\theta-\theta^{\prime}\right|<1\right\}} \zeta\left(\mathrm{d} \theta \mathrm{d} \theta^{\prime}\right)\\
&\leq \int_{\mathbb{R}^{d}} \int_{\mathbb{R}^{d}} 2\left(|\theta|^{2}+\left|\theta^{\prime}\right|^{2}\right) \mathbf{1}_{\left\{\left|\theta-\theta^{\prime}\right| \geq 1\right\}} \zeta\left(\mathrm{d} \theta \mathrm{d} \theta^{\prime}\right)+\int_{\mathbb{R}^{d}} \int_{\mathbb{R}^{d}}\left|\theta-\theta^{\prime}\right|\left(|\theta|+\left|\theta^{\prime}\right|\right) \mathbf{1}_{\left\{\left|\theta-\theta^{\prime}\right|<1\right\}} \zeta\left(\mathrm{d} \theta \mathrm{d} \theta^{\prime}\right)\\
&\leq \int_{\mathbb{R}^{d}} \int_{\mathbb{R}^{d}} 2\left(1+V_{2}(\theta)+V_{2}\left(\theta^{\prime}\right)\right) \mathbf{1}_{\left\{\left|\theta-\theta^{\prime}\right| \geq 1\right\}} \zeta\left(\mathrm{d} \theta \mathrm{d} \theta^{\prime}\right)\\
&+\int_{\mathbb{R}^{d}} \int_{\mathbb{R}^{d}} 2\left|\theta-\theta^{\prime}\right|\left(1+V_{2}(\theta)+V_{2}\left(\theta^{\prime}\right)\right) \mathbf{1}_{\left\{\left|\theta-\theta^{\prime}\right|<1\right\}} \zeta\left(\mathrm{d} \theta \mathrm{d} \theta^{\prime}\right)\\
&=2 \int_{\mathbb{R}^{d}} \int_{\mathbb{R}^{d}}\left[1 \wedge\left|\theta-\theta^{\prime}\right|\right]\left(1+V_{2}(\theta)+V_{2}\left(\theta^{\prime}\right)\right) \mathbf{1}_{\left\{\left|\theta-\theta^{\prime}\right| \geq 1\right\}} \zeta\left(\mathrm{d} \theta \mathrm{d} \theta^{\prime}\right)\\
&+2 \int_{\mathbb{R}^{d}} \int_{\mathbb{R}^{d}}\left[1 \wedge\left|\theta-\theta^{\prime}\right|\right]\left(1+V_{2}(\theta)+V_{2}\left(\theta^{\prime}\right)\right) \mathbf{1}_{\left\{\left|\theta-\theta^{\prime}\right|<1\right\}} \zeta\left(\mathrm{d} \theta \mathrm{d} \theta^{\prime}\right)\\
&=2 \int_{\mathbb{R}^{d}} \int_{\mathbb{R}^{d}}\left[1 \wedge\left|\theta-\theta^{\prime}\right|\right]\left(1+V_{2}(\theta)+V_{2}\left(\theta^{\prime}\right)\right) \zeta\left(\mathrm{d} \theta \mathrm{d} \theta^{\prime}\right) .
\end{aligned}\]
Taking infimum over $\zeta$ leads to
\[W_2(\mu,\nu)\leq 2 w_{1,2}\]
which completes the proof.
\end{proof}
\begin{proof}[\textbf{Proof of Proposition \ref{eberle}}] See \cite[ Proposition 3.14.]{nonconvex}.

\end{proof}

\begin{lemma}
The contraction constant in Proposition \ref{eberle} is given by
\[
\dot{c}=\min \{\bar{\phi}, \bar{c}(p), 4 \tilde{c}(p) \epsilon \bar{c}(p)\} / 2
\]
where the explicit expressions for $\bar{c}(p)$ and $\tilde{c}(p)$ can be found in Lemma \ref{Lemma citation} and $\bar{\phi}$ is given by
\[
\bar{\phi}=\left(\sqrt{4 \pi / K_{1} b} \exp \left((\bar{b} \sqrt{K_{1}} / 2+2 / \sqrt{K_{1}})^{2}\right)\right)^{-1}
\]
Furthermore, any $\epsilon$ can be chosen which satisfies the following inequality
\[
\epsilon \leq 1 \wedge\left(8 \tilde{c}(p) \sqrt{\pi / K_{1}} \int_{0}^{\tilde{b}} \exp \left((s \sqrt{K_{1}} / 2+2 / \sqrt{K_{1}})^{2}\right) \mathrm{d} s\right)^{-1}
\]
where $K_{1}=a$, $\tilde{b}=\sqrt{2 \tilde{c}(p) / \bar{c}(p)-1}$ and $\bar{b}=\sqrt{4 \tilde{c}(p)(1+\bar{c}(p)) / \bar{c}(p)-1} .$ The constant $\hat{c}$ is
given as the ratio $\mathrm{C}_{11} / \mathrm{C}_{10}$,  where $\mathrm{C}_{11}$, $\mathrm{C}_{10}$  are given explicitly in \cite[Lemma  3.24]{nonconvex}.
\end{lemma}

\begin{proof}[\textbf{Proof of Lemma \ref{Lemma 4.7}}] One initially observes that
\begin{align*}
    |\bar{\theta}^\lambda_t-\bar{\zeta}_{t}^{\lambda, n}|^2&=-2\lambda\int_{nT}^{t}\langle \bar{\zeta}_{s}^{\lambda, n}-\bar{\theta}^\lambda_s,h(\bar{\zeta}_{s}^{\lambda, n})-H_\lambda (\bar{\theta}^\lambda_{\lfloor s \rfloor},X_{\lceil s \rceil}) \rangle \\&=-2\lambda\int_{nT}^{t}\langle \bar{\zeta}_{s}^{\lambda, n}-\bar{\theta}^\lambda_s,h(\bar{\zeta}_{s}^{\lambda, n})-H(\bar{\theta}^\lambda_{\lfloor s \rfloor},X_{\left \lceil s  \right \rceil }) \rangle ds
    \\& -2\lambda \int_{nT}^{t}\langle \bar{\zeta}_{s}^{\lambda, n}-\bar{\theta}^\lambda_s,  H(\bar{\theta}^\lambda_{\left \lfloor s \right \rfloor},X_{\left \lceil s  \right \rceil }) -H_\lambda (\bar{\theta}^\lambda_{\lfloor s \rceil}, X_{\lceil s \rceil})\rangle  ds
    \\& = -2 \lambda \int_{nT}^t\langle \zs -\ths ,h(\zs)-h(\ths)\rangle ds
    \\& -2\lambda \int_{nT}^t \langle \zs -\ths ,h(\ths)-h(\fs)\rangle dsth
    \\& -2\lambda  \int_{nT}^t \langle \zs -\ths,h(\fs)-H(\fs,\Xs)\rangle ds
     \\& -2\lambda \int_{nT}^{t}\langle \bar{\zeta}_{s}^{\lambda, n}-\bar{\theta}^\lambda_s,  H(\bar{\theta}^\lambda_{\left \lfloor s \right \rfloor},X_{\left \lceil s  \right \rceil }) -H_\lambda (\bar{\theta}^\lambda_{\lfloor s \rfloor}, X_{\lceil s \rceil})\rangle  ds.
\end{align*}
Taking expectations on both sides yields that
\[
\begin{aligned}\label{L2_estimate}
    \E|\bar{\theta}^\lambda_t-\bar{\zeta}_{t}^{\lambda, n}|^2 &= -2 \lambda \int_{nT}^t\E\langle \zs -\ths ,h(\zs)-h(\ths)\rangle ds
    \\& -2\lambda \int_{nT}^t \E\langle \zs -\ths ,h(\ths)-h(\fs)\rangle ds
    \\& -2\lambda  \int_{nT}^t \E\langle \zs -\ths,h(\fs)-H(\fs,\Xs)\rangle ds
     \\& -2\lambda \int_{nT}^{t}\E\langle \bar{\zeta}_{s}^{\lambda, n}-\bar{\theta}^\lambda_s,  H(\bar{\theta}^\lambda_{\left \lfloor s \right \rfloor},X_{\left \lceil s  \right \rceil }) -H_\lambda (\bar{\theta}^\lambda_{\lfloor s \rceil}, X_{\lceil s \rceil})\rangle  ds
    \\&\leq-2 \lambda \int_{nT}^t\E\langle \zs -\ths ,h(\zs)-h(\ths)\rangle ds\\&
    +
     \int_{nT}^{T}\frac{\lambda a}{2}\E|\bar{\zeta}_{s}^{\lambda, n}-\bar{\theta}^\lambda_s|^2 +\E\frac{2\lambda}{a}|h(\ths)-h(\fs)|^2 ds \\&
     -2\lambda  \int_{nT}^t \E\langle \zs -\ths,h(\fs)-H(\fs,\Xs)\rangle ds
     \\& +\int_{nT}^{t}\frac{\lambda a}{2}\E|\bar{\theta}^\lambda_s-\bar{\zeta}_{s}^{\lambda, n}|^2+\frac{ 2\lambda }{a}\E|H(\bar{\theta}^\lambda_{\lfloor s \rfloor },X_{\lceil s \rceil})-H_\lambda (\bar{\theta}^\lambda_{\lfloor s \rfloor},X_{\lceil s \rceil} )|^2 ds
     \\&\leq  \int_{nT}^{t} A_s^{\lambda,n}+B_s^{\lambda,n}+E_s^{\lambda,n} +D_s^{\lambda,n} ds,
\end{aligned}\]
where \[A_s^{\lambda,n}=-2 \lambda \E\langle \zs -\ths ,h(\zs)-h(\ths)\rangle+\lambda a  \E|\bar{\zeta}_{s}^{\lambda, n}-\bar{\theta}^\lambda_s|^2,\]
\[B_s^{\lambda,n}=\frac{2\lambda}{a}\E|h(\ths)-h(\fs)|^2,\]
\[E_s^{\lambda,n}= -2\lambda  \E\langle \zs -\ths,h(\fs)-H(\fs,\Xs)\rangle\] and
\[D_s^{\lambda,n}=\frac{ 2\lambda }{a}\E|H(\bar{\theta}^\lambda_{\lfloor s \rfloor },X_{\lceil s \rceil})-H_\lambda (\bar{\theta}^\lambda_{\lfloor s \rfloor},X_{\lceil s \rceil} )|^2. \]\\
Using the property in Proposition \ref{Proposition 1},
one obtains
\begin{equation} \label{eq:A}
  \begin{aligned}
  A_t^{\lambda,n}&=-2 \lambda  \E \langle \bar{\zeta}_{t}^{\lambda, n}-\bar{\theta}^\lambda_t,h(\bar{\zeta}_{t}^{\lambda, n})-h(\bar{\theta}^\lambda_t) \rangle +\lambda a  \E |\bar{\zeta}_{t}^{\lambda, n}-\bar{\theta}^\lambda_t|^2\\&
   \leq2\lambda a  \E |\bar{\zeta}_{t}^{\lambda, n}-\bar{\theta}^\lambda_t|^2 +\lambda a  \E |\bar{\zeta}_{t}^{\lambda, n}-\bar{\theta}^\lambda_t|^2\\&\leq  3\lambda a  \E |\bar{\zeta}_{t}^{\lambda, n}-\bar{\theta}^\lambda_t|^2.
  \end{aligned}
    \end{equation}
In addition, taking advantage of the polynomial Lipschitzness of $H$( and consequently for $h$), one observes that
\begin{align*}
   B_t^{\lambda,n}&=\frac{2\lambda}{a}\E|h(\tht)-h(\ft)|^2
    \\&\leq \frac{2L\lambda}{a}\E\left[ (1 +|X_{0}|)^{2\rho}(1+|\bar{\theta}^\lambda_t|+|\bar{\theta}^\lambda_{\lfloor t \rfloor}|)^{2l}|\bar{\theta}^\lambda_t-\bar{\theta}^\lambda_{\lfloor t \rfloor}|^2\right ]
    \\&\leq \frac{2L\lambda}{a}\sqrt{\E\left[ (1 +|X_{0}|)^{4\rho}(1+|\bar{\theta}^\lambda_t|+|\bar{\theta}^\lambda_{\lfloor t \rfloor}|)^{4l}\right]}\sqrt{\E\left[|\bar{\theta}^\lambda_t-\bar{\theta}^\lambda_{\lfloor t \rfloor}|^4\right ]}
    \\&\leq \frac{2L\lambda}{a}\sqrt{\E\left[ (1+|X_{0}|)^{4\rho}(1+2|\bar{\theta}^\lambda_{\lfloor t \rfloor}| +|\bar{\theta}^\lambda_t-\bar{\theta}^\lambda_{\lfloor t \rfloor}|)^{4l}\right]}\sqrt{\E\left[|\bar{\theta}^\lambda_t-\bar{\theta}^\lambda_{\lfloor t \rfloor}|^4\right ]}.
\end{align*}
Furthermore, one applies again the Cauchy-Schwarz inequality to obtain
\begin{equation}\label{eq:B_intermediate}
\begin{aligned}
B_t^{\lambda,n}
&\leq \frac{2L}{a}(\E(1+|X_0|)^{8\rho})^{\frac{1}{4}}9^l \left(1+2^{8l}\E|\bar{\theta}^\lambda_{\lfloor t \rfloor}| ^{8l} +\E|\bar{\theta}^\lambda_t-\bar{\theta}^\lambda_{\lfloor t \rfloor}| ^{8l}\right)^{\frac{1}{4}}\\&\phantom{\frac{2L}{a}(\E(1+|X_0|)^{8\rho})^{\frac{1}{4}}}\times\lambda \sqrt{\E|\bar{\theta}^\lambda_t-\bar{\theta}^\lambda_{\lfloor t \rfloor}|^4}.
\end{aligned}
\end{equation}
By taking into consideration that
\begin{align*}
|\bar{\theta}^\lambda_t -\bar{\theta}^\lambda_{\lfloor t \rfloor}| \leq & \lambda|\int_{\lfloor t \rfloor}^t H_\lambda(\bar{\theta}^\lambda_{\lfloor u\rfloor},X_{\lceil u \rceil}) du | + \sqrt{\frac{2\lambda}{\beta}}|\tilde{B}^{\lambda}_{t} -\tilde{B}^{\lambda}_{\left \lfloor t \right\rfloor}| \\ \leq & \sqrt{\lambda}\left(\int_{\lfloor t \rfloor}^t K(X_{\lceil u \rceil})+\eta |\bar{\theta}^\lambda_{\lfloor u \rfloor}| du+\sqrt{\frac{2}{\beta}}|\tilde{B}^{\lambda}_{t} - \tilde{B}^{\lambda}_{\left \lfloor t \right\rfloor}|\right),
\end{align*}
and that both the requires moments of $X_{\lceil t \rceil}$ and of $\bar{\theta}^\lambda_{\lfloor t \rfloor}$ are finite due to Assumption \ref{A3} and \eqref{law-connection} respectively, one deduces that $\sqrt{\E|\bar{\theta}^\lambda_t-\bar{\theta}^\lambda_{\lfloor t \rfloor}|^4}\leq \tilde{C}_1\lambda$,
where $$\tilde{C}_1 = 9\sqrt{\E[K^{4}(X_0)]+\eta^4 (\E|\theta_0|^4+C'_2) + \frac{48}{\beta^4}d^2}.$$ Similarly, $$\E|\bar{\theta}^\lambda_t-\bar{\theta}^\lambda_{\lfloor t \rfloor}| ^{8l} \le  \tilde{C}_2 \lambda^{4l},$$ where
$$\tilde{C}_2 = 3^{4l}\sqrt{\E K^{8l}(X_0)+\eta^{8l}(\E|\theta_0|^{8l}+C'_{4l}+(\frac{2}{\beta})^{8l}d^{4l}(8l-1)!!}.$$ Here, the fact was used that the increment of a $d$-dimensional Brownian motion has a $d$-dimensional Gaussian distribution with mean 0 and covariance matrix $(t-\lfloor t \rfloor) \mathbb{I}_d$. Its $2m$-th moment is given by
\[
\E \left[|\tilde{B}^{\lambda}_{t} - \tilde{B}^{\lambda}_{\left \lfloor t \right\rfloor}|^{2m}\right]=\E\left[\left(\sum_{i=1}^d Y_i^2\right)^m\right] \leq d^m \E Z^{2m},
\]
where $Y_i$, $i \in\{1,\ldots,d\}$, are the increments of the one dimensional Brownian motions which follow the same distribution as $Z \sim\mathcal{N}(0,t-\lfloor t \rfloor)$. Hence,
\[
\E \left[|\tilde{B}^{\lambda}_{t} - \tilde{B}^{\lambda}_{\left \lfloor t \right\rfloor}|^{2m}\right] \leq d^m (2m-1)!!(t-\lfloor t \rfloor)^m\leq d^m (2m-1)!!.
\]
Thus, \eqref{eq:B_intermediate} implies that
\begin{equation} \label{eq:B}
B_t^{\lambda,n}\leq  C_1\lambda^2
\end{equation}
where
\begin{equation} \label{C_1}
    C_1=2\frac{L}{a}2^{4l}\tilde{C}_1(\E(1+X_0)^{8\rho})^{\frac{1}{4}}\left(1+2^{8l}(\E|\theta_0|^{8l}+C'_{4l})+\tilde{C}_2\right)^\frac{1}{4}.
\end{equation}
Furthermore, the term $E_t^{\lambda,n}$ can be analysed as follows
\[\begin{aligned}
E_t^{\lambda,n}&= -2\lambda\E \langle \zt -\ft ,h(\ft)-H(\ft,\Xt)\rangle-2\lambda  \E \langle \ft -\tht ,h(\ft)-H(\ft,\Xt)\rangle
\\&= \E \left[\E \langle \zt -\ft ,h(\ft)-H(\ft,\Xt)\rangle\big| \zt,\ft \right]\\&+ \E \langle \ft -\tht ,h(\ft)-H(\ft,\Xt)\rangle.
\end{aligned}
\]
Using the unbiased estimator property, the first term is zero so \begin{equation}\label{eq Et}
\begin{aligned}
E_t^{\lambda,n}&=-2\lambda \E \langle \ft -\tht ,h(\ft)-H(\ft,\Xt)\rangle\\&= -2\lambda \E \left\langle \int_{\lfloor t \rfloor}^{t} \lambda H_\lambda(\fs, X_{\lceil s \rceil})ds+\frac{\sqrt{2\lambda}}{\beta}  (\tilde{B}^\lambda_t-\tilde{B}^\lambda_{\lfloor t \rfloor}),h(\fs)-H(\fs,\Xs)\right\rangle
\\&=-2\lambda \E \left\langle  \lambda H_\lambda(\ft, X_{\lceil t \rceil})({t}-{\lfloor t \rfloor}),h(\ft)-H(\ft,\Xt)\right\rangle
\\&\leq 2\lambda^2\sqrt{\E |H_\lambda(\ft,\Xt)|^2}\sqrt{\E |h(\ft)-H(\ft,\Xt)|^2}.
\end{aligned}
\end{equation}
In addition, one observes that the first term of the above product yields that
\[2\lambda^2\sqrt{\E |H_\lambda(\ft,\Xt)|^2}\leq 2\lambda^2 \sqrt{\E 2K^2(X_0)\E (1+|\ft|)^{2q} +2\eta^2 \E |\ft |^{4r+2}}\]
which implies that
\[2\lambda^2\sqrt{\E |H_\lambda(\ft,\Xt)|^2}\leq2\lambda^2C_{E1}.
\]
where $C_{E1}=\sqrt{\E 2^{2q+3} K^2(X_0)(1+\E |\theta_0|^{2q}+C'_q)+2\eta^2(\E |\theta_0|^{4r+2}+C'_{2r+1})}$.
\\
\\
For the second term, using the unbiased estimator property of $h$, an application of Jensen's inequality leads to
\[\begin{aligned}
\sqrt{\E |h(\ft)-H(\ft,\Xt)|^2}&\leq \sqrt{2\E |H(\ft,\Xt)|^2}\\&\leq  2L\sqrt{\E |H(0,X_0)|^2 + \E (1+|X_0|)^{2\rho}\E (1+|\ft|)^{2l+2} }
\\&\leq C_{E2}
\end{aligned}\]
where $C_{E2}=2L \sqrt{ \E |H(0,X_0)|^2+2^{2l+2} \E (1+|X_0|)^{2\rho}(1+C'_{l+1}+\E|\theta_0|^{2l+2})}$.\\\\
Combining the above estimates and inserting in \eqref{eq Et}, one deduces that
\begin{equation}\label{estimate Es}
\begin{aligned}
    E_t^{\lambda,n}\leq \lambda^2 C_3
    \end{aligned}
\end{equation}
where \begin{equation}
    \begin{aligned}C_3&=2\sqrt{2^{2q+3} \E K^2(X_0)(1+\E |\theta_0|^{2q}+C'_q)+2\eta^2(\E |\theta_0|^{4r+2}+C'_{2r+1})}\\&\times 2L \sqrt{ \E |H(0,X_0)|^2+2^{2l+2} \E (1+|X_0|)^{2\rho}(1+C'_{l+1}+\E|\theta_0|^{2l+2})}.
    \end{aligned}
\end{equation}
Moreover,
\begin{align} \label{eq:C}
    \hspace{-43pt}D_t^{\lambda,n} = &\frac{ 2\lambda }{a}\E|H(\bar{\theta}^\lambda_{\lfloor t \rfloor },X_{\lceil t \rceil})-H_\lambda (\bar{\theta}^\lambda_{\lfloor t \rfloor},X_{\lceil t \rceil} )|^2 \nonumber
    \\ \leq & \frac{ 2\lambda^2 }{a}\E\left[|H(\bar{\theta}^\lambda_{\lfloor t \rfloor },X_{\lceil t \rceil})||\bar{\theta}^\lambda_{\lfloor t \rfloor }||^{2r}\right]^2 \nonumber
    \\ \leq & \frac{4 \lambda^2 }{a}\left(\E|H(\bar{\theta}^\lambda_{\lfloor t \rfloor },X_{\lceil t \rceil})-H(\bar{\theta}^\lambda_{ 0  },X_{\lceil t \rceil})|^2|\bar{\theta}^\lambda_{\lfloor t \rfloor }|^{4r} +\E|H(\bar{\theta}^\lambda_{ 0  },X_{\lceil t \rceil})|^2|\bar{\theta}^\lambda_{\lfloor t \rfloor }|^{4r}\right) \nonumber
    \\ \leq & \frac{\lambda a }{2}\E|\bar{\theta}^\lambda_t-\bar{\zeta}_{t}^{\lambda, n}|^2+ \frac{L^2\lambda^2}{a}\E\left[ (1 +|X_{\left \lceil t \right \rceil}|)^{2\rho}(1+|\theta_0|+|\bar{\theta}^\lambda_{\lfloor t \rfloor}|)^{2l}|\theta_0-\bar{\theta}^\lambda_{\lfloor t \rfloor}|^2|\bar{\theta}^\lambda_{\lfloor t \rfloor }|^{4r}\right ] \nonumber \\ & +\frac{4\lambda^2}{a}\E|H(\theta_0,X_0)|^2|\bar{\theta}^\lambda_{\lfloor t \rfloor }|^{4r} \nonumber
    \\ \leq &\frac{L^2\lambda^2}{a}\sqrt{\E\left[ (1 +|X_{\left \lceil t \right \rceil}|)^{4\rho}(1+|\theta_0|+|\bar{\theta}^\lambda_{\lfloor t \rfloor}|)^{4l}|\theta_0-\bar{\theta}^\lambda_{\lfloor t \rfloor}|^4\right]}\sqrt{\E|\bar{\theta}^\lambda_{\lfloor t \rfloor }|^{8r}} \nonumber \\ & + \frac{4\lambda^2}{a}\sqrt{\E|H(\theta_0,X_0)|^4}\sqrt{\E|\bar{\theta}^\lambda_{\lfloor t \rfloor }|^{8r}}
   \nonumber \\ \leq & \sqrt{C'_{4r}+\E|\theta_0|^{8r}} \nonumber \\ & \times\left(\frac{L^2\lambda^2}{a}(\E(1+|X_0|)^{8\rho})^{\frac{1}{4}}\left({\E(1+|\theta_0| +\bar{\theta}^\lambda_{\lfloor t \rfloor })^{8l}|\theta_0-\bar{\theta}^\lambda_{\lfloor t \rfloor }|^{8}}\right)^\frac{1}{4}\right) \nonumber \\ &+\sqrt{C'_{4r}+\E|\theta_0|^{8r}}\frac{4\lambda^2}{a}\sqrt{\E|H(\theta_0,X_0)|^4}
   \nonumber \\ \leq &
   \sqrt{C'_{4r}+\E|\theta_0|^{8r}}  \nonumber \\ & \times\frac{L^2\lambda^2}{a}(\E(1+|X_0|)^{8\rho})^{\frac{1}{4}}2^{2l+2}\left(\E(1 +|\theta_0|)^{16l}+\E|\bar{\theta}^\lambda_{\lfloor t \rfloor }|^{16l}\right)^\frac{1}{8}\left(\E|\theta_0|^{16}+\E |\bar{\theta}^\lambda_{\lfloor t \rfloor }|^{16}\right)^\frac{1}{8} \nonumber \\ & +\sqrt{C'_{4r}+\E|\theta_0|^{8r}}\frac{4\lambda^2}{a}\sqrt{\E|H(\theta_0,X_0)|^4}
   \nonumber \\ \leq &C_2\lambda ^2
\end{align}
where
\begin{align} \label{C_2}
  C_2 = & \sqrt{C'_{4r}+\E|\theta_0|^{8r}}\bigg(\frac{L^2}{a}(\E(1+|X_0|)^{8\rho})^{\frac{1}{4}}2^{2l+2}\left(\E(1+|\theta_0|)^{16l} + \E|\theta_0|^{16l}+C'_{8l}\right)^\frac{1}{8}    \nonumber \\& +\left(\E|\theta_0|^{16} +\E |\theta_0|^{16}+C'_8\right)^\frac{1}{8}\bigg)
  +\frac{4}{a}\sqrt{\E|H(\theta_0,X_0)|^4}.
\end{align}
In view of the estimates \eqref{eq:A}, \eqref{eq:B} and \eqref{eq:C}, \eqref{estimate Es} one concludes that equation \eqref{L2_estimate} can be rewritten as
\begin{equation*}
    \E|\bar{\theta}^\lambda_t-\bar{\zeta}_{t}^{\lambda, n}|^2\leq 3\lambda a\int_{nT}^t\E|\bar{\theta}^\lambda_s-\bar{\zeta}_{s}^{\lambda, n}|^2ds+(C_1+C_2+C_3)\lambda <\infty.
\end{equation*}
The application of Gronwall's Lemma implies that
\begin{equation*}
    \E|\bar{\theta}^\lambda_t-\bar{\zeta}_{t}^{\lambda, n}|^2\leq c \lambda, \qquad \mbox{where } c=e^{3a}(C_1+C_2+C_3)
\end{equation*}
which yields the desired rate while the constant $c$ is independent of $t$ and $\lambda$.
\end{proof}

\begin{proof}[\textbf{Proof of Lemma \ref{Lemma 4.8}}] In view of the result in Lemma \ref{Lemma 4.7},
\[
    \begin{aligned}
&W_{1}\left(\mathcal{L}\left(\bar{\zeta}_{t}^{\lambda, n}\right), \mathcal{L}\left(Z_{t}^{\lambda}\right)\right)\\
&\leq \sum_{k=1}^{n} W_{1}\left(\mathcal{L}\left(\bar{\zeta}_{t}^{\lambda, k}\right), \mathcal{L}\left(\bar{\zeta}_{t}^{\lambda, k-1}\right)\right)\\
&\leq \sum_{k=1}^{n} w_{1,2}\left(\mathcal{L}\left(\zeta_{t}^{k T, \bar{\theta}_{k T}^{\lambda}, \lambda}\right), \mathcal{L}\left(\zeta_{t}^{k T, \bar{z}_{k T}^{\lambda, k-1}, \lambda}\right)\right)\\
&\leq \hat{c} \sum_{k=1}^{n} \exp (-\dot{c}(n-k)) w_{1,2}\left(\mathcal{L}\left(\bar{\theta}_{k T}^{\lambda}\right), \mathcal{L}\left(\bar{\zeta}_{k T}^{\lambda, k-1}\right)\right)\\&
\leq \hat{c} \sum_{k=1}^{n} \exp (-\dot{c}(n-k)) W_{2}\left(\mathcal{L}\left(\bar{\theta}_{k T}^{\lambda}\right), \mathcal{L}\left(\bar{\zeta}_{k T}^{\lambda, k-1}\right)\right)
\\&\phantom{\hat{c} \sum_{k=1}^{n} \exp (-\dot{c}(n-k))}\times \left[1+\left\{\mathbb{E}\left[V_{4}\left(\bar{\theta}_{k T}^{\lambda}\right)\right]\right\}^{1 / 2}+\left\{\mathbb{E}\left[V_{4}\left(\bar{\zeta}_{k T}^{\lambda, k-1}\right)\right]\right\}^{1 / 2}\right]
\\&\leq (\sqrt{\lambda})^{-1} \hat{c} \sum_{k=1}^{n} \exp (-\dot{c}(n-k)) W_{2}^{2}\left(\mathcal{L}\left(\bar{\theta}_{k T}^{\lambda}\right), \mathcal{L}\left(\bar{\zeta}_{k T}^{\lambda, k-1}\right)\right) \\
&+3 \sqrt{\lambda} \hat{c} \sum^{n} \exp (-\dot{c}(n-k))\left[1+\mathbb{E}\left[V_{4}\left(\bar{\theta}_{k T}^{\lambda}\right)\right]+\mathbb{E}\left[V_{4}\left(\bar{\zeta}_{k T}^{\lambda, k-1}\right)\right]\right]
\\& \leq \sqrt{e^{3a}(C_1+C_2+C_3)}\sqrt{\lambda} \frac{\hat{c}}{1-\exp (-\dot{c})}
\\+&3\sqrt \lambda \frac{\hat{c}}{1-\exp (-\dot{c})}\left(1+2\mathbb{E}|\theta_0|^4+2+2C'_2 +\frac{\tilde{c}(4)}{\bar{c}(4)}+2\mathbb{E}|\theta_0|^4+2+2C'_2\right)
\\&=\sqrt{\lambda} z_1
\end{aligned}\]
where \begin{equation}\label{eq-z_1}
        z_1 =\frac{\hat{c}}{1-exp(-\dot{c})}\left[\sqrt{e^{3a}(C_1+C_2+C_3)} + 3\left(5+4C'_2 \frac{\tilde{c}(4)}{\bar{c}(4)}+4\mathbb{E}|\theta_0|^4 \right)\right]
\end{equation}
and $C_1$, $C_2$ are given by (\ref{C_1}) , (\ref{C_2}) respectively.
\end{proof}

\subsection{Proof of main results} \label{main_proofs}

\begin{lemma} \label{W1 convergence}
Let Assumptions \ref{A2} and \ref{A3} hold. Then for $0<\lambda\leq \lambda_{\max},$ $t\in [nT,(n+1)T]$,
\[W_{1}\left(\mathcal{L}\left(\bar{\theta}_{t}^{\lambda}\right), \mathcal{L}\left(Z_{t}^{\lambda}\right)\right)\leq \sqrt{\lambda} \sqrt{e^{3a}(C_1+C_2+C_3)}+\sqrt{\lambda} z_1=
\sqrt{\lambda} (z_1+\sqrt{e^{3a}(C_1+C_2+C_3)})
\]
where $C_1$, $C_2$ and $z_1$ are given by (\ref{C_1}) , (\ref{C_2}) and (\ref{eq-z_1}) respectively.
\end{lemma}
\begin{proof}
Combining the results stated in Lemmas \ref{Lemma 4.7} and \ref{Lemma 4.8}
\[\begin{aligned}
    W_{1}\left(\mathcal{L}\left(\bar{\theta}_{t}^{\lambda}\right), \mathcal{L}\left(Z_{t}^{\lambda}\right)\right)&
\leq W_{1}\left(\mathcal{L}\left(\bar{\theta}_{t}^{\lambda}\right), \mathcal{L}\left(\bar{\zeta}_{t}^{\lambda, n}\right)\right)+W_{1}\left(\mathcal{L}\left(\bar{\zeta}_{t}^{\lambda, n}\right), \mathcal{L}\left(Z_{t}^{\lambda}\right)\right)
\\
&
\leq W_{2}\left(\mathcal{L}\left(\bar{\theta}_{t}^{\lambda}\right), \mathcal{L}\left(\bar{\zeta}_{t}^{\lambda, n}\right)\right)+W_{1}\left(\mathcal{L}\left(\bar{\zeta}_{t}^{\lambda, n}\right), \mathcal{L}\left(Z_{t}^{\lambda}\right)\right)
\\& \leq \sqrt{\lambda} \sqrt{e^{3a}(C_1+C_2+C_3)}+\sqrt{\lambda} z_1\\&=
\sqrt{\lambda} (z_1+\sqrt{e^{3a}(C_1+C_2+C_3)}),
\end{aligned}
\]
which yields the desired result.
\end{proof}

\begin{proof}[\textbf{Proof of Theorem \ref{Thrm1}}] By taking into consideration the result in the Lemma \ref{W1 convergence} and the property of $w_{1,2}$ in Proposition \ref{eberle}, one calculates
\[\begin{aligned}
W_{1}\left(\mathcal{L}\left(\theta_{t}^{\lambda}\right), \pi_{\beta}\right) &\leq W_{1}\left(\mathcal{L}\left(\bar{\theta}_{t}^{\lambda}\right), \mathcal{L}\left(Z_{t}^{\lambda}\right)\right)+W_{1}\left(\mathcal{L}\left(Z_{t}^{\lambda}\right), \pi_{\beta}\right)
\\&\leq \sqrt{\lambda} (z_1+\sqrt{e^{3a}(C_1+C_2+C_3)})  +\hat{c} e^{-\dot{c} \lambda t} w_{1,2}\left(\theta_{0}, \pi_{\beta}\right)
\\& \leq \sqrt{\lambda} (z_1+\sqrt{e^{3a}(C_1+C_2+C_3)})\\&\phantom{\sqrt{\lambda} (z_1+}+\hat{c} e^{-\dot{c} \lambda t}\left[1+\mathbb{E}\left[V_{2}\left(\theta_{0}\right)\right]+\int_{\mathbb{R}^{d}} V_{2}(\theta) \pi_{\beta}(d \theta)\right]
\\&\leq   \sqrt{\lambda} (z_1+\sqrt{e^{3a}(C_1+C_2+C_3)})\\&\phantom{\sqrt{\lambda} (z_1+}+\hat{c} e^{-\dot{c} n }\left[1+\mathbb{E}\left[V_{2}\left(\theta_{0}\right)\right]+\int_{\mathbb{R}^{d}} V_{2}(\theta) \pi_{\beta}(d \theta)\right]
\end{aligned}\]
where $C_1$, $C_2$ and $z_1$ are given by (\ref{C_1}) , (\ref{C_2}) and (\ref{eq-z_1}) respectively.
\end{proof}

\begin{lemma} \label{Lemma 4.10}
Let Assumptions \ref{A2} and \ref{A3} hold. Then, for $0<\lambda \leq \lambda_{max}$ and $t\in [nT,(n+1)T]$,
\[
W_{2}\left(\mathcal{L}\left(\bar{\zeta}_{t}^{\lambda, n}\right), \mathcal{L}\left(Z_{t}^{\lambda}\right)\right)\leq \lambda^{\frac{1}{4}}z_2
\]
where $z_2$ is given by \eqref{eq-z2}.
\end{lemma}
\begin{proof}
Using that $W_2\leq \sqrt{2w_{1,2}},$ one obtains
\begin{align*}
 W_{2}\left(\mathcal{L}\left(\bar{\zeta}_{t}^{\lambda, n}\right), \mathcal{L}\left(Z_{t}^{\lambda}\right)\right)
\leq & \sum_{k=1}^{n} W_{2}\left(\mathcal{L}\left(\bar{\zeta}_{t}^{\lambda, k}\right), \mathcal{L}\left(\bar{\zeta}_{t}^{\lambda, k-1}\right)\right)\\
\leq & \sum_{k=1}^{n} \sqrt{2} w_{1,2}^{1 / 2}\left(\mathcal{L}\left(\zeta_{t}^{k T, \bar{\theta}_{k T}^{\lambda}, \lambda}\right), \mathcal{L}\left(\zeta_{t}^{k T, \bar{\zeta}_{k T}^{\lambda, k-1}, \lambda}\right)\right)\\
\leq & \sqrt{2 \hat{c}} \sum_{k=1}^{n} \exp (-\dot{c}(n-k) / 2) W_{2}^{1 / 2}\left(\mathcal{L}\left(\bar{\theta}_{k T}^{\lambda}\right), \mathcal{L}\left(\bar{\zeta}_{k T}^{\lambda, k-1}\right)\right) \\
& \phantom{\sqrt{2 \hat{c}} }\times\left[1+\left\{\mathbb{E}\left[V_{4}\left(\bar{\theta}_{k T}^{\lambda}\right)\right]\right\}^{1 / 2}+\left\{\mathbb{E}\left[V_{4}\left(\bar{\zeta}_{k T}^{\lambda, k-1}\right)\right]\right\}^{1 / 2}\right]^{1 / 2}\\
\leq &  \lambda^{-1 / 4} \sqrt{2 \hat{c}} \sum_{k=1} \exp (-\dot{c}(n-k) / 2) W_{2}\left(\mathcal{L}\left(\bar{\theta}_{k T}^{\lambda}\right), \mathcal{L}\left(\bar{\zeta}_{k T}^{\lambda, k-1}\right)\right) \\
&+\lambda^{1 / 4} \sqrt{2 \hat{c}} \sum_{k=1}^{n} \exp (-\dot{c}(n-k) / 2)\\ & \times \left[1+\left\{\mathbb{E}\left[V_{4}\left(\bar{\theta}_{k T}^{\lambda}\right)\right]\right\}^{1 / 2}+\left\{\mathbb{E}\left[V_{4}\left(\bar{\zeta}_{k T}^{\lambda, k-1}\right)\right]\right\}^{1 / 2}\right]\\
= & \lambda^{\frac{1}{4}}\sqrt{2\hat{c}}\frac{1}{1-exp(-\dot{c}/2)}e^{3a}(C_1+C_2+C_3)\\
&+
\lambda^{\frac{1}{4}}\sqrt{2\hat{c}}\frac{1}{1-exp(-\dot{c}/2)} \\ & \times\left[1+\sqrt{2\mathbb{E}|\theta_0|^4+2+2C'_2 +\frac{\tilde{c}(4)}{\bar{c}(4)}}+\sqrt{2\mathbb{E}|\theta_0|^4+2+2C'_2}\right] \\
= &\lambda^{\frac{1}{4}}z_2
\end{align*}
where
\begin{equation}\label{eq-z2}
\begin{aligned}
z_2&=\sqrt{2\hat{c}}\frac{1}{1-exp(-\dot{c}/2)}\\&\times\left[e^{3a}(C_1+C_2+C_3)+1+\sqrt{2\mathbb{E}|\theta_0|^4+2+2C'_2 +\frac{\tilde{c}(4)}{\bar{c}(4)}}+\sqrt{2\mathbb{E}|\theta_0|^4+2+2C'_2}\right].
\end{aligned}
\end{equation}
\end{proof}

\begin{proof}[\textbf{Proof of Corollary \ref{Thrm2}}]
Combining Lemma \ref{Lemma 4.7} and Lemma \ref{Lemma 4.10}, one obtains
\begin{equation*}
    \begin{aligned}
W_{2}\left(\mathcal{L}\left(\theta_{t}^{\lambda}\right), \pi_{\beta}\right)  \leq & W_{2}\left(\mathcal{L}\left(\theta_{t}^{\lambda}\right), \mathcal{L}\left(Z_{t}^{\lambda}\right)\right)+W_{2}\left(\mathcal{L}\left(Z_{t}^{\lambda}\right), \pi_{\beta}\right) \\
 \leq & W_{2}\left(\mathcal{L}\left(\bar{\theta}_{t}^{\lambda}\right), \mathcal{L}\left(\bar{\zeta}_{t}^{\lambda, n}\right)\right)\\&+W_{2}\left(\mathcal{L}\left(\bar{\zeta}_{t}^{\lambda, n}\right), \mathcal{L}\left(Z_{t}^{\lambda}\right)\right)+W_{2}\left(\mathcal{L}\left(Z_{t}^{\lambda}\right), \pi_{\beta}\right)
\\ \leq & \sqrt{e^{3a}(C_1+C_2+C_3)}\sqrt{\lambda}+z_2\lambda^\frac{1}{4}+\sqrt{2 w_{1,2}\left(\mathcal{L}\left(Z_{t}^{\lambda}\right), \pi_{\beta}\right)}
\\ \leq & \sqrt{e^{3a}(C_1+C_2+C_3)}\sqrt{\lambda}+z_2\lambda^\frac{1}{4}+\hat{c}^{1 / 2} e^{-\dot{c} \lambda t / 2} \sqrt{2 w_{1,2}\left(\theta_{0}, \pi_{\beta}\right)}
\\ \leq & \sqrt{e^{3a}(C_1+C_2+C_3)}\sqrt{\lambda}+z_2\lambda^\frac{1}{4} +\sqrt{2} \hat{c}^{1 / 2} e^{-\dot{c} \lambda t / 2} \\ & \times \left(1+\mathbb{E}\left[V_{2}\left(\theta_{0}\right)\right]+\int_{\mathbb{R}^{d}} V_{2}(\theta) \pi_{\beta}(d \theta)\right)^{1 / 2}\\
\leq & \sqrt{e^{3a}(C_1+C_2+C_3)}\sqrt{\lambda}+z_2\lambda^\frac{1}{4}
\\&+\sqrt{2} \hat{c}^{1 / 2} e^{-\dot{c} n/ 2}\left(1+\mathbb{E}\left[V_{2}\left(\theta_{0}\right)\right]+\int_{\mathbb{R}^{d}} V_{2}(\theta) \pi_{\beta}(d \theta)\right)^{1 / 2}.
\end{aligned}
\end{equation*}
\end{proof}

\begin{proof}[\textbf{Proof of Lemma \ref{T_1}}]
Taking into account that \[h(\theta)=\E[G(\theta,X_0)] +\eta |\theta|^{2r}\theta\]
and the polynomial growth of $G$ in \eqref{G_growth}, there exist $r_1=\E[K(X_0)] +\eta $, $r_2=2 \E[K(X_0)]$, such that \[|h(\theta)|\leq r_1 |\theta|^l+r_2 \quad \forall \theta \in \mathbb{R}^d,
\]
where $l=2r+1$. As a result,
\[\begin{aligned}
u(w)-u(v) &=\int_{0}^{1}\langle w-v, \nabla u((1-t) v+t w)\rangle \mathrm{d} t \\
& \leq \int_{0}^{1}|\nabla u((1-t) v+t w)||w-v| \mathrm{d} t \\
& \leq \int_{0}^{1}\left(a_1(1-t)^l|v|^l+a_1 t^l|w|^l+r_{2}\right)|w-v| \mathrm{d} t \\
&=\left(\frac{a_1}{l+1}|v|^l+\frac{a_1}{l+1}|w|^l+r_{2}\right)|w-v|
\end{aligned}\]
where $a_1=2^l r_1$ .
Let $P$ the coupling of $\mu$ and $\nu$ that achieves $W_2(\mu,\nu)$, that is $P=\left(\mathcal{L}(W),\mathcal{L}(V)\right) $ with $\mu=\mathcal{L}(W)$ and $\nu=\mathcal{L}(V)$.
Taking a closer look one notices that
\[
\begin{aligned}
\int_{\mathbb{R}^{d}} u \mathrm{d} \mu-\int_{\mathbb{R}^{d}} u \mathrm{d} v &=\mathbb{E}_{P}[u(W)-u(V)] \\
& \leq \sqrt{\mathbb{E}_{P}\left(\frac{a_1}{l+1}|W|^l+\frac{a_1}{l+1}|V|^l+r_2\right)^{2}} \cdot \sqrt{\mathbb{E}_{\mathrm{P}}\left[|W-V|^{2}\right]} \\
& \leq\left(\frac{a_1}{l+1} \sqrt{\E|W|^{2l}}+\frac{a_1}{l+1} \sqrt{\E|V|^{2l}}+r_2\right) \cdot \mathcal{W}_{2}(\mu, v)
\end{aligned}
\]
Applying this to the particular case where $W=\theta_n^\lambda$ and $V=\theta_\infty$ yields
\[
\E u(\theta_n^\lambda)-\E u(\theta_\infty)\leq\left( \frac{a_1}{l+1}\sqrt{\E|\theta_0|^{2l}+C'_l}+\frac{a_1}{l+1}\sqrt{\sigma_{2l}}+r_2\right)
W_2\left(\mathcal{L}\left(\theta_{n}^{\lambda}\right),\pi_\beta\right)
\]

where $\sigma_{2l}$ is the $2l$-moment of $\pi_\beta$.
\end{proof}

\begin{proof}[\textbf{Proof of Lemma \ref{T_2}}] A similar approach as in \cite[Section~3.5]{raginsky} is employed here, however due to the difference in the smoothness condition for $H$ (and consequently for $h$), see our Proposition \ref{polip H} in contrast to global Lipschitzness which is required in \cite{raginsky}, we provide the details for obtaining a bound for $\log\Lambda$. Recall that $\Lambda$ represents the normalizing constant, i.e.
\[
\Lambda:= \int_{\R^d} e^{-\beta u(\theta)}d\theta.
\]
Initially, one observes that due to the monotonicity condition \eqref{eq-AB},
\[
\langle \theta^*, h(\theta^*)\rangle
\geq A|\theta^*|^2-B \implies |\theta^*|\leq \sqrt{\frac{B}{A}}\leq R_0.
\]
Consequently, one calculates that
\[
\begin{aligned}
u_*-u(w)&=\int_{0}^1 \langle h(w+t(\theta^*-w),\theta^*-w\rangle dt
\\&=\int_{0}^1 \langle h(w+t(\theta^*-w)-h(\theta^*),\theta^*-w\rangle dt
\\&=\int_{0}^1 \frac{1}{t-1}\langle h(w+t(\theta^*-w)-h(\theta^*),w-\theta^*+t(\theta^*-w)\rangle dt,
\end{aligned}
\]
which due to the polynomial lipschitzness of $h$ yields that
\begin{align} \label{eq-smoothing}
-\beta(u_*-u(w))& = \beta|u_*-u(w)| \nonumber \\
&\leq \beta\int_{0}^1 \frac{1}{1-t}\left|\langle h\left(w+t(\theta^*-w)\right)-h(\theta^*),w-\theta^*+t(\theta^*-w)\rangle\right| dt \nonumber \\
&\leq \int_{0}^1 b'(1+|w|+|\theta^*-w|+|\theta^*|)^l(1-t) |w-\theta^*|^2 dt \nonumber \\
&\leq
b'(1+2|\theta^*|+2 |\theta^*-w|)^l\frac{|w-\theta^*|^2}{2},
\end{align}
where $b'=L\mathbb{E}(1+|X_0|)^\rho \beta$. As a result,
\begin{equation*}
    \begin{aligned}
   I =\int_{\mathbb{R}^d} e^{\beta (u_*-u(w))}dw  & \geq
      \int_{\mathbb{R}^d} e^{-b'(1+2|w-\theta^*|+2|\theta^*|)^l (\frac{|w-\theta^*|^2}{2})}dw
      \\& \geq  \int_{ \bar {B}(\theta^*,R_0 )} e^{-b'(1+4R_0)^l (\frac{|w-\theta^*|^2}{2})}dw
      \\&= \left(\frac{2\pi}{b''}\right)^\frac{d}{2}\int_{ \bar {B}(\theta^*,R_0) }f_X(w) dw
    \end{aligned}
\end{equation*}
where $b''=b'(1+4R_0)^l$, $f$ is a density function of a multivariate normal variable $X$ with mean $\theta^*$ and covariance matrix $V=1/b'' I_d$, where $I_d$ is the $d$-dimensional identity matrix.
This means that $\sqrt{b''}(X-\theta^*)$ follows a standard d-dimensional Gaussian distribution. Applying the standard concentration inequality for d-dimensional Gaussian yields
\[\begin{aligned}
   P\left(||X-\theta^*||>R_0\right)&=P\left(||\sqrt{b''}(X-\theta^*)||>R_0\sqrt{b''}\right)\\&\leq
   P\left(||\sqrt{b''}(X-\theta^*)||-\sqrt{d}>R_0\sqrt{b''}-\sqrt{d}\right)\\&\leq e^{-(R_0\sqrt{b''}-\sqrt{d})^2}
\end{aligned}\]
which leads to
\[I\geq \left(\frac{2\pi}{b''}\right)^\frac{d}{2}\left( 1-e^{-(R_0\sqrt{b''}-\sqrt{d})^2} \right).\]
Consequently, following \cite[Section~3.5]{raginsky}, one obtains
\[
\log \Lambda\geq -\beta u_*+\frac{d}{2} \log \left(\frac{2 \pi}{b'' }\right)+ \log \left(1-e^{-(R_0\sqrt{b''}-\sqrt{d})^2} \right).\]
Thus, by setting $K:=b''/\beta=L\E (1+|X_0|)^\rho (1+4R_0)^l$ and in view of \eqref{eq-AB} and \cite[Lemma~3]{raginsky}, one obtains
\[
\E u(\theta_\infty)-u_*\leq \frac{d}{2 \beta} \log \left(\frac{e K}{A}\left(\frac{B \beta}{d}+1\right)\right)-\frac{1}{\beta}\log\left(1-e^{-(R_0\sqrt{K\beta}-\sqrt{d})^2}\right).
\]
\end{proof}

\subsection{Complementary details to Section \ref{MNNs}}
	
We start with an easy observation about the equivalence of the operator norm and Euclidean norm of a linear operator. For any $k,l\in\N_+$, $W\in\lin{\R^k,\R^l}$ and $z\in\R^k$,
	\begin{align*}
		|W z|^2 &= \sum_{i=1}^l [Wz]_i^2
		=\sum_{i=1}^l \left(\sum_{j=1}^k W_{ij}z_j\right)^2 \\
		&\le\left(\sum_{j=1}^k z_j^2\right)\sum_{i=1}^l\sum_{j=1}^k W_{ij}^2
		=|z|^2|W|^2.
	\end{align*}
	On the other hand, if $l\le k$, then
	$|W|^2=\sum_{i=1}^l\sum_{j=1}^k W_{ij}^2 =\sum_{i=1}^l [WW^\ast]_{ii}\le l\Vert W\Vert^2$
	and similarly, for $k\le l$, $|W|^2\leq k\Vert W^\ast\Vert^2=k\Vert W\Vert^2$.
	As a result, we obtain
	\begin{equation}\label{eq:equiv}
		\Vert W\Vert\le |W|\le\min (\sqrt{k},\sqrt{l})\Vert W\Vert.
	\end{equation}
	In particular, if $k=1$ or $l=1$ then the Euclidean and operator norms coincide. As easily seen, for any $\eta\in C_b(\R)$, $W\in\lin{\R^k,\R^l}$ and $z\in\R^k$,
	\begin{align}
		|\eta_{W}(z)| &\le \sqrt{l}\vinf{\eta},\label{eq:bd1} \\
		\Vert M_{\eta_{W}(z)} \Vert &\le \vinf{\eta}\label{eq:bd2}.
	\end{align}
The next lemma establishes upper bound on the norm of $\partial_\theta f(\theta,\mathbf{z})$
	involving an order $n$ polynomial of $|\theta|$.
	
	\begin{lemma}\label{lem:deriv}
		Let $\theta = (\phi,\mathbf{w})\in\Theta$ and
		$x=(\mathbf{z},y)\in\R^{m-1}\times\R$ arbitrary.
		Then, for the Euclidean norm of the partial derivatives of the regression function with respect to the learning parameter, we have
		\begin{equation}\label{eq:dfnorm}
		|\partial_\theta f(\theta,\mathbf{z})|
		\le
		D^{1/2}\sqrt{n+1}(1+|x|)(1+\vsob{\sigma})^{n+1}(1+|\theta|^n).
		\end{equation}
		
		Furthermore, for the operator norm of the partial derivatives of nonlinear maps appearing in the definition of $f$,
see \eqref{eq:fdef},
one obtains that
		\begin{equation}\label{eq:dsnorm}
	    \lvrv{\partial_{W_i}\sigma\left(\mathbf{w}_1^n,\mathbf{z}\right)}
		\le \sqrt{D}(1+|x|^{})(1+\vsob{\sigma})^{n-i+2}
		|\theta|^{n-i} \quad i=1,\ldots,n
		\end{equation}
		holds.
		
	\end{lemma}
	\begin{proof}
	In what follows, we calculate $\partial_\theta f(\theta,\mathbf{z})\in \Theta^\ast$
	at a fixed $\theta\in\Theta$ and $\mathbf{z}\in\R^{m-1}$. For any $\tilde{\theta}=(\tilde{\phi},\tilde{\mathbf{w}})$, where $\tilde{\mathbf{w}}=(\tilde{W_1},\ldots,\tilde{W_n})$,
	\begin{equation*}
		\partial_\theta f(\theta,\mathbf{z})(\tilde{\theta}) =
		\partial_{\phi} f(\theta,\mathbf{z})(\tilde{\phi})+\sum_{i=1}^{n}
		\partial_{W_i} f(\theta,\mathbf{z})(\tilde{W_i}).
	\end{equation*}
	The map $\phi\mapsto f((\phi,\mathbf{w}),\mathbf{z})$ is linear hence, by \eqref{eq:bd1},
	\begin{equation*}
		|\partial_{\phi} f(\theta,\mathbf{z})|=|\sigma\left(\mathbf{w}_1^n,\mathbf{z}\right)|
		\le\sqrt{d_n}\vinf{\sigma}.
	\end{equation*}
Moreover,
	\begin{equation*}
		\partial_{W_i} f(\theta,\mathbf{z})(\tilde{W_i}) = \phi
		\circ\partial_{W_i}\sigma\left(\mathbf{w}_1^n,\mathbf{z}\right)(\tilde{W_i}).
	\end{equation*}
Thus, by the chain rule, for $i=1,\ldots,n$, one deduces that
	\begin{align}\label{eq:dsig}
	\begin{split}
		\partial_{W_i}\sigma\left(\mathbf{w}_1^n,\mathbf{z}\right)(\tilde{W_i})
		&=\left[
		\prod_{j=1}^{n-i} \partial_\mathbf{z}\sigma_{W_{n-j+1}}\left(
		\sigma
		\left(\mathbf{w}_1^{n-j},\mathbf{z}
		\right)
		\right)
		\right]
		\partial_{W_i}\sigma_{W_i}\left(\sigma\left(\mathbf{w}_1^{i-1},\mathbf{z}\right)\right)(\tilde{W_i})
		\\
		&=\left[
		\prod_{j=1}^{n-i}
		M_{\sigma'_{W_{n-j+1}}\left(\sigma
			\left(\mathbf{w}_1^{n-j},\mathbf{z}
			\right)\right)}
		 W_{n-j+1}
		\right]
		M_{\sigma'_{W_i}\left(\sigma\left(\mathbf{w}_1^{i-1},\mathbf{z}\right)\right)}\tilde{W_i}\sigma\left(\mathbf{w}_1^{i-1},\mathbf{z}\right).
		\end{split}
	\end{align}
	Furthermore, by \eqref{eq:bd1} and \eqref{eq:bd2}, and the sub-multiplicativity of the operator norm, one obtains the first inequality	
	\begin{align*}
		\lvrv{\partial_{W_i}\sigma\left(\mathbf{w}_1^n,\mathbf{z}\right)}
		&\le
		\sqrt{d_{i-1}}(\vinf{\sigma}+|\mathbf{z}|)\vinf{\sigma'}^{n-i+1}\prod_{j=1}^{n-i}
		\Vert W_{n-j+1}\Vert \\
		&\le
		\sqrt{D}(1+|x|)(1+\vsob{\sigma})^{n-i+2}
		|\theta|^{n-i},
	\end{align*}
since, by definition, $\sigma (\mathbf{w}_1^0,\mathbf{z})=\mathbf{z}$. In addition, due to the properties of the Euclidean norm,
		\begin{align*}
	|\partial_\theta f(\theta,\mathbf{z})|^2 = &
	|\partial_{\phi} f(\theta,\mathbf{z})|^2+\sum_{i=1}^{n}
	|\partial_{W_i} f(\theta,\mathbf{z})|^2 \\ \le &
	d_n\vinf{\sigma}^2 + D |\phi|^2\sum_{i=1}^{n}
(1+|x|^{})^2(1+\vsob{\sigma})^{2(n-i+2)}
|\theta|^{2(n-i)} \\ \le &
D(1+|x|)^2(1+\vsob{\sigma})^{2(n+1)}
\sum_{i=0}^{n}|\theta|^{2(n-i+1)} \\ \le &
D(n+1)(1+|x|)^2(1+\vsob{\sigma})^{2(n+1)}(1+|\theta|^{2n}).
	\end{align*}
Finally, the subadditivity of the square root function yields that
	\begin{equation*}
		|\partial_\theta f(\theta,\mathbf{z})|
		\le
		D^{1/2}\sqrt{n+1}(1+|x|)(1+\vsob{\sigma})^{n+1}(1+|\theta|^n)
	\end{equation*}
which completes the proof.
\end{proof}

\begin{corollary}\label{cor:diff}
	Let $\theta,\theta'\in\Theta$ and $x\in\R^m$ be such that $\theta = (\phi,\mathbf{w}_1^{n})$,
	$\theta = (\phi',\mathbf{w'}_1^{n})$ and $x=(\mathbf{z},y)$, where $\mathbf{w}_1^{n},\mathbf{w'}_1^{n}\in\bigoplus_{i=1}^n\lin{\R^{d_{i-1}},\R^{d_i}}$,  $\phi,\phi'\in\left(\R^{d_n}\right)^\ast$ and $x\in\R^m$ are arbitrary. Then, by Lemma \ref{lem:deriv}, for $t\in [0,1]$ and $i=1,\ldots,n$, follows that
	\begin{align*}
		\lvrv{
			\partial_{\mathbf{w}_1^i}\sigma ((1-t)\mathbf{w}_1^i+t\mathbf{w'}_1^i,\mathbf{z})
		}^2
		\le &
		\sum_{j=1}^i
		\lvrv{
			\partial_{W_j} \sigma ((1-t)\mathbf{w}_1^i+t\mathbf{w'}_1^i,\mathbf{z})
		}^2
		\\ \le&
		D(1+|x|)^2\sum_{j=1}^i   (1+\vsob{\sigma})^{2(n-j+2)}
		|(1-t)\theta+t\theta'|^{2(n-j)}
		\\ \le&
		n D (1+|x|)^2 (1+\vsob{\sigma})^{2(n+1)} (1+|\theta|+|\theta'|)^{2(n-1)}
	\end{align*}
	which leads to the uniform estimate
	\begin{align*}
			\left|
		\sigma\left(\mathbf{w}_1^{i},\mathbf{z}\right)
		-
		\sigma\left(\mathbf{w'}_1^{i},\mathbf{z}\right)
		\right|
		\le & \sup_{t\in [0,1]}\Vert
		\partial_{\mathbf{w}_1^i}\sigma ((1-t)\mathbf{w}_1^i+t\mathbf{w'}_1^i,\mathbf{z})
		\Vert |\mathbf{w}_1^{i}-\mathbf{w'}_1^{i}|
		\\ \le&
		D^{1/2}\sqrt{n}(1+|x|)(1+\vsob{\sigma})^{n+1}(1+|\theta|+|\theta'|)^{n-1} |\mathbf{w}_1^{i}-\mathbf{w'}_1^{i}| \\ \le&
		D^{1/2}\sqrt{n}(1+|x|)(1+\vsob{\sigma})^{n+1}(1+|\theta|+|\theta'|)^{n-1} |\theta-\theta'|
	\end{align*}
	$i=1,\ldots,n$.
	
\end{corollary}	
	
\begin{lemma}\label{lem:plip}
	Let $\theta,\theta'\in\Theta$ and $x\in\R^m$ be such that $\theta = (\phi,\mathbf{w}_1^{n})$,
$\theta = (\phi',\mathbf{w'}_1^{n})$ and $x=(\mathbf{z},y)$, where $\mathbf{w}_1^{n},\mathbf{w'}_1^{n}\in\bigoplus_{i=1}^n\lin{\R^{d_{i-1}},\R^{d_i}}$,  $\phi,\phi'\in\left(\R^{d_n}\right)^\ast$ and $x\in\R^m$ are arbitrary. Then, for $i=1,\ldots,n$, we have
\begin{align*}
\lvrv{\partial_{W_i}\sigma (\mathbf{w}_1^n,\mathbf{z})
	-
	\partial_{W_i}\sigma (\mathbf{w'}_1^n,\mathbf{z})
}
=
2\sqrt{n}D(1+|x|)^2(1+\vsob{\sigma})^{2n-i+4}
(1+|\theta|+|\theta'|)^{2n-i}|\theta-\theta'|.
\end{align*}
\end{lemma}	
\begin{proof}
	Let $i\in\{1,\ldots,n\}$ be arbitrary and fixed.
	By the definition of $\sigma (\mathbf{w}_1^n,\mathbf{z})$, for $k<n$,
	$\sigma\left(\mathbf{w}_1^{k+1},\mathbf{z}\right) = \sigma_{W_{k+1}}\circ\sigma\left(\mathbf{w}_1^{k},\mathbf{z}\right)$.  Hence, for $i\le k<n$,
\begin{equation*}
	\partial_{W_i}\sigma\left(\mathbf{w}_1^{k+1},\mathbf{z}\right)
	=
	M_{\sigma'_{W_{k+1}}\left(\sigma
		\left(\mathbf{w}_1^{k},\mathbf{z}
		\right)\right)}
	W_{k+1}
	\partial_{W_i}\sigma\left(\mathbf{w}_1^{k},\mathbf{z}\right)
	\end{equation*}
	which implies that
	\begin{align*}
	\lvrv{
	\partial_{W_i}\sigma\left(\mathbf{w}_1^{k+1},\mathbf{z}\right)
	-
	\partial_{W_i}\sigma\left(\mathbf{w'}_1^{k+1},\mathbf{z}\right)
	}    \le &
	\lvrv{
	M_{\sigma'_{W_{k+1}}\left(\sigma
		\left(\mathbf{w}_1^{k},\mathbf{z}
		\right)\right)}
	-
	M_{\sigma'_{W'_{k+1}}\left(\sigma
		\left(\mathbf{w'}_1^{k},\mathbf{z}
		\right)\right)}
	}
	\\ & \times\lvrv{W_{k+1}}\lvrv{\partial_{W_i}\sigma\left(\mathbf{w}_1^{k},\mathbf{z}\right)}
	\\ & +
	\lvrv{ M_{\sigma'_{W'_{k+1}}\left(\sigma
		\left(\mathbf{w'}_1^{k},\mathbf{z}
		\right)\right)}
	}
	\lvrv{
		W_{k+1}
		-
		W'_{k+1}
	}
	\\ & \times\lvrv{
		\partial_{W_i}\sigma\left(\mathbf{w}_1^{k},\mathbf{z}\right)
	} \\ &+
	\lvrv{ M_{\sigma'_{W'_{k+1}}\left(\sigma
		\left(\mathbf{w'}_1^{k},\mathbf{z}
		\right)\right)}
	}
	\\ & \times\Vert
	W'_{k+1}
	\Vert
	\lvrv{
	\partial_{W_i}\sigma\left(\mathbf{w}_1^{k},\mathbf{z}\right)
	-
	\partial_{W_i}\sigma\left(\mathbf{w'}_1^{k},\mathbf{z}\right)
	} \\ \le &
	\vinf{\sigma'}|\mathbf{w'}_1^n|
	\lvrv{
		\partial_{W_i}\sigma\left(\mathbf{w}_1^{k},\mathbf{z}\right)
		-
		\partial_{W_i}\sigma\left(\mathbf{w'}_1^{k},\mathbf{z}\right)
	} \\
    &+
    \lvrv{
    	\partial_{W_i}\sigma\left(\mathbf{w}_1^{k},\mathbf{z}\right)
    }\\ & \times\bigg[
    \lvrv{
    	M_{\sigma'_{W_{k+1}}\left(\sigma
    		\left(\mathbf{w}_1^{k},\mathbf{z}
    		\right)\right)}
    	-
    	M_{\sigma'_{W'_{k+1}}\left(\sigma
    		\left(\mathbf{w'}_1^{k},\mathbf{z}
    		\right)\right)}
    } |\mathbf{w}_1^n|  \\ & \phantom{+\lvrv{
    	\partial_{W_i}\sigma\left(\mathbf{w}_1^{k},\mathbf{z}\right)
    }\bigg[}+
    \vinf{\sigma'} |\mathbf{w}_1^n-\mathbf{w'}_1^n|
    \bigg]
	\end{align*}
	holds for the corresponding operator norms. Further, for $i=1,\ldots,n$ and by taking into consideration Corollary \ref{cor:diff}, one obtains that
\begin{align}\label{eq:Mdiff}
\lvrv{
M_{\sigma'_{W_i}\left(\sigma\left(\mathbf{w}_1^{i-1},\mathbf{z}\right)\right)}
-
M_{\sigma'_{W'_i}\left(\sigma\left(\mathbf{w'}_1^{i-1},\mathbf{z}\right)\right)}
} = &
\lvrv{
M_{\sigma'_{W_i}\left(\sigma\left(\mathbf{w}_1^{i-1},\mathbf{z}\right)\right)
	-
	\sigma'_{W'_i}\left(\sigma\left(\mathbf{w'}_1^{i-1},\mathbf{z}\right)\right)}
} \nonumber \\ \le &
\vinf{\sigma'_{W_i}\left(\sigma\left(\mathbf{w}_1^{i-1},\mathbf{z}\right)\right)
	-
	\sigma'_{W'_i}\left(\sigma\left(\mathbf{w'}_1^{i-1},\mathbf{z}\right)\right)}
\nonumber \\ \le &
\vlip{\sigma'}\bigg(\Vert W_{i}\Vert
\left|
\sigma\left(\mathbf{w}_1^{i-1},\mathbf{z}\right)
-
\sigma\left(\mathbf{w'}_1^{i-1},\mathbf{z}\right)
\right|
\nonumber \\ & \phantom{\vlip{\sigma'}\bigg(}+
\Vert W_{i}-W'_{i}\Vert
\left|\sigma\left(\mathbf{w'}_1^{i-1},\mathbf{z}\right)
\right|
\bigg)
\nonumber \\ \le &
D^{1/2}\sqrt{n}
(1+\vsob{\sigma})^{n+2}(1+|x|)
(1+|\theta|+|\theta'|)^n |\theta-\theta'|
\end{align}
which is uniform in $i$. Combining these with inequality \eqref{eq:dsnorm} in Lemma \ref{lem:deriv}, for $i\le k<n$, one obtains the following recursive estimate 	
\begin{align}\label{eq:rec}
\begin{split}
	\lvrv{
		\partial_{W_i}\sigma\left(\mathbf{w}_1^{k+1},\mathbf{z}\right)
		-
		\partial_{W_i}\sigma\left(\mathbf{w'}_1^{k+1},\mathbf{z}\right)
	}
&\le
A
\lvrv{
\partial_{W_i}\sigma\left(\mathbf{w}_1^{k},\mathbf{z}\right)
-
\partial_{W_i}\sigma\left(\mathbf{w'}_1^{k},\mathbf{z}\right)
}+B A^{n+k-i+1},
\end{split}
\end{align}	
where
\begin{align*}
A & = (1+\vsob{\sigma})(1+|\theta|+|\theta'|)\\
B & = 2\sqrt{n}D(1+|x|)^2(1+\vsob{\sigma})^4|\theta-\theta'|.
\end{align*}	
By induction, for $i=1,\ldots,n$, one deduces that
\begin{align}\label{eq:c}
\begin{split}
\lvrv{
	\partial_{W_i}\sigma\left(\mathbf{w}_1^{n},\mathbf{z}\right)
	-
	\partial_{W_i}\sigma\left(\mathbf{w'}_1^{n},\mathbf{z}\right)
} 	
&\le
A^{n-i}
\lvrv{
	\partial_{W_i}\sigma\left(\mathbf{w}_1^{i},\mathbf{z}\right)
	-
	\partial_{W_i}\sigma\left(\mathbf{w'}_1^{i},\mathbf{z}\right)
}+(n-i)B A^{2n-i}.
\end{split}
\end{align} 	
Using basic properties of the operator norm and inequality \eqref{eq:equiv}, for $i=n$, yields that	
\begin{align*}
\left|
	\partial_{W_i}\sigma\left(\mathbf{w}_1^i,\mathbf{z}\right)(\tilde{W_i})
-
	\partial_{W_i}\sigma\left(\mathbf{w'}_1^i,\mathbf{z}\right)(\tilde{W_i})
\right|
=&
	\bigg|
		M_{\sigma'_{W_i}\left(\sigma\left(\mathbf{w}_1^{i-1},\mathbf{z}\right)\right)}\tilde{W_i}\sigma\left(\mathbf{w}_1^{i-1},\mathbf{z}\right)
	\\ & 	-
		M_{\sigma'_{W'_i}\left(\sigma\left(\mathbf{w'}_1^{i-1},\mathbf{z}\right)\right)}\tilde{W_i}\sigma\left(\mathbf{w'}_1^{i-1},\mathbf{z}\right)
	\bigg| \\ \le &
	\lvrv{
		M_{\sigma'_{W_i}\left(\sigma\left(\mathbf{w}_1^{i-1},\mathbf{z}\right)\right)}
	}
		\left|
	\sigma\left(\mathbf{w}_1^{i-1},\mathbf{z}\right)
	-
	\sigma\left(\mathbf{w'}_1^{i-1},\mathbf{z}\right)
	\right|
	 |\tilde{W_i}|
	\\ & +
	\lvrv{
	M_{\sigma'_{W_i}\left(\sigma\left(\mathbf{w}_1^{i-1},\mathbf{z}\right)\right)}
	-
	M_{\sigma'_{W'_i}\left(\sigma\left(\mathbf{w'}_1^{i-1},\mathbf{z}\right)\right)}
	}
	\\ & \times\left|
	\sigma\left(\mathbf{w'}_1^{i-1},\mathbf{z}\right)
	\right|
	 |\tilde{W_i}|
\end{align*}	
which, due to Corollary \ref{cor:diff} and \eqref{eq:Mdiff}, implies that
\begin{align*}
	\lvrv{\partial_{W_i}\sigma\left(\mathbf{w}_1^i,\mathbf{z}\right)
		-
		\partial_{W_i}\sigma\left(\mathbf{w'}_1^i,\mathbf{z}\right)} \le &
		\lvrv{
		M_{\sigma'_{W_i}\left(\sigma\left(\mathbf{w}_1^{i-1},\mathbf{z}\right)\right)}
	}
	\left|
	\sigma\left(\mathbf{w}_1^{i-1},\mathbf{z}\right)
	-
	\sigma\left(\mathbf{w'}_1^{i-1},\mathbf{z}\right)
	\right| \\
	&+
		\lvrv{
		M_{\sigma'_{W_i}\left(\sigma\left(\mathbf{w}_1^{i-1},\mathbf{z}\right)\right)}
		-
		M_{\sigma'_{W'_i}\left(\sigma\left(\mathbf{w'}_1^{i-1},\mathbf{z}\right)\right)}
	}
	\left|
	\sigma\left(\mathbf{w'}_1^{i-1},\mathbf{z}\right)
	\right|
	\\
	\le &
	B A^n.
\end{align*}
Finally, combine this estimate with \eqref{eq:c} yields that
\begin{align*}
\lvrv{\partial_{W_i}\sigma (\mathbf{w}_1^n,\mathbf{z})
	-
	\partial_{W_i}\sigma (\mathbf{w'}_1^n,\mathbf{z})
}
&\le
(n-i+1)BA^{2n-i}
\\
&=
2\sqrt{n}D(1+|x|)^2(1+\vsob{\sigma})^{2n-i+4}
(1+|\theta|+|\theta'|)^{2n-i}|\theta-\theta'|
\end{align*}
which completes the proof.

\end{proof}	
	
\begin{lemma}\label{lem:dlip}
	Let $x=(\mathbf{z},y)$, where $\mathbf{z}\in\R^{m-1}$ and $y\in\R$ are arbitrary.
	Then, for any $\theta,\theta'\in\Theta$,
	\begin{align*}
	|\partial_\theta f(\theta,\mathbf{z})-\partial_\theta f(\theta',\mathbf{z})|
	&\le
	4(n+1) D (1+|x|)^2 (1+\vsob{\sigma})^{2n+3}
	(1+|\theta|+|\theta'|)^{2n}|\theta-\theta'|.
	\end{align*}
\end{lemma}
\begin{proof}
For the Euclidean norm of the partial derivative of the regression function with respect to the learning parameter, we have
\begin{equation*}
|\partial_\theta f(\theta,\mathbf{z})|^2 =
|\sigma (\mathbf{w}_1^n,\mathbf{z})|^2
+
\sum_{i=1}^n
|\phi\circ\partial_{W_i}\sigma (\mathbf{w}_1^n,\mathbf{z})|^2
\end{equation*}
and thus we have
\begin{align*}
|\partial_\theta f(\theta,\mathbf{z})-\partial_\theta f(\theta',\mathbf{z})|^2
&=
|\sigma (\mathbf{w}_1^n,\mathbf{z})-\sigma (\mathbf{w'}_1^n,\mathbf{z})|^2 +
\sum_{i=1}^n
|\phi\circ\partial_{W_i}\sigma (\mathbf{w}_1^n,\mathbf{z})-\phi'\circ\partial_{W_i}\sigma (\mathbf{w'}_1^n,\mathbf{z})|^2.
\end{align*}
Using Lemma \ref{lem:deriv} and \ref{lem:plip}, one deduces that
\begin{align*}
	\left|
	\phi\circ\partial_{W_i}\sigma (\mathbf{w}_1^n,\mathbf{z})
	-
	\phi'\circ\partial_{W_i}\sigma (\mathbf{w'}_1^n,\mathbf{z})
	\right|^2
	\le &
	2\bigg(
	|\phi|^2
	\lvrv{
		\partial_{W_i}\sigma (\mathbf{w}_1^n,\mathbf{z})
		-
		\partial_{W_i}\sigma (\mathbf{w'}_1^n,\mathbf{z})	
	}^2
	\\ & \phantom{2\bigg(} +
	|\phi-\phi'|^2
	\lvrv{\partial_{W_i}\sigma (\mathbf{w'}_1^n,\mathbf{z})}^2
	\bigg)\\
	\le &
	8nD^2(1+\vsob{\sigma})^{2(2n-i+4)}(1+|\theta|+|\theta'|)^{2+4n-2i}|\theta-\theta'|^2
\\
	&+
	2D(1+|x|)^2(1+\vsob{\sigma})^{2(n-i+2)}(1+|\theta|+|\theta'|)^{2n-2i}| \\ & \times\theta-\theta'|^2 \\
\le	&
	16nD^2(1+|x|)^4(1+\vsob{\sigma})^{2(2n-i+4)}(1+|\theta|+|\theta'|)^{2+4n-2i} \\ & \times|\theta-\theta'|^2.
\end{align*}
moreover by, Corollary \ref{cor:diff}, for the first term, we have
\begin{equation*}
|\sigma (\mathbf{w}_1^n,\mathbf{z})-\sigma (\mathbf{w'}_1^n,\mathbf{z})|^2 \le
	D n (1+|x|)^2(1+\vsob{\sigma})^{2(n+1)}(1+|\theta|+|\theta'|)^{2(n-1)} |\theta-\theta'|^2.
\end{equation*}
Hence
\begin{align*}
|\partial_\theta f(\theta,\mathbf{z})-\partial_\theta f(\theta',\mathbf{z})|^2
&\le
16(n+1)^2 D^2 (1+|x|)^4 (1+\vsob{\sigma})^{2(2n+3)}
(1+|\theta|+|\theta'|)^{4n}|\theta-\theta'|^2.
\end{align*}
\end{proof}
The next Proposition asserts that the growth condition \ref{G_growth} holds with
	\begin{equation*}
	K(x) = 4D\sqrt{n+1}(1+|x|)^2(1+\vsob{\sigma})^{n+2}
	\end{equation*}
	whenever $r\ge \frac{n+3}{2}$.
	
	\begin{proposition}\label{kellmajd}
	For any $\theta\in\Theta$ and $x\in\R^m$,	
	\begin{equation*}
	|G(\theta,x)|\le 4D\sqrt{n+1}(1+|x|)^2(1+\vsob{\sigma})^{n+2}(1+|\theta|^{n+1}).
	\end{equation*}		
	\end{proposition}
	\begin{proof}
		By Lemma \ref{lem:deriv}, for arbitrary $x\in\R^m$ and $\theta\in\R^d$ , one calculates
		\begin{align*}
			|G(\theta,x)| &= \Vert G(\theta,x) \Vert =
			2|y-f(\theta,\mathbf{z})| |\partial_\theta f(\theta,\mathbf{z})|
			\\
			&\le
			2(|y|+|f(\theta,\mathbf{z})|) |\partial_\theta f(\theta,\mathbf{z})|
			\\
			&\le
			2(1+|x|)D^{1/2}(1+\vsob{\sigma})(1+|\theta|)
			|\partial_\theta f(\theta,\mathbf{z})|
			\\
			&\le
			4D\sqrt{n+1}(1+|x|)^2(1+\vsob{\sigma})^{n+2}(1+|\theta|^{n+1})
		\end{align*}
		since $|\theta|+|\theta|^n\le 1+|\theta|^{n+1}$, for any $n\ge 1$.
	\end{proof}
The next Proposition states that Assumption \ref{A2} is satisfied with $\rho=3$,
$q-1=\max (2n+1,2r)$ and
\begin{equation*}
	L_1 = 16(1+\eta)(2r+1)(n+1)D^{3/2}(1+\vsob{\sigma})^{2n+4}.
\end{equation*}

\begin{proposition}[\textit{Link to Assumption \ref{A2} and  Proposition} \ref{polip H}]
	For any $\theta\in\Theta$ and $x\in\R^m$,
	\begin{align*}
	|H(\theta,x)-H(\theta',x)| &\le
	16(1+\eta)(2r+1)(n+1)D^{3/2}(1+|x|)^3 (1+\vsob{\sigma})^{2n+4}  (1+|\theta|+|\theta'|)^{q-1}|\theta-\theta'|
	\end{align*}
	where $q-1= \max (2n+1,2r)$.
\end{proposition}
\begin{proof}[\textbf{Proof of Proposition \ref{A2_for_ANN}}]
In view of Lemmas \ref{lem:deriv} and \ref{lem:dlip} and Corollary \ref{cor:diff}, one obtains for \textit{the first term that satisfies Assumption \ref{A2}} since
\begin{align*}
	\frac{1}{2}|G(\theta,x)-G(\theta',x)|
	\le
	&
	|y-f(\theta,\mathbf{z})||\partial_\theta f(\theta,\mathbf{z})-\partial_\theta f(\theta',\mathbf{z})|+
	|f(\theta,\mathbf{z})-f(\theta',\mathbf{z})|
	|\partial_\theta f(\theta',\mathbf{z})| \\
	\le &
	4(n+1)D^{3/2}(1+|x|)^3 (1+\vsob{\sigma})^{2n+4}
	(1+|\theta|+|\theta'|)^{2n+1}|\theta-\theta'| \\
	 & +
	2(n+1)(1+|x|)^2(1+\vsob{\sigma})^{2n+2}(1+|\theta|+|\theta'|)^{2n}|\theta-\theta'| \\\le&
	8(n+1)D^{3/2}(1+|x|)^3(1+\vsob{\sigma})^{2n+4}(1+|\theta|+|\theta'|)^{2n+1}
	|\theta-\theta'|,
\end{align*}
which completes the proof.	
\end{proof}

\section{Tables of constants} \label{TablesofConstants}

{We conclude the Appendix by presenting two tables of constants, which appear in our main results, either written in full analytic form or by declaring their dependencies on key parameters. }

\begin{table}[htb]
\renewcommand{\arraystretch}{2}
\centering
    \caption{Analytic expressions of constants}
    \label{tab:Other constants}
\begin{sc}
\scriptsize
    \begin{tabular}{@{}ll@{}} 
        \toprule
      Constant & \multicolumn{1}{c}{Full expression} \\ 
      \midrule
      $\bar{M}_{p}$& $\sqrt{1 / 3+4 B /(3 A)+4 d /(3 A \beta)+4(p-2) /(3 A \beta)}$ \\\hline  $\bar{c}(p)$&$A p / 4$  \\\hline $\tilde{c}(p)$&$(3 / 4) A p v_{p}\left(\bar{M}_{p}\right)$\\ \hline
      $\tilde{C}_1 $& $9\sqrt{\E[K^{4}(X_0)]+\eta^4 (\E|\theta_0|^4+C'_2) + \frac{48}{\beta^4}d^2}$\\ \hline
      $\tilde{C}_2$ & $3^{4l}\sqrt{\E K^{8l}(X_0)+\eta^{8l}(\E|\theta_0|^{8l}+C'_{4l}+(\frac{2}{\beta})^{8l}d^{4l}(8l-1)!!}$\\ \hline
          $C_1$ &$2\frac{L}{a}2^{4l}\tilde{C}_1(\E(1+X_0)^{8\rho})^{\frac{1}{4}}\left(1+2^{8l}(\E|\theta_0|^{8l}+C'_{4l})+\tilde{C}_2\right)^\frac{1}{4}.$\\ \hline
          $C_2$ &  $\sqrt{C'_{4r}+\E|\theta_0|^{8r}}\bigg(\frac{L^2}{a}(\E(1+|X_0|)^{8\rho})^{\frac{1}{4}}2^{2l+2}\left(\E(1+|\theta_0|)^{16l} + \E|\theta_0|^{16l}+C'_{8l}\right)^\frac{1}{8} $   \nonumber \\& $+\left(\E|\theta_0|^{16} +\E |\theta_0|^{16}+C'_8\right)^\frac{1}{8}\bigg)
  +\frac{4}{a}\sqrt{\E|H(\theta_0,X_0)|^4}$
         \\ \hline
         $C_3$ & $2\sqrt{2^{2q+3} \E K^2(X_0)(1+\E |\theta_0|^{2q}+C'_q)+2\eta^2(\E |\theta_0|^{4r+2}+C'_{2r+1})}$
         \nonumber\\&$\times 2 L \sqrt{ \E |H(0,X_0)|^2+2^{2l+2} \E (1+|X_0|)^{2\rho}(1+C'_{l+1}+\E|\theta_0|^{2l+2})}$ \\ \hline
         $z_1$ &$  \frac{\hat{c}}{1-exp(-\dot{c})}\left[\sqrt{e^{3a}(C_1+C_2+C_3)} + 3\left(5+4C'_2 \frac{\tilde{c}(4)}{\bar{c}(4)}+4\mathbb{E}|\theta_0|^4 \right)\right]$\\ \hline
         $z_2$& $\sqrt{2\hat{c}}\frac{e^{3a}(C_1+C_2+C_3)+1+\sqrt{2\mathbb{E}|\theta_0|^4+2+2C'_2 +\tilde{c}(4)/\bar{c}(4)}+\sqrt{2\mathbb{E}|\theta_0|^4+2+2C'_2}}{1-exp(-\dot{c}/2)}.$ \\ 
          \bottomrule
    \end{tabular}
    \end{sc}
\end{table}

{Taking a closer look at the constants in the two tables, one observes that the constants $z_1$, $z_2$, which appear in our convergence estimates in $W_1$, $W_2$ respectively, exhibit exponential dependence on the dimension of the problem. In fact, one can trace this exponential dependence to $\hat{c}$, a constant which is produced from the application of the  contraction results in \cite{Harris} to our non-convex setting. Note that our setting assumes only local Lipschitz continuity for the gradient of the non-convex objective  function. In other words, any problem-specific information which can improve the contraction estimates in \cite{Harris} by reducing their dependence to the dimension from exponential to polynomial, produces the same reduction in our estimates. \\
One also observes the effect of the regularisation parameter $\eta$ to the magnitude of our main constants. In particular, it is clear that $C'_p$, which is a class of constants most notably appearing in the moment estimates, depends on the negative $p$-th power of $\eta$. This is a direct consequence of the proposed regularization.\\
Another interesting observation is the relationship between $d$ and $\beta$ and their interplay with key constants such as $C'_p$ and $\dot{c}$. As it can be seen from Table \ref{tab:basic constants}, these constants depend on $d/\beta$. This implies that the choice of the temperature parameter can significantly reduce the impact of the dimension to these constants.\\
Finally, it is worth mentioning here that in our simulation results for the empirical risk minimization of (feed-forward) neural networks, our estimates seem not to suffer from such 'exploding' constants, which lead us to believe that in practice, and in particular in applications to non-'pathological' problems, the actual values of these constants are significantly lower than what is currently estimated.}
\begin{table}[htb]
\renewcommand{\arraystretch}{2}
\centering
    \caption{Main constants and their dependency to key parameters}
    \label{tab:basic constants}
\begin{sc}
\scriptsize
    \begin{tabular}{@{}lllll@{}}
    \toprule
     \multicolumn{1}{c}{Constant} &  \multicolumn{4}{c}{Key parameters}   \\
    \midrule
    \phantom{Constant} &$d$&$\beta$& Moments of $X_0$& $\eta$ \\\hline
         $A$ & - & - &$\mathcal{O}(\E K(X_0))$ & -  \\\hline
          $B$& - & - &$\mathcal{O}(\E K(X_0)^{q+2})$ &$\mathcal{O}(\frac{1}{\eta^{q+1}})$\\\hline
          $R$ & - & - &$\mathcal{O}(\E|X_0|^\rho)$& $\mathcal{O}(\frac{1}{\eta^{2r-q}})$\\\hline
          $a$ & - & - & $\mathcal{O}\left(\E |X_0|^{\rho(q-1)}\right)$& $\mathcal{O}(\frac{1}{\eta^{(2r-q)(q-1)}})$\\\hline
       $C'_p$ & $\mathcal{O}(\frac{d}{\beta})$ & $\mathcal{O}(\frac{d}{\beta})$ & $\mathcal{O}((\E K(X_0)^{2p})\frac{1}{2p})$& $\mathcal{O}(\frac{1}{\eta^p}) $ \\\hline  $\dot{c}$&  $poly(\frac{d}{\beta})$& $poly(\frac{d}{\beta})$& $\E poly (K(X_0)^{(q+1)/2})$& $poly(\frac{1}{\eta^{(q+1)/2}})$\\ \hline
       $\hat{c}$& $\mathcal{O}(e^d)$  & \multicolumn{3}{l}{Inherited from contraction estimates in \cite{eberle2019couplings}}  \\
       \bottomrule
    \end{tabular}
    \end{sc}
\end{table}

\bibliography{references_ArXiv23}

\clearpage
\onecolumn

\end{document}